\newcommand*{\ICLR}{}

\newcommand*{\CAMREADY}{}

\ifdefined\ARXIV
	\documentclass[10pt]{article}
		\usepackage{times}
	\usepackage{fullpage}
	
	\usepackage{natbib}
	
	\renewcommand{\cite}[1]{\citep{#1}}
	\setcitestyle{authoryear,round,citesep={;},aysep={,},yysep={;}}
	
	\usepackage[utf8]{inputenc} 
	\usepackage[T1]{fontenc}    
	\usepackage{booktabs}       
	\usepackage{multirow}
	\usepackage{microtype}      
	\usepackage{xcolor}         
	
	\usepackage{tabularx}
	\usepackage{thmtools}
	\usepackage{thm-restate}
	\usepackage[left=1.25in,right=1.25in,top=1in]{geometry}
    \usepackage[normalem]{ulem}

	\usepackage{comment}
	\usepackage{hyperref}
	\definecolor{mydarkblue}{rgb}{0,0.08,0.5}
	\hypersetup{ %
		pdftitle={},
		pdfauthor={},
		pdfsubject={},
		pdfkeywords={},
		colorlinks=true,
		linkcolor=mydarkblue,
		citecolor=mydarkblue,
		filecolor=mydarkblue,
		urlcolor=mydarkblue
	}
	
	\skip\footins 9pt plus 4pt minus 2pt
	\def\footnoterule{\kern-3pt \hrule width 12pc \kern 2.6pt }
	\leftmargini\leftmargin \leftmarginii 2em
	\renewenvironment{abstract}%
	{%
		\vskip 0in%
		\centerline%
		{\large\bf Abstract}%
		\vspace{-1ex}%
		\begin{quote}%
		}
		{
			\par%
		\end{quote}%
		\vskip 0ex%
	}
	
	\title{\vskip -4ex \bf{TBD} \vskip 0ex}
	\author{%
		\textbf{Yonatan Slutzky}\\ 
		Tel Aviv University\\
		\texttt{\normalsize slutzky1@mail.tau.ac.il} \\
		\and
		\textbf{Yotam Alexander}\\
		Tel Aviv University\\
		\texttt{\normalsize yotam.alexander@gmail.com} \\
		\and
		\textbf{Noam Razin}\\
		PLI, Princeton University\\
		\texttt{\normalsize noamrazin@princeton.edu}
		\and
		\textbf{Nadav Cohen}\\
		Tel Aviv University\\
		\texttt{\normalsize cohennadav@tauex.tau.ac.il}
	}	
	\date{}
\fi

\ifdefined\NEURIPS
	\PassOptionsToPackage{numbers}{natbib}
	\documentclass{article}
	\ifdefined\CAMREADY
		\usepackage[preprint]{neurips_2026}
	\else
		\usepackage{neurips_2026}
	\fi
	\usepackage{footmisc}
	\usepackage[utf8]{inputenc} 
	\usepackage[T1]{fontenc}    
	\usepackage{hyperref}       	
	\usepackage{url}            	
	\usepackage{booktabs}       
	\usepackage{amsfonts}       
	\usepackage{nicefrac}       	
	\usepackage{microtype}      
\fi

\ifdefined\CVPR
	\documentclass[10pt,twocolumn,letterpaper]{article}
	\usepackage{footmisc}
	\usepackage{cvpr}
	\usepackage{times}
	\usepackage{epsfig}
	\usepackage{graphicx}
	\usepackage{amsmath}
	\usepackage{amssymb}
	\usepackage[square,comma,numbers]{natbib}
\fi
\ifdefined\AISTATS
	\documentclass[twoside]{article}
	\ifdefined\CAMREADY
		\usepackage[accepted]{aistats2015}
	\else
		\usepackage{aistats2015}
	\fi
	\usepackage{url}
	\usepackage[round,authoryear]{natbib}
\fi
\ifdefined\ICML
	\documentclass{article}
	\usepackage{footmisc}
	\usepackage{microtype}
	\usepackage{graphicx}
	\usepackage{booktabs}
	\usepackage{hyperref}
	
	\ifdefined\CAMREADY
		\usepackage[accepted]{icml2025}
	\else
		\usepackage{icml2025}
	\fi
\fi
\ifdefined\ICLR
	\documentclass{article}
	\usepackage{iclr2027_conference,times}
	\usepackage{microtype}
	\usepackage{footmisc}
	\usepackage{hyperref}
	\usepackage{xcolor}
	\definecolor{mydarkblue}{rgb}{0,0.08,0.5}
	\hypersetup{ %
		pdftitle={},
		pdfauthor={},
		pdfsubject={},
		pdfkeywords={},
		colorlinks=true,
		linkcolor=mydarkblue,
		citecolor=mydarkblue,
		filecolor=mydarkblue,
		urlcolor=mydarkblue}

	\usepackage{url}

	\ifdefined\CAMREADY
		\iclrfinalcopy
	\fi
\fi
\ifdefined\COLT
	\ifdefined\CAMREADY
		\documentclass{colt2017}
	\else
		\documentclass[anon,12pt]{colt2017}
	\fi
\	usepackage{times}
\fi

\usepackage{color}
\usepackage{amsfonts}
\usepackage{algorithm}
\usepackage{algorithmic}
\usepackage{graphicx}
\usepackage{nicefrac}
\usepackage{bbm}
\usepackage{wrapfig}
\usepackage{tikz}
\usepackage[titletoc,title]{appendix}
\usepackage{stmaryrd}
\usepackage{comment}
\usepackage{endnotes}
\usepackage{hyperendnotes}
\usepackage{enumerate}
\usepackage{enumitem}
\usepackage[normalem]{ulem}
\usepackage{booktabs}
\usepackage{multirow}
\usepackage{amsmath}
	
\ifdefined\COLT
\else
	\usepackage{amsmath,amsthm,amssymb,mathtools}
\fi
\usepackage[small]{caption}
\usepackage[caption = false]{subfig}

\usepackage[extdef=true]{delimset}
\usepackage[capitalize,noabbrev]{cleveref}

\ifdefined\COLT
	\newtheorem{claim}[theorem]{Claim}
	\newtheorem{fact}[theorem]{Fact}
	\newtheorem{procedure}{Procedure}
	\newtheorem{conjecture}{Conjecture}	
	\newtheorem{hypothesis}{Hypothesis}	
	\Crefname{claim}{Claim}{Claims}
	\Crefname{fact}{Fact}{Facts}
	\Crefname{procedure}{Procedure}{Procedures}
	\Crefname{conjecture}{Conjecture}{Conjectures}
	\Crefname{hypothesis}{Hypothesis}{Hypotheses}
	
\else
	\newtheorem{lemma}{Lemma}
	\newtheorem{corollary}{Corollary}
	\newtheorem{theorem}{Theorem}
	\newtheorem{proposition}{Proposition}
	\newtheorem{assumption}{Assumption}

	\theoremstyle{definition}
	\newtheorem{definition}{Definition}
	\Crefname{lemma}{Lemma}{Lemmas}
	\Crefname{corollary}{Corollary}{Corollaries}
	\Crefname{theorem}{Theorem}{Theorems}
	\Crefname{proposition}{Proposition}{Propositions}
	\Crefname{assumption}{Assumption}{Assumptions}
	\Crefname{observation}{Observation}{Observations}
	\Crefname{example}{Example}{Examples}
	\Crefname{remark}{Remark}{Remarks}
	\Crefname{claim}{Claim}{Claims}
	\Crefname{fact}{Fact}{Facts}
	\Crefname{procedure}{Procedure}{Procedures}
	\Crefname{conjecture}{Conjecture}{Conjectures}
	\Crefname{hypothesis}{Hypothesis}{Hypotheses}
	\Crefname{definition}{Definition}{Definitions}
	\Crefname{appendix}{Appendix}{Appendices} %
\fi

\definecolor{green}{rgb}{0.0, 0.5, 0.0}

\definecolor{xcolor-gray}{gray}{0.95}

\newcommand{\kbf}{{\mathbf k}}
\newcommand{\qbf}{{\mathbf q}}
\newcommand{\xbf}{{\mathbf x}}
\newcommand{\ybf}{{\mathbf y}}

\newcommand{\ubf}{{\mathbf u}}
\newcommand{\vbf}{{\mathbf v}}
\newcommand{\rr}{{\mathbf r}}
\newcommand{\wbf}{{\mathbf w}}

\newcommand{\ebf}{{\mathbf e}}

\newcommand{\obf}{{\mathbf o}}

\newcommand{\0}{{\mathbf 0}}

\newcommand{\D}{{\mathcal D}}

\newcommand{\F}{{\mathcal F}}

\newcommand{\NN}{{\mathcal N}}
\newcommand{\PP}{{\mathcal P}}
\newcommand{\RR}{{\mathcal R}}

\newcommand{\T}{{\mathcal T}}
\newcommand{\U}{{\mathcal U}}
\newcommand{\V}{{\mathcal V}}

\newcommand{\X}{{\mathcal X}}
\newcommand{\Y}{{\mathcal Y}}
\newcommand{\OO}{{\mathcal O}}

\renewcommand{\L}{\mathcal{L}}

\newcommand{\EE}{\mathop{\mathbb E}} 
\newcommand{\R}{{\mathbb R}}
\newcommand{\N}{{\mathbb N}}

\newcommand{\BN}{\mathbb{N}}
\newcommand{\BR}{\mathbb{R}}
\newcommand{\BC}{\mathbb{C}}

\DeclareMathOperator{\Tr}{Tr}

\DeclareMathOperator{\diag}{Diag}

\newcommand{\dist}{\text{Dist}}

\newcommand{\dstate}{m}
\newcommand{\Klqr}{\Kunsafe}
\newcommand{\Kinf}{\Ksafe}
\newcommand{\gammainf}{\gamma_\infty}
\newcommand{\Hinf}{H_\infty}

\newcommand{\Plqr}{P_{\mathrm{unsafe}}}
\newcommand{\Pinf}{P_{\mathrm{safe}}}
\newcommand{\lip}{\text{Lip}}
\newcommand{\wt}{\widetilde}
\newcommand{\fst}{f_{\text{state}}}
\newcommand{\fobs}{f_{\text{obs}}}
\newcommand{\frwd}{f_{\text{reward}}}
\newcommand{\fdist}{f_{\text{dist}}}
\newcommand{\Fdist}{\F_{\text{dist}}}
\newcommand{\Xban}{\X_{\text{ban}}}
\newcommand{\Thetatr}{\Theta_{\mathrm{train}}}
\newcommand{\Thetate}{\Theta_{\mathrm{test}}}

\newcommand{\piteach}{\pi_{\phi^*}}
\newcommand{\Kunsafe}{K_{\mathrm{unsafe}}}
\newcommand{\Ksafe}{K_{\mathrm{safe}}}
\newcommand{\phiunsafe}{\phi_{\mathrm{unsafe}}}
\newcommand{\phisafe}{\phi_{\mathrm{safe}}}
\newcommand{\ksafe}{k^{\mathrm{safe}}}
\newcommand{\Anorm}{\norm{A}_{2}}
\newcommand{\Bnorm}{\norm{B}_{2}}

\newcommand{\Rinvnorm}{\norm{R^{-1}}_{2}}
\newcommand{\kbfsafe}{\kbf^{\mathrm{safe}}}
\newcommand{\alignfactor}{\alpha_{\mathrm{align}}}
\newcommand{\Dtrain}{\mathcal{D}^{\mathrm{train}}}
\newcommand{\Devaltr}{\mathcal{D}^{\mathrm{eval}}_{\mathrm{train}}}
\newcommand{\Devalte}{\mathcal{D}^{\mathrm{eval}}_{\mathrm{test}}}

\renewcommand{\bibsection}{}

\DeclareFontFamily{U}{mathx}{\hyphenchar\font45}
\DeclareFontShape{U}{mathx}{m}{n}{<-> mathx10}{}
\DeclareSymbolFont{mathx}{U}{mathx}{m}{n}
\DeclareMathAccent{\widebar}{0}{mathx}{"73}

\definecolor{darkspringgreen}{rgb}{0.09, 0.45, 0.27}
\usetikzlibrary{decorations.pathreplacing,calligraphy}
\usetikzlibrary{patterns}

\ifdefined\ENABLEENDNOTES
	\renewcommand{\theendnote}{\alph{endnote}} 
\else
	\renewcommand{\endnote}[1]{\null} 
\fi
\let\note\footnote

\ifdefined\CVPR
	\usepackage[breaklinks=true,bookmarks=false]{hyperref}
	\ifdefined\CAMREADY
		\cvprfinalcopy 
	\fi
	
\fi

\ifdefined\ARXIV
	\newcommand*{\ABBR}{}
\fi
\ifdefined\ICML
	\newcommand*{\ABBR}{}
\fi
\ifdefined\NEURIPS
	\newcommand*{\ABBR}{}
\fi
\ifdefined\ICLR
	\newcommand*{\ABBR}{}
\fi
\ifdefined\COLT
	\newcommand*{\ABBR}{}
\fi
\ifdefined\ABBR
	\newcommand{\eg}{{\it e.g.}}
	\newcommand{\ie}{{\it i.e.}}

\fi

\begin{document}
	
	
	\ifdefined\ARXIV
		\maketitle
	\fi
	\ifdefined\NEURIPS
	\title{Why Does Agentic Safety Fail to \\ Generalize Across Tasks?}
		\author{
		\hspace{-2.35em}
		\begin{tabular}[t]{c@{\hspace{3em}}c@{\hspace{3em}}c}
			\textbf{Yonatan Slutzky} & \textbf{Yotam Alexander} & \textbf{Tomer Slor} \\
			Tel Aviv University & Tel Aviv University & Tel Aviv University \\
			\texttt{slutzky1@mail.tau.ac.il} & \texttt{yotam.alexander@gmail.com} & \texttt{tomerslor@gmail.com}
		\end{tabular}
		\and
		\\[1ex]
		\begin{tabular}[t]{c@{\hspace{3em}}c}
			\textbf{Yoav Nagel} & \textbf{Nadav Cohen} \\
			\textbf{Tel Aviv University} & \textbf{Tel Aviv University} \\
			\texttt{yoavyosefn@mail.tau.ac.il} & \texttt{cohennadav@tauex.tau.ac.il}
		\end{tabular}
		}
		\maketitle
	\fi
	\ifdefined\CVPR
		\title{Paper Title}
		\author{
			Author 1 \\
			Author 1 Institution \\	
			\texttt{author1@email} \\
			\and
			Author 2 \\
			Author 2 Institution \\
			\texttt{author2@email} \\	
			\and
			Author 3 \\
			Author 3 Institution \\
			\texttt{author3@email} \\
		}
		\maketitle
	\fi
	\ifdefined\AISTATS
		\twocolumn[
		\aistatstitle{Paper Title}
		\ifdefined\CAMREADY
			\aistatsauthor{Author 1 \And Author 2 \And Author 3}
			\aistatsaddress{Author 1 Institution \And Author 2 Institution \And Author 3 Institution}
		\else
			\aistatsauthor{Anonymous Author 1 \And Anonymous Author 2 \And Anonymous Author 3}
			\aistatsaddress{Unknown Institution 1 \And Unknown Institution 2 \And Unknown Institution 3}
		\fi
		]	
	\fi
	\ifdefined\ICML
		\icmltitlerunning{A Provable Pitfall of Generalization in SSMs}
		\twocolumn[
		\icmltitle{A Provable Pitfall of Generalization in SSMs} 
		\icmlsetsymbol{equal}{*}
		\begin{icmlauthorlist}
			\icmlauthor{Author 1}{inst} 
			\icmlauthor{Author 2}{inst}
		\end{icmlauthorlist}
		\icmlaffiliation{inst}{Some Institute}
		\icmlcorrespondingauthor{Author 1}{author1@email}
		\icmlkeywords{}
		\vskip 0.3in
		]
		\printAffiliationsAndNotice{} 
	\fi
	\ifdefined\ICLR
		\title{Why Does Agentic Safety Fail to \\ Generalize Across Tasks?}
		\author{%
			Yonatan Slutzky \\
			Tel Aviv University \\
			\texttt{slutzky1@mail.tau.ac.il} \\
			\And
			Yotam Alexander \\
			Tel Aviv University \\
			\texttt{yotam.alexander@gmail.com} \\
			\And
			Tomer Slor \\
			Tel Aviv University \\
			\texttt{tomerslor@gmail.com} \\
			\And
			Yoav Nagel \\
			Tel Aviv University \\
			\texttt{yoavyosefn@mail.tau.ac.il} \\
			\And
			Nadav Cohen \\
			Tel Aviv University \\
			\texttt{cohennadav@tauex.tau.ac.il} \\
		}
		\maketitle
		\lhead{Preprint.}
	\fi
	\ifdefined\COLT
		\title{Paper Title}
		\coltauthor{
			\Name{Author 1} \Email{author1@email} \\
			\addr Author 1 Institution
			\And
			\Name{Author 2} \Email{author2@email} \\
			\addr Author 2 Institution
			\And
			\Name{Author 3} \Email{author3@email} \\
			\addr Author 3 Institution}
		\maketitle
	\fi

	\begin{abstract}
AI agents are increasingly deployed in multi-task settings, where the task to perform is specified at test time, and the agent must generalize to unseen tasks.
A major concern in such settings is safety: often, an agent must not only execute unseen tasks, but do so while avoiding risks and handling ones that materialize.
Empirical evidence suggests that even when the ability to execute generalizes to unseen tasks, the ability to do so safely frequently does not.
This paper provides theory and experiments indicating that failures of agentic safety to generalize across tasks are not merely due to limitations of training methods, but can also reflect an inherent property of safety itself: the relationship between a task and its safe execution is more complex than the relationship between a task and its execution alone.
Theoretically, we analyze linear-quadratic control with $\Hinf$-robustness, and prove that in many cases, the mapping from task specification to an optimal controller has higher Lipschitz constant with safety requirements than without, yielding a Lipschitz bound of independent interest.
Empirically, we demonstrate our conclusions in simulated quadcopter navigation with a neural network agent and in CRM with an LLM agent, showing that imitating safe and unsafe teachers is similarly straightforward on tasks seen in training, yet generalizing across tasks is considerably more difficult with the safe teacher.
Our findings suggest that current efforts to enhance agentic safety may be insufficient, and point to a need for fundamentally different approaches.
\end{abstract}

	\ifdefined\COLT
		\medskip
	\fi

	
	\vspace{-3mm}
\section{Introduction}
\label{sec:intro}
\vspace{-2mm}

The evolution of Artificial Intelligence (AI) is manifested by its progression from passive settings, where it processes static data for generating predictions, to \emph{agentic} settings, where it interacts with dynamic environments for performing tasks entailing long-term rewards.
Agentic settings have been studied for decades in the context of Reinforcement Learning (RL)~\citep{kaelbling1996reinforcement,sutton1998reinforcement,li2017deep,bertsekas2019reinforcement} and control theory~\citep{rynaski1965theory,molinari1977time,zhou1996robust,aastrom2021feedback}, and in recent years---due to the rapid proliferation of Large Language Models (LLMs)~\citep{brown2020language,team2023gemini,openai2023gpt4,caruccio2024claude,bi2024deepseek,comanici2025gemini,liu2025deepseek}---they have become a central facet of practical AI~\citep{yao2023react,schick2023toolformer,deng2023mind2web,zhou2024webarena,liu2024agentbench,mialon2024gaia,guo2024large}.
Unlike classic RL and control theory formulations, which typically consider a single (fixed) task, modern agentic AI is predominantly \emph{multi-task}: there are multiple possible tasks, and the one to perform is specified only post-deployment (\ie, at test time)~\citep{dayan1993feudal,duan2016rl2,wang2017learning,finn2017model,rakelly2019efficient,zintgraf2020varibad,yu2020meta,zitkovich2023rt,kirk2023survey}.
Often, it is impractical to train on all possible tasks, and the hope is that a learned agent will \emph{generalize across tasks}, meaning it will perform tasks unseen in training as effectively as it performs tasks that are seen in training.

A primary concern in agentic AI is \emph{safety}: when decisions are made and acted upon with no human in the loop, the implications of unsafe behavior can be catastrophic~\citep{banerjee2023system,jimenez2024swe,yang2024swe,fan2025workflowllm,amro2026github}.
For example, an autonomous navigation agent may lead a vehicle to a high-risk region (\eg, one with many obstacles), and a customer-facing agent may succumb to phishing attacks, revealing sensitive enterprise information.
Two important characteristics of agentic safety manifested in the latter two examples are:
\emph{(i)}~\emph{risk avoidance}: a preference for avoiding situations of high risk;
and
\emph{(ii)}~\emph{risk handling}: an ability to cope with risks when they materialize.

In many applications it is difficult to formulate criteria that ensure agentic safety~\citep{jilk2018limits,gough2026bounded,biswas2026responsible}.
A common approach for circumventing said difficulty is to incorporate safety into a learned agent through \emph{imitation learning}, \ie, through demonstrations from a trusted safe teacher (often a human)~\citep{bai2022training,ouyang2022training,dai2024safe,swamy2024minimaximalist}.
While this approach has been shown to enhance safe behavior on tasks included in training~\citep{liu2023robustness,ren2024codeattack}, the \emph{safe behavior often fails to generalize across tasks}, even in cases where the agent's ability to execute does generalize~\citep{ren2024codeattack,mou2024sg,andriushchenko2025agentharm,zinjad2025can,kwon2026safety}.
This phenomenon raises a fundamental question: are failures of agentic safety to generalize across tasks caused by limitations in current training techniques, or are they due to inherent properties that make safety less transferable across tasks than the ability to execute?

In this paper, we theoretically and empirically investigate the foregoing question. 
Theoretically, we analyze what is perhaps the simplest embodiment of learning a multi-task agent with safety requirements: learning to control a Linear Dynamical System (LDS) to maximize a task-dependent quadratic reward, while handling a risk of disturbances, \ie, while ensuring robustness in the sense of~$\Hinf$~norm~\citep{rynaski1965theory,molinari1977time,zhou1996robust,anderson2007optimal,bacsar2008h,green2012linear}.
We prove that in many cases, the mapping from task specification (quadratic reward matrix) to an optimal controller has higher Lipschitz constant under safety requirements (\ie, with $\Hinf$-robustness) than it does when safety is ignored (\ie, without $\Hinf$-robustness).
In light of the literature~\citep{barron1993universal,bousquet2002stability,gyorfi2002distribution,tsybakov2003introduction,xu2012robustness,yarotsky2017error,poggio2018theory,schmidt2020nonparametric} tying the Lipschitz constant of a target mapping to the difficulty of generalizing when learning it, our result indicates that \emph{generalization across tasks is fundamentally more difficult with safety requirements than without}.
This holds \emph{regardless of how well safety requirements are met on tasks seen in training}.
As a technical contribution of independent interest, we establish an upper bound for the Lipschitz constant of the Linear-Quadratic Regulator (LQR) as a function of its state cost matrix~$Q$~\citep{dorato1995linear,anderson2007optimal}.

We corroborate our theory via experiments in three settings that span different types of agents and different notions of safety:
\emph{(i)}~a setting that mirrors our theoretical analysis: learning to control an LDS to maximize a task-dependent quadratic reward while handling a risk of disturbances;
\emph{(ii)}~a setting in which a neural network agent is trained to navigate a simulated quadcopter to a task-dependent location, while avoiding regions marked as high-risk or handling a risk of disturbances;
and
\emph{(iii)}~a setting in which an LLM agent (LLaMA-3.2 model~\citep{meta2024llama32}) is fine-tuned to perform tasks in a realistic Customer Relationship Management (CRM) benchmark~\citep{levy2026st}, while avoiding situations marked as high-risk and handling events defined as a risk materializing.
Our experiments compare learning to imitate a safe teacher to learning to imitate an unsafe teacher (\ie, a teacher that executes tasks while ignoring safety).
In all settings, we find that although imitating safe and unsafe teachers on tasks seen in training is similarly straightforward, generalizing across tasks is considerably more difficult with the safe teacher.

Overall, our findings suggest that failures of agentic safety to generalize across tasks are at least in part due to inherent properties of safety, and more specifically, to a complex relationship between the task to be executed and the characteristics of safe execution.
From a practical standpoint, this implies that current efforts to enhance agentic safety through more training on a subset of tasks may be insufficient, necessitating a fundamentally different approach.
A promising possibility suggested by our findings is to learn representations for task specifications and safety characteristics which simplify the relationship between them.
We hope that our work will spur progress along this line.

\vspace{-3mm}
\subsection{Paper Organization}\label{sec:intro:org}
\vspace{-2mm}

The remainder of the paper is organized as follows. 
\cref{sec:prelim} presents a formalism for reasoning about learning multi-task agents with safety requirements.
\cref{sec:analysis} delivers our theoretical analysis.
\cref{sec:exper} reports our experimental results.
\cref{sec:limit} discusses limitations of our work.
\cref{sec:related} surveys related work.
Finally, \cref{sec:conclusion} concludes.

	\vspace{-3mm}
\section{Preliminaries}
\label{sec:prelim}
\vspace{-2mm}

\subsection{Notation}
\label{sec:prelim:not}
\vspace{-2mm}

We use non-boldface lowercase letters for denoting scalars (\eg, $\alpha \in \R$, $d \in \N$), boldface lowercase letters for denoting vectors (\eg, $\xbf \in \R^d$), and non-boldface uppercase letters for denoting matrices (\eg, $A \in \R^{d ,  d}$).
For $d \in \N$, we let $[d] := \{1,\ldots,d\}$.
We respectively use $\ebf_j$ and $I$ to denote the $j$-th standard basis vector and the identity matrix, with their dimensions omitted from the notation and inferred from context.
We use~$\|\cdot\|_{2}$ to denote the Euclidean norm for vectors and the induced spectral norm for matrices, and $\|\cdot\|_{F}$ to denote the Frobenius norm for matrices.
For a square matrix $P$, we use $\Tr(P)$ to denote its trace, and write $P \succ 0$ (respectively, $P \succeq 0$) to indicate that $P$ is positive definite (respectively, positive semidefinite).
For $d \in \BN$ and a function $f : \V \to \BR^{d , d}$ over some domain $\V \subseteq \BR^{d , d}$, we use $\lip ( f )$ to denote the Lipschitz constant of~$f ( \cdot )$, defined by \smash{$\lip ( f ) := \sup_{V_{1}, V_{2} \in \V , V_{1} \neq V_{2}} \norm{f(V_{1}) - f(V_{2})}_{F} \big/ \norm{V_{1} - V_{2}}_{F}$}.

\vspace{-2mm}
\subsection{Learning a Multi-Task Agent with Safety Requirements}
\label{sec:prelim:setup}
\vspace{-2mm}

We consider the below formulation of learning a multi-task agent with safety requirements.
The formulation captures not only conventional multi-task settings in RL and control theory, but also contemporary settings involving LLM agents; see \cref{sec:exper:LLM} for details.

An \emph{environment} comprises a space of \emph{states}~$\X$, a space of \emph{observations}~$\OO$, a space of \emph{actions}~$\U$, and a space of \emph{tasks}~$\Theta$.
At every time $t\in \N\cup\{0\}$, the environment holds a state $\xbf_t \in \X$, produces an observation $\obf_t \in \OO$, intakes an action $\ubf_t \in \U$, and grants a reward $r_t \in \BR$, where the latter depends on a chosen task $\theta \in \Theta$.
The state evolves per the update rule $\xbf_{t + 1} = \fst ( \xbf_t , \ubf_t , t )$, where $\fst ( \cdot )$ is the environment's \emph{state transition function}, and the initial state $\xbf_0$ is set by some (deterministic or random) assignment.
The observation is produced per $\obf_t = \fobs ( \xbf_t , t )$, where $\fobs ( \cdot )$ is the environment's \emph{observation function}.
The reward is assigned per $r_t = \frwd ( \xbf_t , \obf_t , \ubf_t , t ; \theta )$, where $\frwd ( \cdot )$ is the environment's \emph{reward function}.\note{%
It is possible to account for environments where the reward depends on states, observations and actions from the past (\ie, where $\frwd ( \cdot )$ respectively intakes $( \xbf_\tau )_{\tau = 0}^t$, $( \obf_\tau )_{\tau = 0}^t$ and $( \ubf_\tau )_{\tau = 0}^t$ in place of $\xbf_t$, $\obf_t$ and $\ubf_t$), but for simplicity we consider the typical setting of an instantaneous reward.
}
Each of the functions $\fst ( \cdot )$, $\fobs ( \cdot )$ and $\frwd ( \cdot )$ may be stochastic (in the sense that its output may be sampled from some random distribution), and is allowed to vary with time (hence it receives the time index as an input argument).

An \emph{agent} interacts with the environment by assigning, at every time $t\in \N\cup\{0\}$, the action~$\ubf_t$ based on: the concurrent observation~$\obf_t$, the past observations~$ ( \obf_\tau )_{\tau = 0}^{t - 1}$, the past actions~$ ( \ubf_\tau )_{\tau = 0}^{t - 1}$, and the chosen task~$\theta$.
Namely, the agent implements $\ubf_t = \pi_\phi ( \obf_t , ( \obf_\tau )_{\tau = 0}^{t - 1} ,  ( \ubf_\tau )_{\tau = 0}^{t - 1} ; \theta)$, where $\pi_\phi ( \cdot )$ is the agent's \emph{policy function}, which may be stochastic, and is parameterized by $\phi \in \Phi$.
The ability of the agent to execute the task~$\theta$ is quantified by the \emph{expected cumulative reward} $\RR ( \phi ; \theta ) := \EE \big[ \sum_{t = 0}^\infty r_t \big]$, where the expectation is over the (possible) randomness in the environment (state transition function~$\fst ( \cdot )$, observation function~$\fobs ( \cdot )$, reward function~$\frwd ( \cdot )$, and the assignment of the initial state~$\xbf_0$) and in the agent (policy function~$\pi_\phi ( \cdot )$).

The extent to which the agent executes a task safely is evaluated through the complementary notions of \emph{risk avoidance} and \emph{risk handling}.
Risk avoidance is formalized by defining a set of banned states $\Xban \subset \X$ known to the agent, and requiring the agent to avoid those, \ie, to assign actions such that $\xbf_t \notin \Xban$ for every $t\in \N\cup\{0\}$.
This corresponds to agentic settings with known ``no-go'' regions, for example autonomous navigation around restricted areas, robotic manipulation near fragile objects, and industrial control systems with hard operating limits.
Risk handling is formalized by introducing disturbances to state transitions and requiring the agent to mitigate degradation in the expected cumulative reward.
Specifically, the requirement is that the agent assign actions such that replacement of the state transition function~$\fst ( \cdot )$ with any function from a set~$\Fdist$ does not lead to a significant reduction in~$\RR ( \cdot )$.
This captures agentic settings where risks cannot be fully avoided and must be dealt with as they materialize, for example autonomous driving under sudden wind gusts or slippery roads, LLM agents exposed to phishing attacks, and industrial control systems encountering never-before-seen regimes.

The agent's policy function~$\pi_\phi ( \cdot )$ is trained via \emph{imitation learning} from a \emph{teacher policy function}~$\piteach ( \cdot )$ over a set of \emph{training tasks} $\Thetatr \subset \Theta$.
That is, the agent's parameters~$\phi$ are tuned with the objective of having $\pi_\phi ( \cdot \, ; \theta )$ match~$\piteach ( \cdot \, ; \theta )$ for every $\theta \in \Thetatr$.
\emph{Generalization across tasks} is measured by the extent to which $\pi_\phi ( \cdot \, ; \theta )$ matches~$\piteach ( \cdot \, ; \theta )$ when the task~$\theta$ is unseen in training, \ie, when $\theta \in \Thetate := \Theta \setminus \Thetatr$.
In practical settings, $\pi_\phi ( \cdot )$~is trained on a finite sample, \ie, on finitely many input-output examples extracted from observation-action trajectories that are generated by applying~$\piteach ( \cdot ; \theta )$ to the environment, where the task $\theta$ ranges over~$\Thetatr$.
In theoretical settings, to decouple finite sample effects from the study of generalization across tasks, one may consider training~$\pi_\phi ( \cdot )$ in the infinite sample limit, meaning $\pi_\phi ( \cdot \, ; \theta ) = \piteach ( \cdot \, ; \theta )$ for every $\theta \in \Thetatr$.

Our interest lies in comparing different scenarios in terms of generalization across tasks.
Specifically, we compare scenarios where the teacher policy function~$\piteach ( \cdot )$ is safe (in terms of risk avoidance and/or risk handling), against ones where~$\piteach ( \cdot )$ is unsafe (\ie, it executes tasks while ignoring safety).
Since safety requirements can significantly alter attainable rewards, a reward-based comparison may be misleading.
Instead, we base the comparison on \emph{imitation error}, \ie, on the distance between the outputs of $\pi_\phi ( \cdot \, ; \theta )$ and~$\piteach ( \cdot \, ; \theta )$ for identical inputs, where $\theta \in \Thetate$.

\vspace{-2mm}
\subsection{Special Case: Linear-Quadratic Control with $\Hinf$-Robustness}
\label{sec:prelim:setup_theory}
\vspace{-1mm}

Below we describe what is perhaps the simplest special case of the formulation in \cref{sec:prelim:setup}: learning to control a \emph{Linear Dynamical System} (\emph{LDS}) to maximize a task-dependent quadratic reward, while ensuring \emph{$\Hinf$-robustness}~\citep{rynaski1965theory,molinari1977time,zhou1996robust,anderson2007optimal,bacsar2008h,green2012linear}.

For some $\dstate \in \N$, let the state, observation and action spaces, \ie, $\X$, $\OO$ and~$\U$, respectively, all equal~$\R^\dstate$.
Let the state transition function~$\fst ( \cdot )$ be given by:
\begin{align}
   \label{eq:lin_dyn}
   \fst(\xbf_{t}, \ubf_{t}, t) = A \xbf_{t} + B \ubf_{t} + D \wbf_{t}
   \text{\,,}
\end{align}
where:
\emph{(i)}~$A , B , D \in \BR^{\dstate,\dstate}$ are referred to as the \emph{state matrix}, the \emph{input matrix} and the \emph{noise matrix}, respectively;
\emph{(ii)}~$( \wbf_t )_{t = 0}^\infty$~are independent random variables over~$\R^\dstate$ whose means are zero and whose variances are summable, meaning \smash{$\EE \big[ \sum_{t = 0}^{\infty} \norm{\wbf_t}_{2}^{2} \big] \leq c$} for some constant $c \in \R_{> 0}$;\note{%
Summability of variances is required in order to ensure that the expected cumulative reward~$\RR( \cdot )$ is finite; see \cref{app:background:lqr} for details.
}
and
\emph{(iii)}~the initial state~$\xbf_0$ is zero.
Assume stability and non-degeneracy, in the sense that $\Anorm \in ( 0 , 1 )$ and $\Bnorm \neq 0$.
Consider full observability, meaning that observations are equal to their concurrent states, \ie, the observation function~$\fobs ( \cdot )$ is given by $\fobs ( \xbf_{t} , t ) = \xbf_{t}$.
Let the task space~$\Theta$ equal $\{ Q \in \R^{\dstate, \dstate} : Q \succeq 0 \text{ and } \norm{Q}_{F} \leq 1\}$, and the reward function~$\frwd ( \cdot )$ be given by $\frwd( \xbf_{t}, \ubf_{t}, t ; Q ) = - \xbf_{t}^\top Q \xbf_{t} - \ubf_{t}^\top R \ubf_{t}$, where $Q$ is referred to as the \emph{state reward matrix}, and $R \in \R^{\dstate ,  \dstate}$ satisfies $R \succ 0$ and is referred to as the \emph{action reward matrix}.

Consider the risk handling notion of safety, where $\Fdist$ consists of all deterministic analogues of~$\fst ( \cdot )$, \ie, of all functions that adhere to \cref{eq:lin_dyn} and the subsequent text, except that they assign $( \wbf_t )_{t = 0}^\infty$ deterministically, to values which may not be zero but uphold \smash{$\sum_{t=0}^{\infty} \norm{\wbf_t}_2^2 \leq c$}.
The requirement from the agent to mitigate degradation in the expected cumulative reward~$\RR( \cdot )$ is formalized through a worst-case lens.
Namely, for a given task $Q \in \Theta$, the agent is required to maximize $\inf \RR ( \cdot \, ; Q )$, where the infimum is taken over all cases where $\fst ( \cdot )$ is replaced by a member of~$\Fdist$.
This is known in the control theory literature as $\Hinf$-robustness~\citep{zhou1996robust,bacsar2008h,green2012linear}.

A fundamental result in control theory states that in the absence of safety requirements, the optimal policy function is the \emph{Linear-Quadratic Regulator} (\emph{LQR})~\citep{zhou1996robust,bacsar2008h,green2012linear}.
Specifically, with $\fst ( \cdot )$ maintaining its original form (\cref{eq:lin_dyn} and subsequent text), for a given task $Q \in \Theta$, the expected cumulative reward~$\RR( \phi ; Q )$ attains its maximal possible value if the agent's parameters~$\phi$ are such that actions are assigned as $\ubf_t = \Kunsafe ( Q ) \obf_t$, where $\Kunsafe ( Q ) \in \BR^{\dstate , \dstate}$ is referred to as the \emph{LQR whose state cost matrix is~$Q$} (see \cref{app:background:lqr} for further details).
Control theory also provides, for any $Q \in \Theta$, a canonical construction for a linear policy function that is optimal under safety requirements~\citep{zhou1996robust,bacsar2008h,green2012linear}.
More precisely, the theory of $\Hinf$-robustness provides a canonical mapping $Q \mapsto \Ksafe ( Q ) \in \BR^{\dstate , \dstate}$ such that the following holds: for a given task $Q \in \Theta$, when $\fst ( \cdot )$ can be replaced by any $\fdist ( \cdot ) \in \Fdist$, if $\phi$ are such that actions are assigned as $\ubf_t = \Ksafe ( Q ) \obf_t$, then $\inf_{\fdist ( \cdot ) \in \Fdist} \RR( \phi ; Q )$ attains its maximal possible value (see \cref{app:background:hinf} for further details).

In accordance with the optimal unsafe and safe policy functions, consider the following form for the agent's policy function: $\pi_\phi ( \obf_t , ( \obf_\tau )_{\tau = 0}^{t - 1} ,  ( \ubf_\tau )_{\tau = 0}^{t - 1} ; Q ) = K_\phi ( Q ) \obf_t$, where $K_\phi : \Theta \to \BR^{\dstate , \dstate}$ is a mapping parameterized by~$\phi$, which realizes $\Kunsafe ( \cdot )$ and~$\Ksafe ( \cdot )$ for some parameter values $\phiunsafe$ and~$\phisafe$, respectively.
We compare a scenario of imitation learning where the teacher policy function is~$\pi_{\phiunsafe} ( \cdot )$, against one where it is~$\pi_{\phisafe} ( \cdot )$.
In the former scenario, training in the infinite sample limit means that $K_\phi ( Q )$ equals~$\Kunsafe ( Q )$ for every $Q \in \Thetatr$, and generalization across tasks is measured by the discrepancy between $K_\phi ( Q )$ and~$\Kunsafe ( Q )$ for $Q \in \Thetate$.
The same holds in the latter scenario, but with $\Ksafe ( Q )$ in place of~$\Kunsafe ( Q )$.

	\vspace{-3mm}
\section{Theoretical Analysis}
\label{sec:analysis}
\vspace{-2mm}

\cref{sec:prelim:setup_theory} presents what is perhaps the simplest embodiment of learning a multi-task agent with safety requirements: learning to control an LDS to maximize a task-dependent quadratic reward, while ensuring $\Hinf$-robustness.
In the setting of \cref{sec:prelim:setup_theory}, the task space~$\Theta$ is $\{ Q \in \R^{\dstate, \dstate} : Q \succeq 0 \text{ and } \norm{Q}_{F} \leq 1\}$, and generalization across tasks with and without safety requirements reduces to generalization in learning certain target mappings $\Ksafe : \Theta \to \R^{\dstate, \dstate}$ and $\Kunsafe : \Theta \to \R^{\dstate, \dstate}$, respectively.
In this section we deliver a theoretical analysis which proves that, under technical conditions, the Lipschitz constant of~$\Ksafe ( \cdot )$ can be arbitrarily greater than that of~$\Kunsafe ( \cdot )$.
In light of the literature~\citep{barron1993universal,bousquet2002stability,gyorfi2002distribution,tsybakov2003introduction,xu2012robustness,yarotsky2017error,poggio2018theory,schmidt2020nonparametric} tying the Lipschitz constant of a target mapping to the difficulty of generalizing when learning it, our theory establishes formal evidence that \emph{generalization across tasks is fundamentally more difficult with safety requirements than without}.
Given that the analyzed setting (\cref{sec:prelim:setup_theory}) considers training in the infinite sample limit, the established difficulty of safety to generalize across tasks holds \emph{despite safety being learned perfectly on training tasks}.

An intermediate step in our analysis is a derivation of an explicit upper bound for the Lipschitz constant of~$\Kunsafe ( \cdot )$.
In control theory, $\Kunsafe ( \cdot )$~represents the mapping from a state cost matrix to its corresponding LQR.
Given the importance of this mapping, the bound we derive for its Lipschitz constant is viewed as a technical contribution of independent interest.

\vspace{-3mm}
\subsection{Stability Margin}
\label{sec:analysis:stability}
\vspace{-2mm}

Our analysis relies on a \emph{stability margin} condition for the LDS, which requires the spectral norm of the state matrix~$A$ to be bounded away from unity by a certain quantity---see \cref{ass:stable}.
While this condition is a technical requirement of our current proof techniques, we demonstrate empirically in \cref{sec:exper:lqr} that our theoretical conclusions persist even when it is violated. 
Relaxing the condition is an important direction for future work (see \cref{sec:limit}).

\begin{assumption}[Stability Margin]
\label{ass:stable}
    The spectral norm of the state matrix~$A$ satisfies:
    \begin{align}
        \Anorm < 1 - \frac{1}{1 + \big( \sqrt{2} + \sqrt{2} \Bnorm^{2} \Rinvnorm \big)^{-1}}
        \text{\,.}
        \label{eq:stability}
    \end{align}
\end{assumption}

\subsection{Upper Bound for the Lipschitz Constant of~$\Kunsafe(\cdot)$}
\label{sec:analysis:ub}
\vspace{-1mm}

\cref{res:ub} establishes an upper bound for the Lipschitz constant of~$\Kunsafe ( \cdot )$.
In control theory, $\Kunsafe ( \cdot )$~represents the mapping from a state cost matrix to its corresponding LQR.
In contrast to classical results whose Lipschitz characterizations for~$\Kunsafe ( \cdot )$ are typically local or implicit \citep{konstantinov1993perturbation}, \cref{res:ub} delivers a characterization that is global and explicit.
Given the importance of~$\Kunsafe ( \cdot )$, we view \cref{res:ub} as a technical contribution of independent interest.

\begin{lemma}\label{res:ub}
    Under \cref{ass:stable}, the Lipschitz constant of the mapping~$\Kunsafe(\cdot)$ satisfies:
    \begin{align}
        \lip(\Klqr) \leq 2 \Anorm \Bnorm \Rinvnorm \big( 1 + 2 \Bnorm^{2} \Rinvnorm \big)
        \text{\,.}
        \label{eq:ub}
    \end{align}
\end{lemma}
\begin{proof}[Proof sketch (full proof in \cref{app:lqr})]
    By a fundamental result in control theory~\citep{zhou1996robust,bacsar2008h,green2012linear,bertsekas2019reinforcement}, for any $Q \in \Theta$, the matrix $\Kunsafe ( Q ) \in \BR^{\dstate , \dstate}$ (LQR whose state cost matrix is~$Q$) is given by a rational function of $A$, $B$ and~$R$, as well as a matrix $P ( Q ) \in \BR^{\dstate , \dstate}$ that is the unique fixed point of a \emph{Riccati operator} depending on~$Q$.
    Under \cref{ass:stable}, there exists a bounded region in which~$P ( Q )$ must lie for any $Q \in \Theta$, and within which the Riccati operator is a contraction for any $Q \in \Theta$.
    This implies that~$P ( Q )$ varies in a controlled manner as $Q$ changes, yielding an upper bound for the Lipschitz constant of the mapping $Q \mapsto P ( Q )$.
    \cref{eq:ub} then follows from propagating this upper bound through the rational function that gives~$\Kunsafe ( Q )$. 
\end{proof}

\vspace{-3mm}
\subsection{Separation between the Lipschitz Constants of $\Ksafe(\cdot)$ and~$\Kunsafe(\cdot)$}
\label{sec:analysis:sep}
\vspace{-1mm}

\cref{res:sep} proves that the Lipschitz constant of~$\Ksafe ( \cdot )$ is greater than or equal to that of~$\Kunsafe ( \cdot )$ times a particular coefficient which can be arbitrarily large.
This establishes formal evidence that \emph{generalization across tasks is fundamentally more difficult with safety requirements than without, even when safety requirements are met perfectly on training tasks} (see opening of \cref{sec:analysis}).

\cref{res:sep} relies on two conditions in addition to stability margin (\cref{ass:stable}).
While these conditions, formulated in \cref{ass:commute,ass:regular}, are technicalities of our current proof technique, we demonstrate empirically in \cref{sec:exper:lqr} that the conclusions of \cref{res:sep} persist even when they are violated.
Relaxing the conditions is an important direction for future work (see \cref{sec:limit}).

The first of the two conditions is that all fixed matrices are symmetric and commuting.
\begin{assumption}[Symmetric Commutativity]
    \label{ass:commute}
    The state, input and noise matrices $A$, $B$ and~$D$, respectively, are symmetric and commute with one another and with the action reward matrix~$R$.
\end{assumption}

\cref{ass:commute} implies that the matrices $A$, $B$, $D$ and~$R$ are simultaneously orthogonally diagonalizable, \ie, there exists an orthogonal matrix $V \in \BR^{\dstate , \dstate}$ such that:
\begin{equation}
    \begin{aligned}
        A = V \diag( \lambda_{1} ( A ) , \ldots ,\lambda_{\dstate} ( A ) ) \, V^{\top} 
        \quad , \quad
        B = V \diag( \lambda_{1} ( B ) , \ldots ,\lambda_{\dstate} ( B ) ) \, V^{\top} 
        \,, \\[0.5mm]
        D = V \diag( \lambda_{1} ( D ) , \ldots ,\lambda_{\dstate} ( D ) ) \, V^{\top} 
        \quad , \quad
        R = V \diag( \lambda_{1} ( R ) , \ldots ,\lambda_{\dstate} ( R ) ) \, V^{\top}
        \,,
    \end{aligned}
    \label{eq:sim_diag}
\end{equation}
where, for $M \in \{ A , B , D , R \}$, $\lambda_{1} ( M ) , \ldots ,\lambda_{\dstate} ( M )$ are the eigenvalues of~$M$ ordered according to the shared eigenbasis.
By the definition of spectral norm:\note{%
Recall that $R \succ 0$, implying $\lambda_{1} ( R ) , \ldots ,\lambda_{\dstate} ( R ) > 0$.
}
\begin{align}
    \max\nolimits_{j \in [ \dstate ]} \abs{\lambda_j ( A )} = \Anorm
    ~~ , ~~
    \max\nolimits_{j \in [ \dstate ]} \abs{\lambda_j ( B )} = \Bnorm
    ~~ , ~~
    \min\nolimits_{j \in [ \dstate ]} \lambda_j ( R ) = \Rinvnorm^{-1}
    \text{\,.}
    \label{eq:extrema}
\end{align}
\cref{def:align} defines the \emph{alignment factor}: a quantity that measures how close $A$, $B$ and~$R$ are to attaining all extrema in \cref{eq:extrema} at a common index $j \in [ \dstate]$.
\begin{definition}[Alignment Factor]
    \label{def:align}
    Under \cref{ass:commute} and the notation of \cref{eq:sim_diag}, define the \emph{alignment factor}~$\alignfactor$ to be the largest $\alpha \in ( 0 , 1 ]$ for which there exists an index $j \in [ \dstate]$ satisfying:\note{%
    $\alignfactor$ is well-defined since \cref{eq:align} is satisfied for some $j \in [ \dstate]$ when~$\alpha$ is sufficiently small.
    }
    \begin{equation}
        \begin{aligned}
            | \lambda_{j} ( A ) | \geq \alpha^{-1} \big( \Anorm + \alpha - 1 \big)
            \quad , \quad
            | \lambda_{j} ( B ) | \geq \alpha \Bnorm
            \quad , \quad
            \lambda_{j} ( R ) \leq \alpha^{-1} \Rinvnorm^{-1}
            \text{\,.}
        \end{aligned}
        \label{eq:align}
    \end{equation}
\end{definition}

The final condition required by \cref{res:sep} is a mild regularity of the alignment factor (\cref{def:align}).
\begin{assumption}[Alignment Regularity]
    \label{ass:regular}
    Under \cref{ass:commute} and the notation of \cref{eq:sim_diag}, among the indices $j \in [ \dstate ]$ that attain \cref{eq:align} with $\alpha = \alignfactor$, at least one satisfies $\lambda_{j} ( A ) \neq 0$ and $\lambda_{j} ( D ) \neq 0$.
\end{assumption}

With \cref{ass:commute,ass:regular} in place (in addition to \cref{ass:stable}), \cref{res:sep} proves that the Lipschitz constant of~$\Ksafe ( \cdot )$ is greater than or equal to that of~$\Kunsafe ( \cdot )$ times a particular coefficient (\cref{eq:sep}).
While this coefficient can be made arbitrarily large (via appropriate choices of the matrices $A$, $B$, $D$ and~$R$), there are cases where it is smaller than one.
Empirically, however, we find the estimated Lipschitz constant of~$\Ksafe ( \cdot )$ to consistently be greater than that of~$\Kunsafe ( \cdot )$ (see \cref{sec:exper:lqr}).
Developing a stronger version of \cref{res:sep} in which the coefficient is always greater than one is an important direction for future work (see \cref{sec:limit}).
\begin{theorem}
    \label{res:sep}
    Under \cref{ass:stable,ass:commute,ass:regular}, the Lipschitz constants of the mappings $\Ksafe(\cdot)$ and $\Kunsafe(\cdot)$ satisfy:
    \begin{align}
        \lip ( \Ksafe ) \geq \frac{\alignfactor^{3}}{\Anorm \big( 2 + 4 \Bnorm^{2} \Rinvnorm \big)} \cdot \lip(\Kunsafe)
        \text{\,.}
        \label{eq:sep}
    \end{align}
\end{theorem}
\begin{proof}[Proof sketch (full proof in \cref{app:separation})]
    The proof establishes \cref{eq:sep} by deriving a lower bound on~$\lip ( \Ksafe )$ and combining that with the upper bound on~$\lip ( \Kunsafe )$ from \cref{res:ub}.
    For the lower bound on~$\lip ( \Ksafe )$, attention is restricted to the subset $\Theta_V \subset \Theta$ comprising the matrices $Q \in \Theta$ that are diagonalized by the orthogonal matrix~$V$ from \cref{eq:sim_diag}.
    Since $V$ diagonalizes $A$, $B$, $D$ and~$R$, when $Q \in \Theta_V$, applying a change of basis using~$V$ decomposes our multi-dimensional $\Hinf$-robustness problem into a collection of one-dimensional problems that admit closed-form solutions.
    While these closed-form solutions do not necessarily agree with~$\Ksafe ( \cdot )$, alignment regularity (\cref{ass:regular}) ensures that there exists an open subset of~$\Theta_V$ in which at least one of the closed-form solutions does agree with~$\Ksafe ( \cdot )$.
    By computing the derivative of an agreeing closed-form solution, we obtain a lower bound on $\lip ( \Ksafe )$, as required.
\end{proof}

	\vspace{-3mm}
\section{Experiments}
\label{sec:exper}
\vspace{-2mm}

This section corroborates our theory (\cref{sec:analysis}) via experiments in three settings that span different types of agents and different notions of safety:
\emph{(i)}~a setting that mirrors our theoretical analysis: learning to control an LDS to maximize a task-dependent quadratic reward while handling a risk of disturbances (\cref{sec:exper:lqr});
\emph{(ii)}~a setting in which a neural network agent is trained to navigate a simulated quadcopter to a task-dependent location, while avoiding regions marked as high-risk or handling a risk of disturbances (\cref{sec:exper:quadcopter});
and
\emph{(iii)}~a setting in which an LLM agent (LLaMA-3.2 model~\citep{meta2024llama32}) is fine-tuned to perform tasks in a realistic Customer Relationship Management (CRM) benchmark~\citep{levy2026st}, while avoiding situations marked as high-risk and handling events defined as a risk materializing (\cref{sec:exper:LLM}).
Our experiments compare learning to imitate a safe teacher to learning to imitate an unsafe teacher (\ie, a teacher that executes tasks while ignoring safety).
In all three settings, we find that although imitating safe and unsafe teachers on tasks seen in training is similarly straightforward, generalizing across tasks is considerably more difficult with the safe teacher.
\ifdefined\CAMREADY
    Code for reproducing our experiments will be made available at \url{https://github.com/Tomerslortau/agentic-safety-generalization}. 
\fi

\vspace{-2mm}
\subsection{Linear-Quadratic Control with $\Hinf$-Robustness}
\label{sec:exper:lqr}
\vspace{-1mm}

Our theoretical analysis (\cref{sec:analysis}) pertains to what is perhaps the simplest embodiment of learning a multi-task agent with safety requirements: learning to control an LDS to maximize a task-dependent quadratic reward, while handling a risk of disturbances, \ie, while ensuring $\Hinf$-robustness.
In the analyzed setting (\cref{sec:prelim:setup_theory}), generalization across tasks with and without safety requirements reduces to generalization in learning certain target mappings $\Ksafe ( \cdot )$ and~$\Kunsafe ( \cdot )$, respectively.
Our main theoretical result (\cref{res:sep}) establishes that, under technical conditions (\cref{ass:stable,ass:commute,ass:regular}), the Lipschitz constant of~$\Ksafe ( \cdot )$ is greater than or equal to that of~$\Kunsafe ( \cdot )$ times a particular coefficient which can be arbitrarily large.
\cref{fig:lip_dim=4} corroborates this result by demonstrating empirically that the estimated Lipschitz constant of~$\Ksafe ( \cdot )$ is consistently greater than that of~$\Kunsafe ( \cdot )$, even when the coefficient is smaller than one or the technical conditions are violated.
\cref{tab:imitate} provides further corroboration by demonstrating empirically that---in line with the literature~\citep{barron1993universal,bousquet2002stability,gyorfi2002distribution,tsybakov2003introduction,xu2012robustness,yarotsky2017error,poggio2018theory,schmidt2020nonparametric} tying the Lipschitz constant of a target mapping to the difficulty of generalizing when learning it---the separation between Lipschitz constants translates to a separation between cross-task generalization errors.

\vspace{-2mm}
\subsection{Quadcopter Navigation}
\label{sec:exper:quadcopter}
\vspace{-1mm}

We demonstrate empirically that the conclusions of our theory (\cref{sec:analysis}) extend to a setting in which a neural network agent is trained to navigate a simulated quadcopter to a task-dependent location, while avoiding regions marked as high-risk or handling a risk of disturbances.
This setting, implemented using a differentiable rigid-body simulator~\citep{panerati2021learning}, is a special case of the formulation in \cref{sec:prelim:setup}, characterized as follows.
A state $\xbf_{t} \in \X = \BR^{12}$ represents the quadcopter's position and tilt angles, along with their respective velocities.
The initial state~$\xbf_0 \in \X$ is set to zero with the quadcopter hovering.
There is full observability, meaning that observations are equal to their concurrent states (\ie, $\OO = \X$ and $\fobs ( \xbf_{t} , t ) = \xbf_{t}$).
An action $\ubf_{t} \in \U = \BR^{4}$ represents the thrusts applied to the quadcopter's rotors.
The task space~$\Theta$ consists of target states $\theta \in \X$ in which the components corresponding to tilt angles and velocities are zero (\ie, only the component corresponding to position varies).
The reward is a negative squared Euclidean distance between the concurrent state~$\xbf_{t}$ and the target state~$\theta$ (\ie, \smash{$\frwd ( \xbf_t , \obf_t , \ubf_t , t ; \theta ) = - \norm{\xbf_t - \theta}_2^2$}).
The agent's policy function~$\pi_\phi ( \cdot )$ is implemented by a three-layer fully connected neural network.
The teacher policy function is either $\pi_{\phiunsafe}(\cdot)$ or~$\pi_{\phisafe}(\cdot)$, where the former steers toward the target state while ignoring safety, whereas the latter steers toward the target state while accounting for safety.
We consider two variants of the setting, one for each notion of safety, with $\pi_{\phiunsafe}(\cdot)$ common to both.
In the risk avoidance variant, the set of banned states $\Xban \subset \X$ consists of all states whose position component lies in a fixed, randomly selected union of cubes.
The difference between $\pi_{\phiunsafe}(\cdot)$ and~$\pi_{\phisafe}(\cdot)$ is that the former ignores~$\Xban$ (\ie, it may lead the state to enter~$\Xban$), whereas the latter makes sure to avoid it.
In the risk handling variant, $\Fdist$~comprises perturbations of the state transition function by a force $\wbf_t \in \BR^{3}$ of bounded magnitude. 
Here, the difference between $\pi_{\phiunsafe}(\cdot)$ and~$\pi_{\phisafe}(\cdot)$ is that the former is trained under the original state transition function (and may thus perform poorly under other members of~$\Fdist$), whereas the latter maximizes the worst-case cumulative reward, \ie, the cumulative reward when the force~$\wbf_t$ is chosen adversarially~\citep{morimoto2005robust,pinto2017robust}, in analogy with the theoretically analyzed setting (\cref{sec:prelim:setup_theory}).
Further details can be found in \cref{app:details:quadcopter}.


\cref{tab:imitate} reports results of the experiment for both variants of the setting.
In line with our theory, it shows that although imitating safe and unsafe teachers for target states seen in training is similarly straightforward, generalizing across target states is considerably more difficult with a safe teacher.

\begin{figure}[t]
    \begin{center}
        \vspace{-4mm}
        \begin{center}
            \includegraphics[width=1\textwidth]{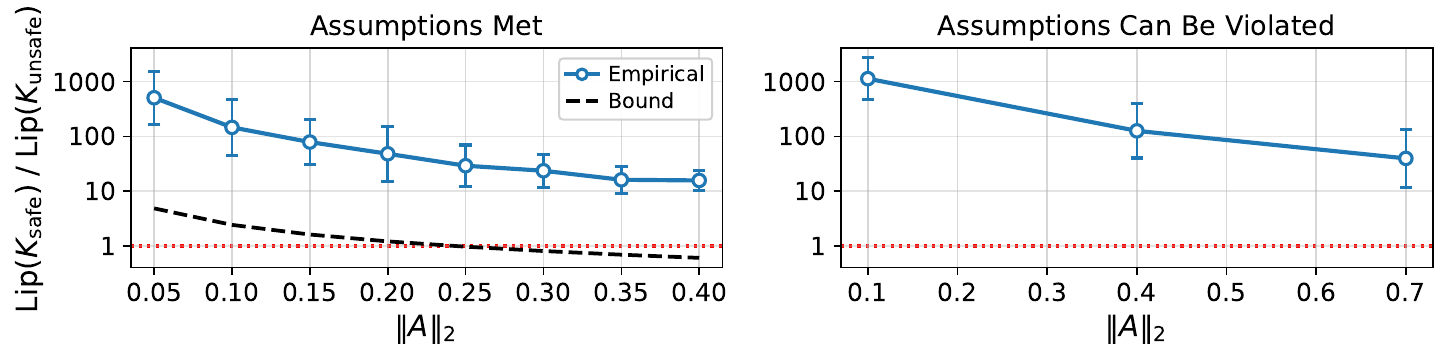}
        \end{center}
        \vspace{-2mm}
    \end{center}
    \caption{
    In the theoretically analyzed setting (linear-quadratic control with $\Hinf$-robustness; \cref{sec:prelim:setup_theory}), the mapping from task specification to an optimal controller has higher Lipschitz constant with safety requirements than without, \ie, the Lipschitz constant of~$\Ksafe ( \cdot )$ is greater than that of~$\Kunsafe ( \cdot )$.
    Left plot compares the estimated value of the Lipschitz ratio $\lip ( \Ksafe ) / \lip ( \Kunsafe )$ to its lower bound established by \cref{res:sep}, when \cref{ass:stable,ass:commute,ass:regular} are met.
    For this plot the state, input, noise and action reward matrices $A$, $B$, $D$ and~$R$, respectively, are drawn from a random distribution that ensures:
    \emph{(i)}~\cref{ass:stable,ass:commute,ass:regular} are met;
    and
    \emph{(ii)}~the common term in the denominator of the lower bound from \cref{res:sep}, \ie, $\| A \|_{2}$, varies, while the remaining terms ($\| B \|_{2}$, $\| R^{-1} \|_{2}$ and~$\alignfactor$) are fixed.
    Right plot differs from left in that the random distribution it uses does not ensure \emph{(i)} and~\emph{(ii)} above, and accordingly, the estimated value of the Lipschitz ratio is shown without the lower bound (the latter is not applicable without \cref{ass:stable,ass:commute,ass:regular}).
    Both plots use dimension $\dstate = 4$ (low enough for Lipschitz constants to be estimated efficiently), and report standard deviations via error bars.
    Notice that the Lipschitz ratio is consistently greater than one, and that its lower bound, although not tight, captures the correct dependence on~$\| A \|_{2}$.
    For further results and details see \cref{app:exper,app:details:lip}, respectively.
    }
    \label{fig:lip_dim=4}
    \vspace{-4mm}
\end{figure}

\vspace{-2mm}
\subsection{CRM via LLM Agent}
\label{sec:exper:LLM}
\vspace{-1mm}

We demonstrate empirically that the conclusions of our theory (\cref{sec:analysis}) extend to a modern agentic setting in which an LLM agent (LLaMA-3.2 model~\citep{meta2024llama32}) is fine-tuned to perform data-handling and record-creation tasks in a realistic CRM benchmark (adaptation of \texttt{ST-WebAgentBench}~\citep{levy2026st}), while avoiding situations marked as high-risk and handling events defined as a risk materializing.
This setting, implemented using a text-based simulated web browser, is a special case of the formulation in \cref{sec:prelim:setup}, characterized as follows.
An observation~$\obf_{t}$ represents the current browser page (in AXTree format).
An action~$\ubf_{t}$ is a token sequence representing a user command (\eg, a click of a button or a text input). 
The task space~$\Theta$ consists of natural language descriptions of different CRM task templates (\eg, ``create an account under the name [NAME] with the email [EMAIL]'').
The reward is a binary score which is returned when the action is a terminal command, and indicates whether the specified task was completed successfully.
The agent's policy function is implemented by a pretrained LLaMA-3.2-1B-Instruct model~\citep{meta2024llama32}, which receives in its context window the CRM task description, the current browser page, the past browser pages, and the past user commands.
Safety is captured via standard CRM requirements that cover both risk avoidance (\eg, avoiding unmonitored high-risk actions such as saving personal information without notifying the user) and risk handling (\eg, detecting materialized input risks such as invalid emails and warning the user before proceeding).
The teacher policy function is implemented by GPT-5.2~\citep{openai2025gpt52}, either with safety requirements in its context window (safe teacher) or without (unsafe teacher).\note{%
We chose GPT-5.2 for the teacher policy function due to its high quality, and LLaMA-3.2-1B-Instruct for the agent's policy function since it is open source.
These choices introduce a slight deviation from the formulation in \cref{sec:prelim:setup}, where both policy functions are implemented by the same parametric mapping.
}
For further details see \cref{app:details:LLM}.

\cref{tab:imitate} reports results of the experiment.
In line with our theory, it shows that although imitating safe and unsafe teachers on CRM tasks seen in training is similarly straightforward, generalizing across CRM tasks is considerably more difficult with the safe teacher.

\begin{table}[t]
    \vspace{-4mm}
    \caption{
In line with our theory, even when imitating safe and unsafe teachers on tasks seen in training is similarly straightforward, generalizing across tasks is considerably more difficult with the safe teacher.
Each row in the table reports imitation errors on tasks seen in training~($\Thetatr$) and on tasks unseen in training~($\Thetate$), for an agent trained to imitate a safe or unsafe teacher.
First four rows correspond to the theoretically analyzed setting (linear-quadratic control with $\Hinf$-robustness; \cref{sec:prelim:setup_theory}), in the infinite sample limit (where imitation errors on~$\Thetatr$ are zero by definition) and in the finite sample regime.
The dimension used for these rows is $\dstate = 10$; for further details see \cref{app:details:lqr}.
Fifth through seventh rows correspond to the quadcopter navigation setting described in \cref{sec:exper:quadcopter}, and last two rows to the CRM via LLM agent setting described in \cref{sec:exper:LLM}; for further details see \cref{app:details:quadcopter,app:details:LLM}, respectively.
Throughout the table, imitation errors are reported as mean $\pm$ standard deviation over independent random trials.
We highlight in boldface an imitation error if it is lower than its counterpart by a statistically significant margin (\ie, if it is below its counterpart minus the sum of their standard deviations).
For further results see \cref{app:exper}.
    }
    \centering
    \small
    \setlength{\tabcolsep}{10pt}
    \begin{tabular}{llcc}
    \toprule
    \textbf{Setting} & \textbf{Teacher} & \textbf{$\Thetatr$ Error} & \textbf{$\Thetate$ Error} \\
    \midrule
    \multirow{2}{*}{Linear-Quadratic, Infinite Sample} & Safe   & $-$ & $1265.53\pm 310.44$ \\
                                 & Unsafe & $-$ & $\mathbf{14.16 \pm 2.43}$ \\
    \midrule
    \multirow{2}{*}{Linear-Quadratic, Finite Sample} & Safe   & $2.16 \pm 0.83$ & $4018.10 \pm 604.55$ \\
                         & Unsafe & $1.36 \pm 0.99$ & $\mathbf{33.76 \pm 9.61}$ \\
    \midrule
    \multirow{3}{*}{Quadcopter Navigation} & Safe (Risk Avoidance) & $0.61 \pm 0.27$ & $144.69 \pm 15.57$ \\
                                 & Safe (Risk Handling) & $0.71 \pm 0.25$ & $194.16 \pm 35.73$ \\
                                 & Unsafe               & $0.33 \pm 0.14$ & $\mathbf{24.81 \pm 5.20}$ \\
    \midrule
    \multirow{2}{*}{CRM via LLM Agent} & Safe   & $16.48 \pm 3.46$ & $25.94 \pm 4.10$ \\
                         & Unsafe & $15.77 \pm 2.61$ & $\mathbf{16.93 \pm 2.14}$ \\
    \bottomrule
    \end{tabular}
    \label{tab:imitate}
    \vspace{-3mm}
\end{table}

	\vspace{-3mm}
\section{Limitations}
\label{sec:limit}
\vspace{-2mm}

It is important to acknowledge several limitations of our work.
First, we consider incorporating safety into a learned agent only via imitation learning, and while this approach is common, there are other approaches that we do not account for, \eg, those based on direct interaction with an environment.
Second, our theoretical analysis is restricted to linear-quadratic control, and as a notion of safety it considers only risk handling, not risk avoidance.
Third, the theoretically analyzed linear-quadratic control setting includes as its task parameter~$\theta$ only the state reward matrix~$Q$, not the action reward matrix~$R$ (the latter is regarded as fixed).
Fourth, our theory considers training in the infinite sample limit, and while this ensures that the established difficulty of safety to generalize across tasks is due to inherent properties of safety, it leaves finite sample effects theoretically unexplored.
Fifth, our theoretical results rely on a stability margin condition (\cref{ass:stable}), and, in the case of \cref{res:sep}, on a symmetric commutativity condition (\cref{ass:commute}) and a mild alignment regularity condition (\cref{ass:regular}).
Sixth, rather than directly characterizing generalization across tasks, our theoretical results use the Lipschitz constant as a surrogate, relying on existing literature to bridge the gap.
Seventh, there are cases where the inequality established by \cref{res:sep} (\cref{eq:sep}) does not guarantee that the Lipschitz constant of~$\Ksafe ( \cdot )$ is greater than that of~$\Kunsafe ( \cdot )$.

Our experiments go beyond the theoretical limitations above.
Indeed, \cref{sec:exper:lqr} demonstrates that in the theoretically analyzed linear-quadratic control setting, the estimated Lipschitz~constant of~$\Ksafe ( \cdot )$ is consistently greater than that of~$\Kunsafe ( \cdot )$, even when the conditions of \cref{res:sep} (\cref{ass:commute,ass:stable,ass:regular}) are violated.
\cref{sec:exper:lqr} also demonstrates that the separation between Lipschitz constants translates to a separation between cross-task generalization errors, whether training occurs in the infinite sample limit or in the finite sample regime. 
Finally, \cref{sec:exper:quadcopter,sec:exper:LLM} demonstrate that the separation between generalization errors extends to settings beyond linear-quadratic control, with risk avoidance in addition to risk handling.
Extending our theoretical analysis to match the generality of these empirical demonstrations is an important direction for future work.

	\vspace{-3mm}
\section{Related Work}
\label{sec:related}
\vspace{-2mm}

Multi-task agentic settings have been studied in RL and control theory, from early works on hierarchical and task-conditioned decision making~\citep{dayan1993feudal,thrun1998learning,sutton1999between,dietterich2000hierarchical,baxter2000model}, to deep learning-era works on meta-RL~\citep{duan2016rl2,wang2017learning,finn2017model,rakelly2019efficient,zintgraf2020varibad,yu2020meta}.
Recently, multi-task agentic settings have been studied with transformers and LLMs, where an agent receives task specifications and environment observations in its context window, and produces actions via token sequences~\citep{melo2022transformers,laskin2023context,schick2023toolformer,shen2023hugginggpt,valmeekam2023planbench,valmeekam2023planning,deng2023mind2web,jimenez2024swe,wang2024voyager,zhou2024webarena,liu2024agentbench,mialon2024gaia,yang2024swe,qian2024chatdev,levy2026st,liu2026paying,ran2026outcome}.
Of particular relevance to our theoretical analysis is a line of work studying generalization in multi-task LQR, where a single controller is shared across tasks~\citep{stamouli2025policy,fallah2026adversarially,pan2026model,toso2026multitask}. 
In contrast, we study the map from a task's specification to its optimal controller, reflecting the practical scenario where an agent is told which task to perform.

Safety is an important pillar of RL and control theory~\citep{zhou1996robust,morimoto2005robust,nilim2005robust,bacsar2008h,lim2013reinforcement,garcia2015comprehensive,achiam2017constrained,dalal2018safe,tessler2019reward,ray2019benchmarking,altman2021constrained,luo2021mesa,zhang2021derivative,xu2022trustworthy,liu2023robustness,greenberg2023train,kim2023learning,khattar2023cmdp,cho2024constrained,toso2024meta,guan2024costaware,zubia2025robust,xu2025efficient,ni2026constrained}, where it is often studied through the notions of risk avoidance and risk handling: risk avoidance is typically formulated as constraint satisfaction~\citep{achiam2017constrained,dalal2018safe,tessler2019reward,ray2019benchmarking,altman2021constrained,luo2021mesa,kim2023learning,khattar2023cmdp,cho2024constrained,xu2025efficient,ni2026constrained}, whereas risk handling is typically formulated as robustness to model mismatch and adversarial attacks~\citep{zhou1996robust,morimoto2005robust,nilim2005robust,bacsar2008h,lim2013reinforcement,zhang2021derivative,liu2023robustness,greenberg2023train,toso2024meta,guan2024costaware,zubia2025robust}.
Safety is also central to LLM agents~\citep{gehman2020realtotoxicityprompts,dai2024safe,debenedetti2024agentdojo,zhang2024agent,yin2024safeagentbench,guo2024redcode,levy2026st}, where risk avoidance is often formulated via prohibited actions~\citep{gehman2020realtotoxicityprompts,bai2022constitutional,bai2022training,dai2024safe,zhang2024safetybench,yin2024safeagentbench}, and risk handling is often formulated via requirements to appropriately respond to risky interactions with humans or tools~\citep{debenedetti2024agentdojo,zhang2024agent,guo2024redcode,levy2026st}.
Incorporating safety into a learned agent is often approached via imitation learning~\citep{bai2022training,ouyang2022training,dai2024safe,swamy2024minimaximalist}, but there are other approaches as well, \eg, ones based on direct interaction with an environment~\citep{achiam2017constrained,dalal2018safe,yang2021safe,cho2024constrained}.

Generalization across tasks is a prominent area of study in multi-task agentic AI~\citep{packer2018assessing,cobbe2019quantifying,cobbe2020leveraging,shah2022goal,liu2026paying}.
Mounting empirical evidence suggests that, in RL and control theory as well as with LLM agents, safe behavior often fails to generalize across tasks, even in cases where the agent's ability to execute does generalize~\citep{kenton2019generalizing,ren2024codeattack,mou2024sg,andriushchenko2025agentharm,zinjad2025can,kwon2026safety}.
To our knowledge, this paper is the first to provide formal support for such empirical evidence: it establishes theoretical results suggesting that \emph{generalization across tasks is fundamentally more difficult with safety requirements than without, regardless of how well safety requirements are met on tasks seen in training}.

	\vspace{-3mm}
\section{Conclusion}
\label{sec:conclusion}
\vspace{-2mm}

AI agents are increasingly deployed in multi-task settings, where the task to be performed is specified only at test time, and successful deployment hinges on generalization to previously unseen tasks.
Empirical evidence suggests that in such settings, even when the ability to execute generalizes to unseen tasks, the ability to do so safely frequently does not.
This paper provides theory and experiments indicating that failures of agentic safety to generalize across tasks are not merely due to limitations of current training techniques, but can also reflect a fundamental property of safety itself: the relationship between a task and its safe execution can be inherently more complex than the relationship between a task and its execution alone.

Our findings carry two important practical implications: 
\emph{(i)}~generalization to unseen tasks of an agent's ability to execute does not imply the same for its ability to do so safely;
and
\emph{(ii)}~the ability to safely execute tasks seen in training, even when it extends to unseen trajectories, does not imply the same for unseen tasks.
Cross-task generalization of safety should therefore be treated as a distinct problem in its own right, separate from cross-task generalization of execution ability and from in-task generalization of safety.

Future work should address the limitations discussed in \cref{sec:limit}.
More broadly, we believe that current efforts to enhance agentic safety through more training on a subset of tasks are insufficient, and it is essential to develop fundamentally different approaches.
A promising possibility suggested by our findings is to learn representations for task specifications and safety characteristics which simplify the relationship between them.
We hope that our work will spur progress along this line.

	\ifdefined\NEURIPS
		\begin{ack}
			This work was supported by the European Research Council (ERC) grant NN4C 101164614, a Google Research Scholar Award, a Google Research Gift, Meta, the Yandex Initiative in Machine Learning, the Israel Science Foundation (ISF) grant 1780/21, the Tel Aviv University Center for AI and Data Science, the Adelis Research Fund for Artificial Intelligence, Len Blavatnik and the Blavatnik Family Foundation, and Amnon and Anat Shashua.
Part of this research was performed while YS was visiting the Institute for Mathematical and  Statistical Innovation (IMSI), which is supported by the National Science Foundation (Grant No. DMS 2425650).
YS was supported in part by the Prof. R. Rahamimoff Travel Grant for Young Scientists of the U.S.-Israel Binational Science Foundation (BSF).
		\end{ack}
	\else
		\newcommand{\ack}{This work was supported by the European Research Council (ERC) grant NN4C 101164614, a Google Research Scholar Award, a Google Research Gift, Meta, the Yandex Initiative in Machine Learning, the Israel Science Foundation (ISF) grant 1780/21, the Tel Aviv University Center for AI and Data Science, the Adelis Research Fund for Artificial Intelligence, Len Blavatnik and the Blavatnik Family Foundation, and Amnon and Anat Shashua.
Part of this research was performed while YS was visiting the Institute for Mathematical and  Statistical Innovation (IMSI), which is supported by the National Science Foundation (Grant No. DMS 2425650).
YS was supported in part by the Prof. R. Rahamimoff Travel Grant for Young Scientists of the U.S.-Israel Binational Science Foundation (BSF).}
	\fi
	\ifdefined\ARXIV
		\section*{Acknowledgements}
		This work was supported by the European Research Council (ERC) grant NN4C 101164614, a Google Research Scholar Award, a Google Research Gift, Meta, the Yandex Initiative in Machine Learning, the Israel Science Foundation (ISF) grant 1780/21, the Tel Aviv University Center for AI and Data Science, the Adelis Research Fund for Artificial Intelligence, Len Blavatnik and the Blavatnik Family Foundation, and Amnon and Anat Shashua.
Part of this research was performed while YS was visiting the Institute for Mathematical and  Statistical Innovation (IMSI), which is supported by the National Science Foundation (Grant No. DMS 2425650).
YS was supported in part by the Prof. R. Rahamimoff Travel Grant for Young Scientists of the U.S.-Israel Binational Science Foundation (BSF).
	\else
		\ifdefined\COLT
			\acks{This work was supported by the European Research Council (ERC) grant NN4C 101164614, a Google Research Scholar Award, a Google Research Gift, Meta, the Yandex Initiative in Machine Learning, the Israel Science Foundation (ISF) grant 1780/21, the Tel Aviv University Center for AI and Data Science, the Adelis Research Fund for Artificial Intelligence, Len Blavatnik and the Blavatnik Family Foundation, and Amnon and Anat Shashua.
Part of this research was performed while YS was visiting the Institute for Mathematical and  Statistical Innovation (IMSI), which is supported by the National Science Foundation (Grant No. DMS 2425650).
YS was supported in part by the Prof. R. Rahamimoff Travel Grant for Young Scientists of the U.S.-Israel Binational Science Foundation (BSF).}
		\else
			\ifdefined\CAMREADY
				\ifdefined\ICLR
					\newcommand*{\subsuback}{}
				\fi
				\ifdefined\NEURIPS
				\else
					\section*{Acknowledgements}
					This work was supported by the European Research Council (ERC) grant NN4C 101164614, a Google Research Scholar Award, a Google Research Gift, Meta, the Yandex Initiative in Machine Learning, the Israel Science Foundation (ISF) grant 1780/21, the Tel Aviv University Center for AI and Data Science, the Adelis Research Fund for Artificial Intelligence, Len Blavatnik and the Blavatnik Family Foundation, and Amnon and Anat Shashua.
Part of this research was performed while YS was visiting the Institute for Mathematical and  Statistical Innovation (IMSI), which is supported by the National Science Foundation (Grant No. DMS 2425650).
YS was supported in part by the Prof. R. Rahamimoff Travel Grant for Young Scientists of the U.S.-Israel Binational Science Foundation (BSF).
				\fi
			\fi
		\fi
	\fi
	

	\section*{References}
	{\small
		\ifdefined\ICML
			\bibliographystyle{icml2025}
		\else
			\bibliographystyle{plainnat}
		\fi
		\bibliography{refs}
	}

	\clearpage
	\appendix
	\crefalias{section}{appendix}
	\crefalias{subsection}{appendix}
	\crefalias{subsubsection}{appendix}
	\onecolumn
	
	\ifdefined\ENABLEENDNOTES
		\theendnotes
	\fi
	


	\section{Theoretical Background}
\label{app:background}

This appendix presents classical background on the unsafe/LQR and safe/$\Hinf$-robust mappings used throughout the paper. 
We refer to \cref{sec:prelim:setup_theory} for the corresponding special case of imitation learning, where the unsafe teacher ignores the risk-handling requirement and the safe teacher enforces it through $\Hinf$-robustness. 
We adopt the more familiar and equivalent control-theoretic objective formulations corresponding to the reward-based objectives. 
This leaves these mappings unchanged while rewriting the problems in the familiar minimization form. 
Accordingly, although \cref{sec:prelim:setup_theory} introduces $Q$ and $R$ as reward matrices, in this appendix we use the standard control-theoretic terminology and refer to them as cost matrices.

\subsection{Linear-Quadratic Control}
\label{app:background:lqr}

Consider the linear-quadratic setting from~\cref{sec:prelim:setup_theory}, with dynamics dictated by~\cref{eq:lin_dyn} and initial state $\xbf_0=\0$, where the risk-handling requirement is ignored. 
An objective formulation equivalent to the maximization problem presented there is the minimization of the expected quadratic cost given by
\begin{align}
    \label{eq:quad_cost}
    \EE\left[
    \sum_{t=0}^{\infty} \xbf_t^\top Q \xbf_t + \ubf_t^\top R \ubf_t
    \right]
    \text{\,,}
\end{align}
where the initial state is fixed at $\xbf_0=\0$, the states $\xbf_t$ evolve according to the dynamics defined in~\cref{eq:lin_dyn}, and the expectation is taken over the disturbance sequence $\{\wbf_t\}_{t=0}^{\infty}$. 
Since the disturbances are independent and have zero means and finite covariances, certainty-equivalence results for finite-horizon linear-quadratic problems imply that these disturbances do not change the minimizing policy when the system is noise-free: they affect the attained value, but not the minimizer (see \eg~\citep{zhou1996robust,bacsar2008h,green2012linear,bertsekas2019reinforcement}).
In order to utilize these results for the infinite-horizon case, we assume the summability of disturbance second moments, \ie, that $\EE[\sum_{t=0}^{\infty}\norm{\wbf_t}_2^2] \leq c$ for some constant $c\in\BR_{>0}$. 
At a high level, this assumption is used to exclude cases where the disturbances alone force the infinite-horizon cost to be infinite, regardless of the policy under consideration.
Without an integrability assumption of this kind, the infinite-horizon stochastic cost may become infinite in many nondegenerate systems of interest, making the comparison between policies vacuous.
Combined with the certainty-equivalence note above, the infinite-horizon stochastic problem is well-posed and has the same minimizer as the corresponding noise-free system; see \eg~\citep{kwakernaak1972linear,anderson2007optimal,bertsekas2019reinforcement} for exact statements and formal justifications.

Classical results in control theory state that the optimal policy for the infinite-horizon cost in a noise-free system is given by the linear policy $\pi_{\phiunsafe}(\xbf;Q)=\Klqr(Q)\xbf$, where the LQR $\Klqr(Q)$ is characterized by the discrete-time algebraic Riccati equation (DARE) (see \eg~\citep{zhou1996robust,bacsar2008h,green2012linear,bertsekas2019reinforcement}).
Concretely, define for any symmetric matrix $P\in\R^{\dstate,\dstate}$ the matrix $X(P)$ as
\begin{align}
    \label{eq:def_X}
    X(P) := R + B^\top P B
    \text{\,.}
\end{align}
Then the PSD matrix $\Plqr(Q)$ is characterized as the (unique, stabilizing) solution to the discrete-time algebraic Riccati equation (DARE) given by
\begin{align}
    \label{eq:dare}
    P = A^\top P A + Q - A^\top P B  X(P)^{-1} B^\top P A
    \text{\,.}
\end{align}
Given $\Plqr(Q)$, the associated LQR $\Klqr(Q)$ is given by
\begin{align}
    \label{eq:klqr}
    \Klqr(Q) := -X\left(\Plqr(Q)\right)^{-1} B^\top \Plqr(Q) A
    \text{\,.}
\end{align}
Under our standing assumptions in \cref{sec:prelim:setup_theory}, classical theory guarantees existence of a unique stabilizing positive semidefinite solution $\Plqr(Q)$ to~\cref{eq:dare}, and so the induced linear policy $\pi_{\phiunsafe}(\xbf;Q)=\Klqr(Q)\xbf$ achieves the unique optimum of the unsafe problem.

\subsection{$\Hinf$-Robust Control}
\label{app:background:hinf}

Consider now the same linear-quadratic setting from \cref{sec:prelim:setup_theory}, but where the risk-handling requirement is enforced. 
For a fixed state cost matrix $Q\in\Theta$, an objective formulation equivalent to the reward-based objective is the minimization of the classical cost-based formulation given by
\begin{align}
    \label{eq:robust_cost}
    \sup_{\substack{\wbf_{t}\\ \sum_{t=0}^{\infty} \norm{\wbf_t}_2^2\leq c}} \sum_{t=0}^{\infty} \xbf_t^\top Q \xbf_t + \ubf_t^\top R \ubf_t
    \text{\,,}
\end{align}
where the initial state is fixed at $\xbf_0=\0$, and the states $\xbf_t$ evolve according to the disturbed dynamics defined by assigning the deterministic disturbances to the dynamics in~\cref{eq:lin_dyn}.
When the controller is linear (or more broadly, $1$-homogeneous in the state), an equivalent unconstrained yet normalized formulation is given by
\begin{align}
    \label{eq:robust_cost_norm}
    \sup_{\wbf_{t}} \frac{\sum_{t=0}^{\infty} \xbf_t^\top Q \xbf_t + \ubf_t^\top R \ubf_t}{\sum_{t=0}^{\infty} \norm{\wbf_t}_2^2}
    \text{\,,}
\end{align}
where the initial state $\xbf_{0}$ and the state evolution are as defined above. 
This will be the formulation we use throughout the remainder of the paper.
Note that in all ratio-based formulations, the supremum is taken over disturbance sequences that are not identically zero.
We next define the feasibility of a robustness level directly in terms of the objective in~\cref{eq:robust_cost_norm}.

\paragraph{Feasibility at level $\gamma$ and the optimal robustness level $\gammainf(Q)$.}
For $\gamma\ge 0$, we say that $\gamma$ is \emph{feasible} for task $Q$ if there exists a sequence of control inputs $(\ubf_{t})_{t=0}^{\infty}$ such that
\begin{align}
    \label{eq:gamma_feas}
    \sup_{\wbf_{t}} \frac{\sum_{t=0}^{\infty} \xbf_t^\top Q \xbf_t + \ubf_t^\top R \ubf_t}{\sum_{t=0}^{\infty} \norm{\wbf_t}_2^2}
    \le
    \gamma^2
    \text{\,,}
\end{align}
where the state evolution is as in \cref{eq:robust_cost_norm}.
The \emph{optimal} (minimal achievable) robustness level for task $Q$ is therefore defined as
\begin{align}
    \label{eq:gamma_inf}
    \gammainf(Q)^2
    :=
    \min_{(\ubf_{t})_{t=0}^{\infty}}
    \sup_{\wbf_{t}} \frac{\sum_{t=0}^{\infty} \xbf_t^\top Q \xbf_t + \ubf_t^\top R \ubf_t}{\sum_{t=0}^{\infty} \norm{\wbf_t}_2^2}
    \text{\,.}
\end{align}
Equivalently, $\gamma$ is feasible for $Q$ if and only if $\gamma \ge \gammainf(Q)$.

\paragraph{Auxiliary level-$\gamma$ game objective.}
A standard route to $\Hinf$ synthesis replaces the ratio formulation by a \emph{level-$\gamma$} objective. 
Concretely, for a fixed $\gamma\ge 0$, consider the control problem obtained by comparing performance against disturbance energy through the quadratic functional
\begin{align}
    \label{eq:level_gamma_game}
    \sup_{\wbf_{t}}
    \sum_{t=0}^{\infty}
    \xbf_t^\top Q \xbf_t + \ubf_t^\top R \ubf_t - \gamma^2 \norm{\wbf_t}_2^2
    \text{\,,}
\end{align}
subject to the disturbed dynamics defined above.
Classical robust control theory shows that the feasibility condition in~\cref{eq:gamma_feas} holds if and only if there exists a solution achieving a finite value in the level-$\gamma$ synthesis problem associated with~\cref{eq:level_gamma_game}. 
Thus, finding a minimal $\gamma$ for which such a solution exists corresponds to finding $\gammainf(Q)$~\citep{zhou1996robust,bacsar2008h,green2012linear}.

\paragraph{DGARE and the Riccati-based $\Hinf$ mapping at feasible $\gamma$.}
Let $P\in\R^{\dstate,\dstate}$ be a PSD matrix. For the level-$\gamma$ dynamic game, define the block matrix $H(P;\gamma)$ and the matrix $F(P)$ as
\begin{align}
    \label{eq:def_HF}
    H(P;\gamma)
    &:=
    \begin{pmatrix}
        R+B^\top P B & B^\top P D\\
        D^\top P B & D^\top P D-\gamma^2 I
    \end{pmatrix},
    \qquad
    F(P)
    :=
    \begin{pmatrix}
        B^\top P A\\
        D^\top P A
    \end{pmatrix}
    \text{\,.}
\end{align}
The discrete-time game algebraic Riccati equation (DGARE) of the level-$\gamma$ problem is given by
\begin{align}
    \label{eq:dgare}
    P = Q + A^\top P A - F(P)^\top H(P;\gamma)^{-1}F(P)
    \text{\,.}
\end{align}
The required definiteness conditions are
\begin{align*}
    R+B^\top P B \succ 0,
    \qquad
    D^\top P D-\gamma^2 I \prec 0,
\end{align*}
so that $H(P;\gamma)$ is nonsingular and the gains are well defined.
Classical $\Hinf$ synthesis results imply that whenever $\gamma$ is feasible for $Q$ (equivalently, $\gamma\ge \gammainf(Q)$), the DGARE in~\cref{eq:dgare} admits a unique stabilizing PSD solution, denoted $\Pinf(Q;\gamma)$, satisfying these definiteness conditions, and yields corresponding optimal level-$\gamma$ control and worst-case disturbance matrices via
\begin{align}
    \label{eq:kinf_level_gamma}
    \begin{pmatrix}
        K_u(Q;\gamma)\\
        K_w(Q;\gamma)
    \end{pmatrix}
    :=
    -H\left(\Pinf(Q;\gamma);\gamma\right)^{-1}F\left(\Pinf(Q;\gamma)\right),
    \qquad
    \Kinf(Q;\gamma):=K_u(Q;\gamma)
    \text{\,.}
\end{align}
The induced control policy $\pi_{\phisafe}(\xbf;Q) = \Kinf(Q;\gamma)\xbf$ achieves robustness level $\gamma$ for task $Q$ in the sense of~\cref{eq:gamma_feas}, while $K_w(Q;\gamma)$ gives the associated worst-case disturbance response in the auxiliary level-$\gamma$ game of~\cref{eq:level_gamma_game}~\citep{zhou1996robust,bacsar2008h,green2012linear}.
In the language of the main text, $\Kinf(Q;\gamma)$ is the safe teacher at robustness level $\gamma$.

\paragraph{Defining $\Kinf(Q)$ via DGARE at $\gammainf(Q)$.}
The ratio problem in~\cref{eq:gamma_inf} may admit multiple optimal matrices attaining $\gammainf(Q)$. 
In our paper, we fix a canonical choice for the safe teacher by selecting the DGARE-synthesized mapping given for the minimal robustness level $\gammainf(Q)$:
\begin{align}
    \label{eq:kinf}
    \Pinf(Q) := \Pinf(Q;\gammainf(Q))\,,
    \qquad
    \Kinf(Q) := \Kinf\left(Q;\gammainf(Q)\right)
    \text{\,,}
\end{align}
and we use the corresponding safe/$\Hinf$-robust mapping $Q\mapsto \Kinf(Q)$ as introduced in \cref{sec:prelim:setup_theory}.
This definition is well-posed whenever $\gammainf(Q)$ is feasible and the DGARE stabilizing solution is unique, as guaranteed by the classical synthesis theorem cited above.

	\section{Proof of \cref{res:ub}}
\label{app:lqr}

In this appendix, we prove \cref{res:ub}, namely the Lipschitz upper bound for the unsafe/LQR mapping $Q\mapsto \Klqr(Q)$. The outline of the proof is as follows. \cref{app:lqr:fixed_point} introduces the operator induced by the DARE (\cref{eq:dare}) and shows that its unique fixed point $\Plqr(Q)$ is of bounded norm for any $Q\in\Theta$. \cref{app:lqr:riccati_bound} proceeds to give a bound on the Lipschitz constant of the mapping $Q\mapsto \Plqr(Q)$ induced by the Riccati operator. \cref{app:lqr:controller_bound} concludes by transitioning from the Lipschitz bound on the Riccati solution mapping to one on the unsafe mapping $\Klqr(Q)$.
In contrast to classical DARE perturbation analyses, which typically linearize around a nominal solution or obtain implicit non-local bounds via Lyapunov majorants~\citep{konstantinov1993perturbation}, our argument below obtains an explicit bound that is uniform over $\Theta$.

\subsection{Riccati Operator Fixed Point}
\label{app:lqr:fixed_point}

We begin by defining the Riccati operator induced by the DARE (\cref{eq:dare}).
\begin{definition}
    \label{def:ric_op}
    For any $Q\in\Theta$, we denote the \emph{Riccati operator} by $\T_{Q}:\BR^{\dstate,\dstate}\to\BR^{\dstate,\dstate}$ and define it for any $P\in\BR^{\dstate,\dstate}$ by
    \begin{align*}
        \T_{Q}(P) := A^\top P A + Q - A^{\top} P B X(P)^{-1} B^\top P A
        \text{\,,}
    \end{align*}
    where $X(P)$ is as defined in \cref{eq:def_X}.
\end{definition}

Our focus will be on the unique fixed point of the Riccati operator. For completeness, we give the formal notion of fixed points.
\begin{definition}
    \label{def:fixed_point}
    Let $\T:\BR^{\dstate,\dstate}\to\BR^{\dstate,\dstate}$. We say that $P\in\BR^{\dstate,\dstate}$ is a \emph{fixed point} of $\T$ when
    \begin{align*}
        \T(P) = P
        \text{\,.}
    \end{align*}
\end{definition}

Importantly, any solution to the DARE (\cref{eq:dare}) is a fixed point of the respective Riccati operator. A classical result in control theory~\citep{zhou1996robust,bacsar2008h,green2012linear} states that under our standing assumptions from \cref{sec:prelim:setup_theory} (specifically, that $\Anorm<1$), the DARE has a unique stabilizing solution for any $Q\in\Theta$, which is used to define the unsafe/LQR teacher.
To obtain a Lipschitz bound on the mapping to these solutions, we'll characterize these unique fixed points in a form that is more amenable to analysis.
We begin by showing that for any $Q\in\Theta$, the unique fixed point of the Riccati operator $\T_{Q}(\cdot)$ lies within the set $\PP$ defined as
\begin{align}
    \label{eq:def_PP}
    \PP := \left\lbrace P\in\BR^{\dstate,\dstate} : P\succeq 0 \text{ and } \norm{P}_{2} \leq \frac{1}{1-\Anorm^{2}}\right \rbrace
    \text{\,.}
\end{align}

The following two lemmas establish that the Riccati operators self-map on $\PP$. First, we show that the Riccati operator yields PSD matrices on $\PP$.
\begin{lemma}
    \label{lemma:ric_op_psd}
    For any $Q\in\Theta$ it holds that the Riccati operator $\T_{Q}(\cdot)$ (\cref{def:ric_op}) satisfies
    \begin{align*}
        \T_{Q}(P)\succeq 0
    \end{align*}
    for any $P\in\PP$.
\end{lemma}
\begin{proof}
    Consider the block matrix $M$ defined as
    \begin{align*}
        M=
        \begin{pmatrix}
            A^{\top}PA+Q\quad &A^{\top}PB\\
            B^{\top}PA\quad &X(P)
        \end{pmatrix}
        =
        \begin{pmatrix}
            A^{\top}PA+Q\quad &A^{\top}PB\\
            B^{\top}PA\quad & R+B^{\top}PB
        \end{pmatrix}
        \text{\,.}
    \end{align*}
    It holds that
    \begin{align*}
        M=
        \begin{pmatrix}
            A^{\top}\\
            B^{\top}
        \end{pmatrix}
        P
        \begin{pmatrix}
            A\quad B
        \end{pmatrix}
        \;+\;
        \begin{pmatrix}
            Q^{\frac{1}{2}}\\
            \0
        \end{pmatrix}
        \begin{pmatrix}
            Q^{\frac{1}{2}}\quad \0
        \end{pmatrix}
        \;+\;
        \begin{pmatrix}
            \0\\
            R^{\frac{1}{2}}
        \end{pmatrix}
        \begin{pmatrix}
            \0 \quad R^{\frac{1}{2}}
        \end{pmatrix}
        \text{\,,}
    \end{align*}
    where $\0$ denotes the all-zeros matrix of appropriate dimensions.
    Since $P,Q$ and $R$ are all PSD, the matrix $M$ is also PSD as a sum of PSD matrices. Additionally, the matrix $X(P)$ (as defined in \cref{eq:def_X}) is PD as a sum of a PD matrix and a PSD matrix. Therefore by \cref{lemma:Schur_comp}, the Schur complement of the block $X(P)$ of $M$, given by
    \begin{align*}
        A^{\top}PA+Q-A^{\top}PBX(P)^{-1}B^{\top}PA=\T_{Q}(P)
        \text{\,,}
    \end{align*}
    is PSD, thus completing the proof.
\end{proof}

Next, we show that the Riccati operators also yield norm-bounded matrices on $\PP$.
\begin{lemma}
    \label{lemma:ric_op_bound}
    For any $Q\in\Theta$ it holds that the Riccati operator $\T_{Q}(\cdot)$ (\cref{def:ric_op}) satisfies
    \begin{align*}
        \norm{\T_{Q}(P)}_{2} \leq \frac{1}{1-\Anorm^{2}}
        \text{\,,}
    \end{align*}
    for any $P\in\PP$.
\end{lemma}
\begin{proof}
    Since $P$ is PSD, the matrix $A^{\top}PA+Q$ is PSD as a sum of PSD matrices, and the matrix $\T_{Q}(P)$ is PSD by \cref{lemma:ric_op_psd}.
    Next, since $R$ is PD, the matrix $X(P)^{-1}=(R+B^{\top}PB)^{-1}$ is PSD and so is the matrix $A^{\top}PBX(P)^{-1}B^{\top}PA$.
    Therefore, noting that rearranging the terms in the expression for $\T_{Q}(P)$ yields
    \begin{align*}
        A^{\top}PBX(P)^{-1}B^{\top}PA=A^{\top}PA+Q-\T_{Q}(P)
        \text{\,,}
    \end{align*}
    we obtain by \cref{lemma:psd_order} and the triangle inequality that
    \begin{align*}
        \norm{\T_{Q}(P)}_{2}\leq \norm{A^{\top}PA+Q}_{2} \leq \norm{A^{\top}PA}_{2}+\norm{Q}_{2}
        \text{\,.}
    \end{align*}
    Lastly, recalling that $\norm{Q}_{2}\leq \norm{Q}_{F}\leq 1$ and $\norm{P}_{2}\leq \frac{1}{1-\Anorm^{2}}$ we obtain by the Cauchy-Schwarz inequality that
    \begin{align*}
        \norm{\T_{Q}(P)}_{2}\leq 1+\frac{\Anorm^{2}}{1-\Anorm^{2}}=\frac{1}{1-\Anorm^{2}}
    \end{align*}
    as required.
\end{proof}

\cref{lemma:ric_op_psd,lemma:ric_op_bound} establish that the Riccati operator $\T_{Q}(\cdot)$ (\cref{def:ric_op}) is a self-map on $\PP$:
\begin{corollary}
    \label{corollary:ric_op_self_map}
    For any $Q\in\Theta$ it holds that $\T_{Q}(P)\in\PP$ for any $P\in\PP$.
\end{corollary}

\subsection{Lipschitz Bound on Riccati Solution}
\label{app:lqr:riccati_bound}

We continue from \cref{app:lqr:fixed_point} by showing that the Riccati operator is a contraction on $\PP$. 
This is done to establish that its unique fixed point is also within $\PP$, later allowing us to derive a Lipschitz bound on the mapping from task $Q$ to the unique fixed point of its corresponding Riccati operator. 
We begin by formally defining the notion of a contraction.
\begin{definition}
    \label{def:cont}
    Let $\left(\Y, \dist(\cdot,\cdot)\right)$ be a metric space. We say that $f:\Y\to\Y$ is a \emph{contraction} on the metric space $\left(\Y, \dist(\cdot,\cdot)\right)$ with contraction constant $\gamma\in[0,1)$ when it holds that
    \begin{align*}
        \dist\left(f(\ybf_{1}), f(\ybf_{2})\right)\leq \gamma\dist(\ybf_{1}, \ybf_{2})
        \text{\,,}
    \end{align*}
    for any $\ybf_{1}, \ybf_{2}\in\Y$ with $\ybf_{1}\neq \ybf_{2}$.
\end{definition}

Before moving on to analyze the contraction properties of the Riccati operator $\T_{Q}(\cdot)$, we translate \cref{ass:stable} into a more amenable form used later in the proof.
\begin{lemma}
    \label{lemma:norm_A_transform}
    If \cref{ass:stable} holds, then it holds that
    \begin{align*}
        \gamma:=\Anorm^{2}\left(1+\frac{\Bnorm^{2}\Rinvnorm}{1-\Anorm^{2}}\right)^{2} < \frac{1}{2}
        \text{\,.}
    \end{align*}
\end{lemma}
\begin{proof}
    By \cref{ass:stable} we have that
    \begin{align*}
        \Anorm &< 1 - \frac{1}{1 + \big( \sqrt{2} + \sqrt{2} \Bnorm^{2} \Rinvnorm \big)^{-1}}\\
        &=1 - \frac{\sqrt{2} + \sqrt{2} \Bnorm^{2} \Rinvnorm}{\sqrt{2} + \sqrt{2} \Bnorm^{2} \Rinvnorm+1}\\
        &=\frac{1+\sqrt{2} + \sqrt{2} \Bnorm^{2} \Rinvnorm- (\sqrt{2} + \sqrt{2} \Bnorm^{2} \Rinvnorm)}{1+\sqrt{2} + \sqrt{2} \Bnorm^{2} \Rinvnorm}\\
        &=\frac{1}{1+\sqrt{2} + \sqrt{2} \Bnorm^{2} \Rinvnorm}\\
        &=\frac{1/\sqrt{2}}{1+1/\sqrt{2}+\Bnorm^{2}\Rinvnorm}
        \text{\,.}
    \end{align*}
    Therefore,
    \begin{align*}
        \left(1+\frac{1}{\sqrt{2}}+\Bnorm^{2}\Rinvnorm\right)\Anorm<\frac{1}{\sqrt{2}}
        \text{\,,}
    \end{align*}
    which implies
    \begin{align*}
        \Anorm\left(1+\Bnorm^{2}\Rinvnorm\right)<\frac{1-\Anorm}{\sqrt{2}}
        \text{\,.}
    \end{align*}
    Subtracting $\Anorm^{2}$ from the left-hand side yields
    \begin{align*}
        \Anorm\left(1-\Anorm+\Bnorm^{2}\Rinvnorm\right)<\frac{1-\Anorm}{\sqrt{2}}
        \text{\,.}
    \end{align*}
    Dividing by $1-\Anorm>0$ gives
    \begin{align*}
        \Anorm\left(1+\frac{\Bnorm^{2}\Rinvnorm}{1-\Anorm}\right)<\frac{1}{\sqrt{2}}
        \text{\,.}
    \end{align*}
    Since $1-\Anorm^{2}=(1-\Anorm)(1+\Anorm)>1-\Anorm$, we have
    \begin{align*}
        \frac{\Bnorm^{2}\Rinvnorm}{1-\Anorm^{2}}<\frac{\Bnorm^{2}\Rinvnorm}{1-\Anorm}
        \text{\,.}
    \end{align*}
    Hence
    \begin{align*}
        \Anorm\left(1+\frac{\Bnorm^{2}\Rinvnorm}{1-\Anorm^2}\right)<\frac{1}{\sqrt{2}}
        \text{\,.}
    \end{align*}
    Finally, since both sides are positive, squaring preserves the inequality yielding that
    \begin{align*}
        \Anorm^2\left(1+\frac{\Bnorm^{2}\Rinvnorm}{1-\Anorm^2}\right)^2<\frac{1}{2}
    \end{align*}
    as required.
\end{proof}

Having established the condition implied by \cref{ass:stable}, we turn to analyzing the contraction properties of the Riccati operator itself.
\cref{lemma:ric_op_cont} below proves that it is a contraction on the metric space given by the set $\PP$ equipped with the Frobenius distance metric.
\begin{lemma}
    \label{lemma:ric_op_cont}
    Let $\dist_{\PP}:\PP\times\PP\to\R_{\geq 0}$ be the Frobenius distance defined for any $P_{1}, P_{2}\in\PP$ as
    \begin{align*}
        \dist_{\PP}(P_{1}, P_{2}) := \norm{P_{1}-P_{2}}_{F}
        \text{\,.}
    \end{align*}
    For any $Q\in\Theta$ it holds that the Riccati operator $\T_{Q}(\cdot)$ (\cref{def:ric_op}) is a contraction on $\left(\PP,\dist_{\PP}(\cdot,\cdot)\right)$ (\cref{def:cont}) with contraction constant $\gamma\in[0,1)$ given by
    \begin{align*}
        \gamma = \Anorm^{2}\left(1+\frac{\Bnorm^{2}\Rinvnorm}{1-\Anorm^{2}}\right)^{2}
        \text{\,.}
    \end{align*}
\end{lemma}
\begin{proof}
    Let $P_{1}, P_{2}\in\PP$.
    Observe that
    \begin{align*}
        \T_{Q}(P_{1})-\T_{Q}(P_{2}) &= A^{\top}P_{1}A+Q-A^{\top}P_{1}BX(P_{1})^{-1}B^{\top}P_{1}A\\
        &-A^{\top}P_{2}A-Q+A^{\top}P_{2}BX(P_{2})^{-1}B^{\top}P_{2}A\\
        &=A^{\top}P_{1}A-A^{\top}P_{2}A\\
        &-A^{\top}P_{1}BX(P_{1})^{-1}B^{\top}P_{1}A+A^{\top}P_{2}BX(P_{2})^{-1}B^{\top}P_{2}A
    \end{align*}

    We thus split $\norm{\T_{Q}(P_{1})-\T_{Q}(P_{2})}_{F}$ into $2$ terms and bound each of them separately. 
    First, by \cref{lemma:frob_op_norm} it holds that
    \begin{align*}
        \norm{A^{\top}P_{1}A-A^{\top}P_{2}A}_{F}\leq \Anorm^{2}\norm{P_{1}-P_{2}}_{F}
        \text{\,.}
    \end{align*}
    Moreover, by \cref{lemma:frob_op_norm} and the triangle inequality it holds that
    \begin{align*}
        &\norm{A^{\top} P_{1} B X(P_{1})^{-1} B^\top P_{1} A-A^{\top} P_{2} B X(P_{2})^{-1} B^\top P_{2} A}_{F}\\
        &\leq \Anorm^{2}\norm{P_{1}BX(P_{1})^{-1}B^{\top}P_{1}-P_{2}BX(P_{1})^{-1}B^{\top}P_{1}}_{F}\\
        &+\Anorm^{2}\norm{P_{2}BX(P_{1})^{-1}B^{\top}P_1-P_{2}BX(P_{2})^{-1}B^{\top}P_1}_{F}\\
        &+\Anorm^{2}\norm{P_{2}BX(P_{2})^{-1}B^{\top}P_{1}-P_{2}BX(P_{2})^{-1}B^{\top}P_{2}}_{F}\\
        &= \Anorm^{2}\norm{(P_{1}-P_{2})BX(P_{1})^{-1}B^{\top}P_{1}}_{F}\\
        &+\Anorm^{2}\norm{P_{2}B(X(P_{1})^{-1}-X(P_{2})^{-1})B^{\top}P_{1}}_{F}\\
        &+\Anorm^{2}\norm{P_{2}BX(P_{2})^{-1}B^{\top}(P_{1}-P_{2})}_{F}
        \text{\,.}
    \end{align*}
    Since $R$ is PD, its smallest eigenvalue is $\Rinvnorm^{-1}$, so the matrix $X(P)$ satisfies 
    \begin{align*}
        X(P)-\Rinvnorm^{-1} I=R+B^{\top}PB-\Rinvnorm^{-1} I\succeq 0
    \end{align*}
    for any PSD matrix $P$.
    Therefore, by \cref{lemma:psd_order} we obtain that $\norm{X(P_{1})^{-1}}_{2}\,,\norm{X(P_{2})^{-1}}_{2}\leq \Rinvnorm$.
    Hence, since $P_{1}, P_{2}\in\PP$ we obtain by \cref{lemma:frob_op_norm} and the Cauchy-Schwarz inequality that
    \begin{align*}
        &\Anorm^{2}\norm{(P_{1}-P_{2})BX(P_{1})^{-1}B^{\top}P_{1}}_{F}+\Anorm^{2}\norm{P_{2}BX(P_{2})^{-1}B^{\top}(P_{1}-P_{2})}_{F}\\
        &\leq \Anorm^{2}\left(\Bnorm^{2}\norm{P_{1}}_{2}\norm{X(P_{1})^{-1}}_{2}\norm{P_{1}-P_{2}}_{F}+\Bnorm^{2}\norm{P_{2}}_{2}\norm{X(P_{2})^{-1}}_{2}\norm{P_{1}-P_{2}}_{F}\right)\\
        &\leq \frac{2\Anorm^{2}\Bnorm^{2}\Rinvnorm}{1-\Anorm^{2}}\norm{P_{1}-P_{2}}_{F}
        \text{\,.}
    \end{align*}
    Additionally, by \cref{lemma:inv_identity} it holds that
    \begin{align*}
        X(P_{1})^{-1}-X(P_{2})^{-1}&=X(P_{1})^{-1}(X(P_{2})-X(P_{1}))X(P_{2})^{-1}\\
        &=X(P_{1})^{-1}(R+B^{\top}P_{2}B-R-B^{\top}P_{1}B)X(P_{2})^{-1}\\
        &=X(P_{1})^{-1}(B^{\top}P_{2}B-B^{\top}P_{1}B)X(P_{2})^{-1}
        \text{\,.}
    \end{align*}
    Thus, since $P_{1}, P_{2}\in\PP$ we obtain by \cref{lemma:frob_op_norm} and the Cauchy-Schwarz inequality that
    \begin{align*}
        &\Anorm^{2}\norm{P_{2}B(X(P_{1})^{-1}-X(P_{2})^{-1})B^{\top}P_{1}}_{F}\\
        &=\Anorm^{2}\norm{P_{2}BX(P_{1})^{-1}(B^{\top}P_{2}B-B^{\top}P_{1}B)X(P_{2})^{-1}B^{\top}P_{1}}_{F}\\
        &=\Anorm^{2}\norm{P_{2}BX(P_{1})^{-1}B^{\top}(P_{2}-P_{1})BX(P_{2})^{-1}B^{\top}P_{1}}_{F}\\
        &\leq \frac{\Anorm^{2}\Bnorm^{4}\Rinvnorm^{2}}{(1-\Anorm^{2})^{2}}\norm{P_{2}-P_{1}}_{F}
        \text{\,.}
    \end{align*}
    Overall we obtain that
    \begin{align*}
        \norm{\T_{Q}(P_{1})-\T_{Q}(P_{2})}_{F}
        &\leq \Anorm^{2}\norm{P_{1}-P_{2}}_{F}\\
        &+\frac{2\Anorm^{2}\Bnorm^{2}\Rinvnorm}{1-\Anorm^{2}}\norm{P_{1}-P_{2}}_{F}\\
        &+\frac{\Anorm^{2}\Bnorm^{4}\Rinvnorm^{2}}{(1-\Anorm^{2})^{2}}\norm{P_{2}-P_{1}}_{F}\\
        &=\Anorm^{2}\left(1+\frac{\Bnorm^{2}\Rinvnorm}{1-\Anorm^{2}}\right)^{2}\norm{P_{1}-P_{2}}_{F}\\
        &=\gamma\norm{P_{1}-P_{2}}_{F}
    \end{align*}
    as required. 
    The fact that the contraction constant $\gamma$ is less than $1$ follows from \cref{lemma:norm_A_transform}.
\end{proof}

\cref{lemma:ric_op_cont} implies that the Riccati operator admits a unique fixed point that is within $\PP$. This is summarized in the following corollary.
\begin{corollary}
    \label{corollary:ric_op_fixed_point}
    For any $Q\in\Theta$ the Riccati operator $\T_{Q}(\cdot)$ (\cref{def:ric_op}) admits a unique fixed point $\Plqr(Q)\in\PP$.
\end{corollary}
\begin{proof}
    The space $\left(\BR^{\dstate,\dstate},\norm{\cdot}_{F}\right)$ is complete, and $\PP$ is closed in this space because the PSD cone is closed and the condition $\norm{P}_{2}\leq \frac{1}{1-\Anorm^{2}}$ defines a closed set. Hence $\left(\PP,\dist_{\PP}(\cdot,\cdot)\right)$ is a complete metric space. Therefore, since $\T_{Q}(\cdot)$ is a contraction on $\left(\PP,\dist_{\PP}(\cdot,\cdot)\right)$ (\cref{def:cont}) with contraction constant $\gamma\in[0,1)$ given by \cref{lemma:ric_op_cont}, the Banach fixed point theorem~\citep{banach1922sur} guarantees that it admits a unique fixed point $\Plqr(Q)\in\PP$.
\end{proof}

We conclude this part of the proof by bounding the Lipschitz constant of the mapping to unique fixed points of the Riccati operator $\T_{Q}(\cdot)$.
\begin{lemma}
    \label{lemma:ric_op_lipschitz}
    For any distinct $Q_{1},Q_{2}\in\Theta$, the mapping to unique fixed points of the Riccati operators $\T_{Q}(\cdot)$ defined by $\Plqr(\cdot)$ (\cref{corollary:ric_op_fixed_point}) satisfies
    \begin{align*}
        \norm{\Plqr(Q_{1})-\Plqr(Q_{2})}_{F}\leq \frac{1}{1-\gamma}\norm{Q_{1}-Q_{2}}_{F}
        \text{\,,}
    \end{align*}
    where $\gamma$ is the contraction constant of the Riccati operator given by \cref{lemma:ric_op_cont}.
\end{lemma}
\begin{proof}
    For any $Q\in\Theta$, the matrix $\Plqr(Q)$ is a fixed point of the Riccati operator $\T_{Q}(\cdot)$ (\cref{def:ric_op}). Therefore by definition it holds that
    \begin{align*}
        \Plqr(Q_{1})-\Plqr(Q_{2})= \T_{Q_{1}}(\Plqr(Q_{1}))-\T_{Q_{2}}(\Plqr(Q_{2}))
        \text{\,.}
    \end{align*}
    Since $\Plqr(Q_{1}), \Plqr(Q_{2})\in\PP$, we have that
    \begin{align*}
        &\T_{Q_{1}}(\Plqr(Q_{1}))-\T_{Q_{2}}(\Plqr(Q_{2}))\\
        &=A^\top \Plqr(Q_{1}) A + Q_{1} - A^{\top} \Plqr(Q_{1}) B X(\Plqr(Q_{1}))^{-1} B^\top \Plqr(Q_{1}) A\\
        &-A^\top \Plqr(Q_{2}) A - Q_{2} + A^{\top} \Plqr(Q_{2}) B X(\Plqr(Q_{2}))^{-1} B^\top \Plqr(Q_{2}) A
        \text{\,.}
    \end{align*}
    By bounding each of the three terms involving the fixed points as done in \cref{lemma:ric_op_cont}, we obtain that
    \begin{align*}
        \norm{\T_{Q_{1}}(\Plqr(Q_{1}))-\T_{Q_{2}}(\Plqr(Q_{2}))}_{F}\leq \gamma \norm{\Plqr(Q_{1})-\Plqr(Q_{2})}_{F}+\norm{Q_{1}-Q_{2}}_{F}
        \text{\,.}
    \end{align*}
    Therefore, rearranging yields
    \begin{align*}
        \norm{\Plqr(Q_{1})-\Plqr(Q_{2})}_{F}\leq \frac{1}{1-\gamma}\norm{Q_{1}-Q_{2}}_{F}
    \end{align*}
    as required.
\end{proof}

\subsection{Lipschitz Bound on the Mapping}
\label{app:lqr:controller_bound}

We conclude the proof of \cref{res:ub} by converting the Lipschitz bound on the mapping to unique fixed points of the Riccati operator obtained in \cref{lemma:ric_op_lipschitz} to a Lipschitz bound on the mapping $Q\mapsto \Klqr(Q)$.
\cref{lemma:controller_op_lipschitz} first provides a bound sharper than the one stated in \cref{res:ub}, which we later simplify to the stated expression.
\begin{lemma}
    \label{lemma:controller_op_lipschitz}
    For any distinct $Q_{1},Q_{2}\in\Theta$ the mapping $Q\mapsto \Klqr(Q)$ satisfies
    \begin{align*}
        \norm{\Klqr(Q_{1})-\Klqr(Q_{2})}_{F}\leq \frac{1}{1-\gamma}\Anorm\Bnorm\Rinvnorm\left(1+\frac{\Bnorm^{2}\Rinvnorm}{1-\Anorm^{2}}\right)\norm{Q_{1}-Q_{2}}_{F}
    \end{align*}
    where $\gamma$ is the contraction constant of the Riccati operator given in \cref{lemma:ric_op_cont}.
\end{lemma}
\begin{proof}
    By \cref{eq:klqr} and \cref{lemma:frob_op_norm} it holds that
    \begin{align*}
        &\norm{\Klqr(Q_{1})-\Klqr(Q_{2})}_{F}\\
        &=\norm{-X(\Plqr(Q_{1}))^{-1}B^{\top}\Plqr(Q_{1})A+X(\Plqr(Q_{2}))^{-1}B^{\top}\Plqr(Q_{2})A}_{F}\\
        &\leq \Anorm\norm{X(\Plqr(Q_{1}))^{-1}B^{\top}\Plqr(Q_{1})-X(\Plqr(Q_{2}))^{-1}B^{\top}\Plqr(Q_{2})}_{F}
    \end{align*}
    where $X(\cdot)$ is as defined in \cref{eq:def_X}.
    By the triangle inequality we obtain that
    \begin{align*}
        &\Anorm\norm{X(\Plqr(Q_{1}))^{-1}B^{\top}\Plqr(Q_{1})-X(\Plqr(Q_{2}))^{-1}B^{\top}\Plqr(Q_{2})}_{F}\\
        &\leq \Anorm\norm{X(\Plqr(Q_{1}))^{-1}B^{\top}\Plqr(Q_{1})-X(\Plqr(Q_{1}))^{-1}B^{\top}\Plqr(Q_{2})}_{F}\\
        &+\Anorm\norm{X(\Plqr(Q_{1}))^{-1}B^{\top}\Plqr(Q_{2})-X(\Plqr(Q_{2}))^{-1}B^{\top}\Plqr(Q_{2})}_{F}\\
        &\leq \Anorm\Bnorm\Rinvnorm\norm{\Plqr(Q_{1})-\Plqr(Q_{2})}_{F}\\
        &+\frac{\Anorm\Bnorm}{1-\Anorm^{2}}\norm{X(\Plqr(Q_{1}))^{-1}-X(\Plqr(Q_{2}))^{-1}}_{F}
        \text{\,,}
    \end{align*}
    where we utilized the fact that $\norm{X(\Plqr(Q_{1}))^{-1}}_{2},\norm{X(\Plqr(Q_{2}))^{-1}}_{2}\leq \Rinvnorm$ established in \cref{lemma:ric_op_cont} and the fact that $\Plqr(Q_{1}), \Plqr(Q_{2})\in\PP$. 
    Lastly, by \cref{lemma:frob_op_norm,lemma:inv_identity} we have
    \begin{align*}
        &\norm{X(\Plqr(Q_{1}))^{-1}-X(\Plqr(Q_{2}))^{-1}}_{F}\\
        &=\norm{X(\Plqr(Q_{1}))^{-1}(X(\Plqr(Q_{2}))-X(\Plqr(Q_{1})))X(\Plqr(Q_{2}))^{-1}}_{F}\\
        &=\norm{X(\Plqr(Q_{1}))^{-1}(R+B^{\top}\Plqr(Q_{2})B-R-B^{\top}\Plqr(Q_{1})B)X(\Plqr(Q_{2}))^{-1}}_{F}\\
        &=\norm{X(\Plqr(Q_{1}))^{-1}(B^{\top}\Plqr(Q_{2})B-B^{\top}\Plqr(Q_{1})B)X(\Plqr(Q_{2}))^{-1}}_{F}\\
        &\leq \Bnorm^{2}\Rinvnorm^{2}\norm{\Plqr(Q_{2})-\Plqr(Q_{1})}_{F}
        \text{\,.}
    \end{align*}
    Therefore, we obtain by \cref{lemma:ric_op_lipschitz} that
    \begin{align*}
        &\norm{\Klqr(Q_{1})-\Klqr(Q_{2})}_{F}\\
        &\leq \left(\Anorm\Bnorm\Rinvnorm+\frac{\Anorm\Bnorm^{3}\Rinvnorm^{2}}{1-\Anorm^{2}}\right)\norm{\Plqr(Q_{1})-\Plqr(Q_{2})}_{F}\\
        &=\Anorm\Bnorm\Rinvnorm\left(1+\frac{\Bnorm^{2}\Rinvnorm}{1-\Anorm^{2}}\right)\norm{\Plqr(Q_{1})-\Plqr(Q_{2})}_{F}\\
        &\leq \frac{1}{1-\gamma}\Anorm\Bnorm\Rinvnorm\left(1+\frac{\Bnorm^{2}\Rinvnorm}{1-\Anorm^{2}}\right)\norm{Q_{1}-Q_{2}}_{F}
        \text{\,,}
    \end{align*}
    thereby concluding the proof.
\end{proof}

We complete the proof of \cref{res:ub} by simplifying the sharper upper bound for the Lipschitz constant provided in \cref{lemma:controller_op_lipschitz} to the stated expression.
Note that by \cref{lemma:norm_A_transform} and by the definition of $\gamma$ from \cref{lemma:ric_op_cont} it holds that
\begin{align*}
    \gamma = \Anorm^{2}\left(1+\frac{\Bnorm^{2}\Rinvnorm}{1-\Anorm^{2}}\right)^{2}<\frac{1}{2}
    \text{\,.}
\end{align*}
Therefore, the above upper bound simplifies to
\begin{align*}
    \frac{1}{1-\gamma}\Anorm\Bnorm\Rinvnorm\left(1+\frac{\Bnorm^{2}\Rinvnorm}{1-\Anorm^{2}}\right)\leq 2\Anorm\Bnorm\Rinvnorm\left(1+\frac{\Bnorm^{2}\Rinvnorm}{1-\Anorm^{2}}\right)
    \text{\,.}
\end{align*}
\cref{ass:stable} implies that $\Anorm^{2}<\frac{1}{2}$, therefore we obtain that
\begin{align*}
    2\Anorm\Bnorm\Rinvnorm\left(1+\frac{\Bnorm^{2}\Rinvnorm}{1-\Anorm^{2}}\right)\leq 2\Anorm\Bnorm\Rinvnorm\left(1+2\Bnorm^{2}\Rinvnorm\right)
    \text{\,,}
\end{align*}
which is precisely the upper bound stated in \cref{res:ub}.

	\section{Proof of \cref{res:sep}}
\label{app:separation}

In this appendix, we prove \cref{res:sep}, namely the comparison between the unsafe/LQR map and the safe/$\Hinf$-robust map in the risk-handling setting of \cref{sec:prelim:setup_theory}. Throughout our proof, we focus on state cost matrices that are simultaneously diagonalizable in the basis $V$, \ie, the set of cost matrices given by
\begin{align*}
    \Theta_{V} := \left\{ Q \in \Theta : Q = V\Lambda_{Q}V^{\top}\,,\,  \Lambda_{Q}:=\diag(q_{1}, \ldots, q_{\dstate})\right\}
    \text{\,.}
\end{align*}
Since $\Theta$ is made up of PSD matrices of Frobenius norm at most 1, and since $V$ is orthogonal, the set $\Theta_{V}$ may be written as
\begin{align*}
    \Theta_{V} = \left\{ V\diag(q_{1},\dots,q_{\dstate})V^\top:q_{j}\in\BR_{\ge 0}\,,\,j\in[\dstate]\,,\,\norm{\qbf}_2\leq 1\right\}
    \text{\,.}
\end{align*}

We will provide a lower bound on the Lipschitz constant within this set, which is sufficient to conclude the bound holds for all cost matrices in $\Theta$. For readability, throughout this appendix we write
\begin{align*}
    A = V\Lambda_{A}V^\top,\qquad
    B = V\Lambda_{B}V^\top,\qquad
    D = V\Lambda_{D}V^\top,\qquad
    R = V\Lambda_{R}V^\top
    \text{\,,}
\end{align*}
where $\Lambda_{A}=\diag(a_1,\dots,a_{\dstate})$, $\Lambda_{B}=\diag(b_1,\dots,b_{\dstate})$, $\Lambda_{D}=\diag(d_1,\dots,d_{\dstate})$, and $\Lambda_{R}=\diag(r_1,\dots,r_{\dstate})$ are diagonal matrices holding the eigenvalues of $A$, $B$, $D$, and $R$ in the shared eigenbasis $V$, respectively.

The outline of the proof is as follows. 
\cref{app:separation:diag} shows that under simultaneous orthogonal diagonalizability of \cref{ass:commute}, the safe/$\Hinf$-robust control problem can be solved through individual scalar $\Hinf$-robust control problems (one for each coordinate).
\cref{app:separation:stack} proceeds to solve these problems and shows that stacking the scalar solutions yields an explicit diagonal mapping that is optimal for the original problem.
\cref{app:separation:coincide} then compares this stacked mapping with the canonical DGARE-synthesized safe mapping we study (defined in \cref{app:background:hinf}), and shows that they could partially coincide in a neighborhood of a state cost matrix.
Based on this local coincidence, \cref{app:separation:lip} provides a lower bound on the Lipschitz constant of the safe mapping $\Kinf(\cdot)$.
\cref{app:separation:strict} concludes by translating the bounds introduced in \cref{res:ub,app:separation:lip} to the inequality of \cref{res:sep}.

\subsection{Simultaneously Diagonalizable $\Hinf$-Robust Control}
\label{app:separation:diag}
We begin by showing that in the simultaneously diagonalizable setting, state cost matrices in $\Theta_{V}$ admit unique values of the safe mapping synthesized via the DGARE (\cref{eq:kinf}) that are also simultaneously diagonalizable in the basis $V$.

\begin{proposition}\label{prop:diag_hinf}
    For any $Q\in \Theta_{V}$ the DGARE-synthesized matrix $\Kinf(Q)$ (\cref{eq:kinf}) is simultaneously diagonalizable in the basis $V$, \ie, $\Kinf(Q) = V\Lambda_{K}(Q)V^{\top}$ where $\Lambda_{K}(Q)\in\BR^{\dstate,\dstate}$ is a diagonal matrix denoted as
    \begin{align*}
        \Lambda_{K}(Q)=:\diag(\ksafe_1(Q),\dots,\ksafe_{\dstate}(Q))
        \text{\,.}
    \end{align*}
\end{proposition}
\begin{proof}
    Let $\gamma\geq \gammainf(Q)$. 
    Recall that the matrices $\Lambda_{A}=V^\top A V$, $\Lambda_{B}=V^\top B V$, $\Lambda_{D}=V^\top D V$ and $\Lambda_{R}=V^\top R V$ are diagonal, and for any symmetric matrix $P\in\BR^{\dstate,\dstate}$, define $\wt{P}:=V^\top P V$.
    For any symmetric matrix $P\in\BR^{\dstate,\dstate}$ denote
    \begin{align*}
        \wt H(\wt P;\gamma)
        :=
        \begin{pmatrix}
            \Lambda_R+\Lambda_B^\top \wt P\Lambda_B & \Lambda_B^\top \wt P\Lambda_D\\~\\
            \Lambda_D^\top \wt P\Lambda_B & \Lambda_D^\top \wt P\Lambda_D-\gamma^2 I
        \end{pmatrix}
    \end{align*}
    and similarly define $\wt F(\wt P)$ as
    \begin{align*}
        \wt F(\wt P)
        :=
        \begin{pmatrix}
            \Lambda_B^\top \wt P\Lambda_A\\~\\
            \Lambda_D^\top \wt P\Lambda_A
        \end{pmatrix}
        \text{\,.}
    \end{align*}
    We have that
    \begin{align*}
        H(P;\gamma)
        &=
        \begin{pmatrix}
            R+B^\top P B & B^\top P D\\
            D^\top P B & D^\top P D-\gamma^2 I
        \end{pmatrix}\\
        &=
        \begin{pmatrix}
            V\Lambda_RV^\top+V\Lambda_B \wt P\Lambda_B V^\top & V\Lambda_B \wt P\Lambda_D V^\top\\~\\
            V\Lambda_D \wt P\Lambda_B V^\top & V\Lambda_D \wt P\Lambda_D V^\top-\gamma^2 VIV^\top
        \end{pmatrix}\\
        &=
        \begin{pmatrix}
            V & \0\\
            \0 & V
        \end{pmatrix}
        \begin{pmatrix}
            \Lambda_R+\Lambda_B^\top \wt P\Lambda_B & \Lambda_B^\top \wt P\Lambda_D\\~\\
            \Lambda_D^\top \wt P\Lambda_B & \Lambda_D^\top \wt P\Lambda_D-\gamma^2 I
        \end{pmatrix}
        \begin{pmatrix}
            V^\top & \0\\
            \0 & V^\top
        \end{pmatrix}\\
        &=
        \begin{pmatrix}
            V & \0\\
            \0 & V
        \end{pmatrix}
        \wt H(\wt P;\gamma)
        \begin{pmatrix}
            V^\top & \0\\
            \0 & V^\top
        \end{pmatrix}
    \end{align*}
    where $\0$ denotes the all-zeros matrix of appropriate dimensions.
    Similarly, we have that
    \begin{align*}
        F(P)
        =
        \begin{pmatrix}
            B^\top P A\\
            D^\top P A
        \end{pmatrix}
        =
        \begin{pmatrix}
            V \Lambda_B \wt P \Lambda_A V^\top\\
            V \Lambda_D \wt P \Lambda_A V^\top
        \end{pmatrix}
        =
        \begin{pmatrix}
            V & \0\\
            \0 & V
        \end{pmatrix}
        \begin{pmatrix}
            \Lambda_B^\top \wt P\Lambda_A\\~\\
            \Lambda_D^\top \wt P\Lambda_A
        \end{pmatrix}V^{\top}
        =
        \begin{pmatrix}
            V & \0\\
            \0 & V
        \end{pmatrix}
        \wt F(\wt P)V^{\top}
        \text{\,.}
    \end{align*}
    Therefore, whenever $H(P;\gamma)$ is nonsingular we have by orthogonality of $V$ that
    \begin{align*}
        H(P;\gamma)^{-1}
        =
        \begin{pmatrix}
            V & \0\\
            \0 & V
        \end{pmatrix}
        \wt H(\wt P;\gamma)^{-1}
        \begin{pmatrix}
            V^{\top} & \0\\
            \0 & V^{\top}
        \end{pmatrix}
        \text{\,.}
    \end{align*}
    Hence, multiplying both sides of the quadratic Riccati term by $V^\top$ from the left and $V$ from the right respectively yields
    \begin{align*}
        V^\top F(P)^\top H(P;\gamma)^{-1}F(P)V
        &=
        V^{\top}F(P)^{\top}\begin{pmatrix}
            V & \0\\
            \0 & V
        \end{pmatrix}
        \wt H(\wt P;\gamma)^{-1}
        \begin{pmatrix}
            V^{\top} & \0\\
            \0 & V^{\top}
        \end{pmatrix}
        F(P)V\\
        &=\wt F(\wt P)^\top \wt H(\wt P;\gamma)^{-1}\wt F(\wt P)
        \text{\,.}
    \end{align*}
    Because $V$ is orthogonal, a matrix $P\in\BR^{\dstate,\dstate}$ satisfies the DGARE (\cref{eq:dgare}) if and only if $\wt P$ satisfies the DGARE in the transformed coordinates, obtained by multiplying both sides of \cref{eq:dgare} by $V^\top$ from the left and $V$ from the right respectively:
    \begin{align}
        \label{eq:dgare_tilde}
        \wt P
        = \Lambda_{A} \wt P \Lambda_{A} + \Lambda_{Q}
        - \wt F(\wt P)^\top \wt H(\wt P;\gamma)^{-1}\wt F(\wt P)
        \text{\,.}
    \end{align}
    Consider the set of diagonal matrices $\D:=\{\wt P\in\BR^{\dstate,\dstate}: \wt P \text{ is diagonal}\}$.
    We claim that the DGARE operator in \cref{eq:dgare_tilde} preserves $\D$:
    if $\wt P\in\D$, then the right-hand side of \cref{eq:dgare_tilde} is in $\D$.
    Indeed, since $\Lambda_{A},\Lambda_{B},\Lambda_{D},\Lambda_{Q},\Lambda_{R}$ are diagonal, the matrix $\Lambda_A\wt P\Lambda_A+\Lambda_Q$ is diagonal.
    Write the diagonal blocks of $\wt H(\wt P;\gamma)$ as
    \begin{align*}
        M:=\Lambda_R+\Lambda_B^\top\wt P\Lambda_B,
        \qquad
        N:=\Lambda_B^\top\wt P\Lambda_D,
        \qquad
        L:=\Lambda_D^\top\wt P\Lambda_D-\gamma^2 I
        \text{\,.}
    \end{align*}
    These matrices are diagonal, hence they commute.
    The commutative block-inverse formula gives
    \begin{align*}
        \wt H(\wt P;\gamma)^{-1}
        =
        \begin{pmatrix}
            L\Delta^{-1} & -N\Delta^{-1}\\
            -N\Delta^{-1} & M\Delta^{-1}
        \end{pmatrix}
        \text{\,,}
    \end{align*}
    where $\Delta:=ML-N^2$.
    Define the diagonal matrices
    \begin{align*}
        G:=\Lambda_B^\top\wt P\Lambda_A,
        \qquad
        J:=\Lambda_D^\top\wt P\Lambda_A
        \text{\,.}
    \end{align*}
    Since $\wt F(\wt P)=\begin{pmatrix}G\\ J\end{pmatrix}$, unfolding the multiplication gives
    \begin{align*}
        \wt F(\wt P)^\top \wt H(\wt P;\gamma)^{-1}\wt F(\wt P)
        &=
        \begin{pmatrix}
            G^\top & J^\top
        \end{pmatrix}
        \begin{pmatrix}
            L\Delta^{-1} & -N\Delta^{-1}\\
            -N\Delta^{-1} & M\Delta^{-1}
        \end{pmatrix}
        \begin{pmatrix}
            G\\
            J
        \end{pmatrix}\\
        &=
        G^\top L\Delta^{-1}G
        -
        G^\top N\Delta^{-1}J
        -
        J^\top N\Delta^{-1}G
        +
        J^\top M\Delta^{-1}J
        \text{\,.}
    \end{align*}
    Each term in the final expression is a product of diagonal matrices, hence $\wt F(\wt P)^\top \wt H(\wt P;\gamma)^{-1}\wt F(\wt P)$ is diagonal.
    This proves invariance: the DGARE map of~\cref{eq:dgare_tilde} maps $\D$ to itself. 
    Now, by assumption there exists a unique stabilizing PSD solution $\Pinf(Q;\gamma)$ to the DGARE of~\cref{eq:dgare}.
    Equivalently, there exists a unique stabilizing PSD solution $\wt P_\infty:=V^\top \Pinf(Q;\gamma)V$ to \cref{eq:dgare_tilde}.
    Therefore by standard results on the convergence of DGAREs~\citep{dragan2008iterative,aberkane2023deterministic}, repeated application of the DGARE in \cref{eq:dgare_tilde} starting from $P=\0$ converges to $\wt P_\infty$.
    This together with the invariance of $\D$ under the DGARE map implies that $\wt P_\infty\in\D$, \ie~the matrix $V^\top \Pinf(Q;\gamma)V$ is diagonal.
    It remains to show that the corresponding $\wt K_u(Q;\gamma)$ (\cref{eq:kinf_level_gamma}) is diagonal in the same basis.
    Evaluate $-\wt H(\wt P_\infty;\gamma)^{-1}\wt F(\wt P_\infty)$ using the same block inverse:
    \begin{align*}
        -\wt H(\wt P_\infty;\gamma)^{-1}\wt F(\wt P_\infty)
        =
        -
        \begin{pmatrix}
            L^\infty(\Delta^\infty)^{-1}G^\infty-N^\infty(\Delta^\infty)^{-1}J^\infty\\
            -N^\infty(\Delta^\infty)^{-1}G^\infty+M^\infty(\Delta^\infty)^{-1}J^\infty
        \end{pmatrix}
        \text{\,,}
    \end{align*}
    where the superscript $\infty$ denotes evaluation at $\wt P_\infty=\diag(p_1^\infty,\dots,p_{\dstate}^\infty)$.
    The matrix $\wt K_u(Q;\gamma)$ is the upper block of this stacked matrix.
    Therefore,
    \begin{align*}
        \wt K_u(Q;\gamma)
        =
        -\left(L^\infty(\Delta^\infty)^{-1}G^\infty-N^\infty(\Delta^\infty)^{-1}J^\infty\right)
        \text{\,.}
    \end{align*}
    Since $L^\infty,\Delta^\infty,G^\infty,N^\infty,$ and $J^\infty$ are diagonal, the above is a sum of products of diagonal matrices hence it is diagonal.
    Transforming back gives $\Kinf(Q;\gamma)=V\wt K_u(Q;\gamma)V^\top$, hence $V^\top \Kinf(Q;\gamma)V$ is diagonal.
    Namely, this holds for $\Kinf(Q)=\Kinf(Q;\gammainf(Q))$ as required.
\end{proof}

\cref{prop:diag_hinf} implies that when restricting to the simultaneously diagonalizable set $\Theta_{V}$, the mapping $Q\mapsto \Kinf(Q)$ can be viewed in the diagonal basis as the mapping
\begin{align}
    \label{eq:diag_hinf_map}
    \qbf \mapsto \kbfsafe(\qbf)
    \text{\,,}
\end{align}
where for any $Q=V\diag(\qbf)V^\top\in\Theta_V$ we define $\kbfsafe(\qbf)\in\BR^{\dstate}$ through
\begin{align*}
    \Kinf(Q)=V\diag(\kbfsafe(\qbf))V^\top
    \text{\,.}
\end{align*}
We proceed to showing that the Lipschitz behavior in the original coordinates is bounded by the Lipschitz behavior of this diagonal-basis mapping. For completeness, we first define the notion of Lipschitz constant for a mapping between vectors.

\begin{definition}
    \label{def:lip_vec}
    Let $\V\subseteq\BR^{\dstate}$ and let  $f:\V\to\BR^{\dstate}$ be some function. We denote the \emph{Lipschitz constant} of $f$ as $\lip(f)$, given by
    \begin{align*}
        \lip(f) := \sup_{\substack{\vbf_{1}, \vbf_{2} \in \V\,,\, \vbf_{1}\neq \vbf_{2}}} \frac{\norm{f(\vbf_{1})-f(\vbf_{2})}_{2}}{\norm{\vbf_{1}-\vbf_{2}}_2}
        \text{\,.}
    \end{align*}
\end{definition}

\begin{lemma}\label{lemma:diag_hinf_lip}
    Consider the set $\D_Q\subseteq\BR^{\dstate}_{\ge 0}$ given by
    \begin{align*}
        \D_{Q}:=\left\{\qbf\in\BR^{\dstate}_{\ge 0}:\ \norm{\qbf}_2\leq 1\right\}
        \text{\,.}
    \end{align*}
    Let $\kbfsafe:\D_Q\to\BR^{\dstate}$ be the mapping
    \begin{align*}
        \kbfsafe(\qbf)=\left(\ksafe_1(\qbf),\dots,\ksafe_{\dstate}(\qbf)\right)
        \text{\,,}
    \end{align*}
    where $\ksafe_j(\qbf)$ is the $j$-th diagonal entry of the matrix $\Kinf(\cdot)$ in the diagonal basis $V$ as in \cref{eq:diag_hinf_map}, so that for any $Q=V\diag(\qbf)V^\top\in\Theta_V$,
    \begin{align*}
        \Kinf(Q)=V\diag(\kbfsafe(\qbf))V^\top
        \text{\,.}
    \end{align*}
    Let $Q^{(0)}\in\Theta_V$ and write $Q^{(0)}=V\diag(\qbf^{(0)})V^\top$ with $\qbf^{(0)}\in\D_Q$.
    For any $j\in[\dstate]$, let $g_j:\BR\to\BR$ be the function defined by
    \begin{align*}
        g_j(h)=\ksafe_j(\qbf^{(0)}+h\ebf_j)
        \text{\,.}
    \end{align*}
    For any $j\in[\dstate]$, if $q_{j}^{(0)}\in(0,1)$ and $g_j(\cdot)$ is differentiable at $h=0$, then the Lipschitz constant of $\Kinf:\Theta_{V}\to\BR^{\dstate,\dstate}$ satisfies
    \begin{align*}
        \lip(\Kinf)
        \geq 
        |g_j'(0)|
        \text{\,.}
    \end{align*}
\end{lemma}
\begin{proof}
    For any $Q^{(1)},Q^{(2)}\in\Theta_V$ with $Q^{(i)}=V\diag(\qbf^{(i)})V^\top$ and $\qbf^{(i)}\in\D_Q$, orthogonality of $V$ implies that
    \begin{align*}
        \norm{\Kinf(Q^{(1)})-\Kinf(Q^{(2)})}_F
        &=\norm{\diag(\kbfsafe(\qbf^{(1)}))-\diag(\kbfsafe(\qbf^{(2)}))}_F
        \\
        &=\norm{\kbfsafe(\qbf^{(1)})-\kbfsafe(\qbf^{(2)})}_2
        \text{\,,}
    \end{align*}
    and similarly $\norm{Q^{(1)}-Q^{(2)}}_F=\norm{\qbf^{(1)}-\qbf^{(2)}}_2$.
    Hence by \cref{def:lip_vec}
    \begin{align*}
        \lip(\Kinf)&\geq \sup_{Q^{(1)},Q^{(2)}\in\Theta_{V}\,,\, Q^{(1)}\neq Q^{(2)}}
        \frac{\norm{\Kinf(Q^{(1)})-\Kinf(Q^{(2)})}_F}{\norm{Q^{(1)}-Q^{(2)}}_F}\\
        &=\sup_{\qbf^{(1)},\qbf^{(2)}\in\D_Q\,,\,\qbf^{(1)}\neq \qbf^{(2)}}
        \frac{\|\kbfsafe(\qbf^{(1)})-\kbfsafe(\qbf^{(2)})\|_2}{\|\qbf^{(1)}-\qbf^{(2)}\|_2}\\
        &=\lip(\kbfsafe)
        \text{\,.}
    \end{align*}
    The fact that $q_j^{(0)}\in(0,1)$ implies that $\qbf^{(0)}+h \ebf_j\in \D_Q$ for all sufficiently small $h\in\BR$.
    Using the definition of $\lip(\kbfsafe)$ and restricting the supremum to the pair $\qbf^{(1)}=\qbf^{(0)}+h \ebf_j$ and $\qbf^{(2)}=\qbf^{(0)}$, we obtain for all sufficiently small $h\neq 0$ that
    \begin{align*}
        \lip(\kbfsafe)\geq
        \frac{\norm{\kbfsafe(\qbf^{(0)}+h \ebf_j)-\kbfsafe(\qbf^{(0)})}_2}{\norm{(\qbf^{(0)}+h \ebf_j)-\qbf^{(0)}}_2}
        \geq
        \frac{\left|\ksafe_j(\qbf^{(0)}+h\ebf_j)-\ksafe_j(\qbf^{(0)})\right|}{|h|}
        \text{\,.}
    \end{align*}
    Taking $h\to 0$ and using differentiability of $g_j(\cdot)$ at $0$ yields
    \begin{align*}
        \lip(\kbfsafe)&\geq
        \lim_{h\to 0}\frac{\left|\ksafe_j(\qbf^{(0)}+h\ebf_j)-\ksafe_j(\qbf^{(0)})\right|}{|h|}
        \\
        &=\lim_{h\to 0}\frac{|g_{j}(h)-g_{j}(0)|}{|h|}\\
        &=|g_{j}'(0)|
        \text{\,,}
    \end{align*}
    as required.
\end{proof}

\subsection{Scalar Solutions and Stacked Mapping}
\label{app:separation:stack}

We now turn to the individual scalar $\Hinf$-robust control problems. 
In this subsection, we characterize the scalar solution in each coordinate $k_{j}^{*}(\cdot)$ and the corresponding scalar optimal value $\Gamma_j^*(\cdot)$. 
We then show that the mapping obtained by stacking these scalar minimizers, denoted $K^{*}(\cdot)$, solves the $\Hinf$-robust control problem.
Although potentially different from the mapping $\Kinf(\cdot)$ synthesized by the coupled block DGARE, the former admits a closed-form expression allowing for concrete analysis.
We will later relate it to the coupled block DGARE mapping we are interested in.

We begin by formulating the closed-form expression for the scalar objective. For completeness, we first provide the classical definition of the discrete-time Fourier transform which is used in the proof.
\begin{definition}\label{def:dtft}
    Let $\{z_{t}\}_{t=0}^{\infty}$ be a sequence of complex numbers such that $\sum_{t=0}^{\infty}|z_{t}|^{2}<\infty$.
    We define the \emph{discrete-time Fourier transform} of $\{z_{t}\}_{t=0}^{\infty}$ as the function $\hat{z}:[0,2\pi]\to\BC$ given by
    \begin{align*}
        \hat{z}(\omega)=\sum_{t=0}^{\infty}z_{t}e^{-i\omega t}
        \text{\,.}
    \end{align*}
\end{definition}

\begin{lemma}\label{lemma:scalar_hinf_obj}
    Fix $j\in[\dstate]$. 
    For any $q_{j}\in[0,1]$ and $k_{j}\in\BR$ such that $|a_{j}+b_{j}k_{j}|<1$ it holds that the robust scalar objective in coordinate $j$ is given by
    \begin{align*}
        \Gamma_{j}(q_{j},k_{j}):=\sup_{w_{t}}\frac{\sum_{t=0}^{\infty}q_{j}x_{t}^{2}+r_{j}(k_{j}x_{t})^{2}}{\sum_{t=0}^{\infty}w_{t}^{2}}=d_{j}^{2}\frac{q_{j}+r_{j}k_{j}^{2}}{(1-|a_{j}+b_{j}k_{j}|)^{2}}
        \text{\,,}
    \end{align*}
    where the state $x_{t}$ is given by $x_{t+1}=a_{j}x_{t}+b_{j}k_{j}x_{t}+d_{j}w_{t}$ with $x_{0}=0$.
\end{lemma}
\begin{proof}
    Denote $\xi:=a_{j}+b_{j}k_{j}$. For $t\in\BN$ define $h_{t}:=\xi ^{t-1}d_{j}$, and additionally define $h_{0}:=0$.
    For any $t\in\BN$, unrolling of the state $x_{t}$ yields a standard convolution expression of the form
    \begin{align*}
        x_{t}=\sum_{\tau=0}^{t-1}\xi ^{t-1-\tau}d_{j}w_{\tau}=\sum_{\tau=0}^{t-1}h_{t-\tau}w_{\tau}
        \text{\,.}
    \end{align*}
    Observe that since $|\xi|<1$ it holds that 
    \begin{align*}
        \sum_{t=0}^{\infty}|h_{t}|=\sum_{t=1}^{\infty}|\xi ^{t-1}d_{j}|=\sum_{t=1}^{\infty}|\xi|^{t-1}|d_{j}|=\frac{|d_{j}|}{1-|\xi|}<\infty
    \end{align*}
    and
    \begin{align*}
        \sum_{t=0}^{\infty}h_{t}^{2}=\sum_{t=1}^{\infty}(\xi ^{t-1}d_{j})^{2}=\sum_{t=1}^{\infty}\xi ^{2(t-1)}d_{j}^{2}=\frac{d_{j}^{2}}{1-\xi ^{2}}<\infty
        \text{\,.}
    \end{align*}
    Let $\{w_{t}\}_{t=0}^{\infty}$ be a sequence of real numbers with $\sum_{t=0}^{\infty}w_{t}^{2}<\infty$.
    Viewing the real sequences $\{x_t\}_{t=0}^\infty$, $\{h_t\}_{t=0}^\infty$, and $\{w_t\}_{t=0}^\infty$ as complex-valued sequences when applying the Fourier-analytic results from \cref{app:aux}, \cref{lemma:dtft_product} gives that the discrete-time Fourier transform of $\{x_{t}\}_{t=0}^{\infty}$ is the function $\hat{x}:[0,2\pi]\to\BC$ given by
    \begin{align*}
        \hat{x}(\omega)=\hat{h}(\omega)\hat{w}(\omega)
        \text{\,.}
    \end{align*}
    By \cref{def:dtft} it holds that
    \begin{align*}
        \hat{h}(\omega)=\sum_{t=1}^{\infty}\xi ^{t-1}d_{j}e^{-i\omega t}=d_{j}e^{-i\omega}\sum_{t=1}^{\infty}\xi ^{t-1}e^{-i\omega (t-1)}=\frac{d_{j}e^{-i\omega}}{1-\xi e^{-i\omega}}=\frac{d_{j}}{e^{i\omega}-\xi}
        \text{\,.}
    \end{align*}
    Since $|\xi|<1$ it holds that the function $f:[0,2\pi]\to\BR$ given by 
    \begin{align*}
        f(\omega):=|\hat{h}(\omega)|^{2}=\frac{d_{j}^{2}}{\left|e^{i\omega}-\xi\right|^{2}}
    \end{align*}
    is continuous on $[0,2\pi]$.
    Hence, by \cref{thm:plancherel} it holds that
    \begin{align*}
        \sum_{t=0}^{\infty}x_{t}^{2}&=\frac{1}{2\pi}\int_{0}^{2\pi}|\hat{x}(\omega)|^{2}d\omega\\
        &=\frac{1}{2\pi}\int_{0}^{2\pi}|\hat{h}(\omega)|^{2}|\hat{w}(\omega)|^{2}d\omega\\
        &\leq \left(\sup_{\omega\in[0,2\pi]} f(\omega)\right)\frac{1}{2\pi}\int_{0}^{2\pi}|\hat{w}(\omega)|^{2}d\omega\\
        &=\left(\sup_{\omega\in[0,2\pi]} f(\omega)\right)\sum_{t=0}^{\infty}w_{t}^{2}
        \text{\,.}
    \end{align*}
    This implies that
    \begin{align*}
        \frac{\sum_{t=0}^{\infty}x_{t}^{2}}{\sum_{t=0}^{\infty}w_{t}^{2}}\leq \sup_{\omega\in[0,2\pi]} f(\omega)
        \text{\,.}
    \end{align*}
    On the other hand, by \cref{thm:plancherel} it holds that
    \begin{align*}
        \frac{\sum_{t=0}^{\infty}x_{t}^{2}}{\sum_{t=0}^{\infty}w_{t}^{2}}=\frac{\frac{1}{2\pi}\int_{0}^{2\pi}f(\omega)|\hat{w}(\omega)|^{2}d\omega}{\frac{1}{2\pi}\int_{0}^{2\pi}|\hat{w}(\omega)|^{2}d\omega}=\frac{\int_{0}^{2\pi}f(\omega)|\hat{w}(\omega)|^{2}d\omega}{\int_{0}^{2\pi}|\hat{w}(\omega)|^{2}d\omega}
    \end{align*}
    Let $\omega^{*}\in[0,2\pi]$ be such that $f(\omega^{*})=\sup_{\omega\in[0,2\pi]} f(\omega)$ (exists since $f(\cdot)$ is continuous on $[0,2\pi]$). 
    Since $|\xi|<1$ and $|e^{i\omega}|=1$ for all $\omega\in[0,2\pi]$ it holds by the reverse triangle inequality that
    \begin{align*}
        |e^{i\omega}-\xi|\geq ||e^{i\omega}|-|\xi||= 1-|\xi|
    \end{align*}
    with equality achieved at $\omega=0$ when $\xi\geq 0$ and at $\omega=\pi$ when $\xi<0$.
    Using the notation of \cref{lemma:dtft_approx}, we can thus take $\omega^*\in\{0,\pi\}$ so the sequences given by the lemma are real-valued.
    It states that there exists a series of real-valued sequences $\left\{\left\{w_{t}^{(n)}\right\}_{t=0}^{\infty}\right\}_{n=1}^{\infty}$ whose discrete-time Fourier transforms $\left\{\hat{w}^{(n)}(\cdot)\right\}_{n=1}^{\infty}$ satisfy
    \begin{align*}
        \lim_{n\to\infty}\frac{\int_{0}^{2\pi}f(\omega)|\hat{w}^{(n)}(\omega)|^{2}d\omega}{\int_{0}^{2\pi}|\hat{w}^{(n)}(\omega)|^{2}d\omega}
        =
        f(\omega^{*})
        \text{\,.}
    \end{align*}
    Therefore,
    \begin{align*}
        \sup_{w_{t}}\frac{\sum_{t=0}^{\infty}x_{t}^{2}}{\sum_{t=0}^{\infty}w_{t}^{2}}
        =
        \sup_{w_{t}}\frac{\int_{0}^{2\pi}f(\omega)|\hat{w}(\omega)|^{2}d\omega}{\int_{0}^{2\pi}|\hat{w}(\omega)|^{2}d\omega}
        \geq
        \lim_{n\to\infty}\frac{\int_{0}^{2\pi}f(\omega)|\hat{w}^{(n)}(\omega)|^{2}d\omega}{\int_{0}^{2\pi}|\hat{w}^{(n)}(\omega)|^{2}d\omega}
        =
        f(\omega^{*})
        \text{\,.}
    \end{align*}
    Combining the two inequalities, we obtain
    \begin{align*}
        \sup_{w_{t}}\frac{\sum_{t=0}^{\infty}x_{t}^{2}}{\sum_{t=0}^{\infty}w_{t}^{2}}= f(\omega^{*})
        \text{\,.}
    \end{align*}
    Therefore,
    \begin{align*}
        \sup_{w_{t}}\frac{\sum_{t=0}^{\infty}x_{t}^{2}}{\sum_{t=0}^{\infty}w_{t}^{2}}=\frac{d_{j}^{2}}{\min_{\omega\in[0,2\pi]}\left|e^{i\omega}-\xi\right|^{2}}=\frac{d_{j}^{2}}{(1-|\xi|)^{2}}
        \text{\,.}
    \end{align*}
    The proof concludes by noting that
    \begin{align*}
        \Gamma_{j}(q_{j},k_{j})=(q_{j}+r_{j}k_{j}^{2})\sup_{w_{t}}\frac{\sum_{t=0}^{\infty}x_{t}^{2}}{\sum_{t=0}^{\infty}w_{t}^{2}}
    \end{align*}
    and plugging back $\xi=a_j+b_jk_j$.
\end{proof}

Having characterized the robust scalar objective, we are now ready to characterize the solution to the robust scalar control problem in coordinate $j$ for a given cost parameter $q_{j}$.
\begin{proposition}\label{prop:scalar_hinf_sol}
    Fix $j\in[\dstate]$. Assume that $a_{j}\neq 0$ and $b_{j}\neq 0$. 
    For any $q_{j}\in[0,\frac{|a_{j}|r_{j}(1-|a_{j}|)}{b_{j}^{2}})$ the robust scalar objective $\Gamma_{j}(q_{j},k_{j})$ from \cref{lemma:scalar_hinf_obj} has a unique minimizer over the interval
    \begin{align*}
        \left\{k_{j}\in\BR: |a_{j}+b_{j}k_{j}|<1\right\}
    \end{align*}
    given by
    \begin{align*}
        k_{j}^{*}(q_{j})=-\frac{a_j}{|a_j|}\frac{b_{j}}{r_{j}(1-|a_{j}|)}q_{j}
        \text{\,.}
    \end{align*}
    Additionally, the corresponding optimal scalar value
    \begin{align*}
        \Gamma_j^*(q_j)
        :=
        \inf_{\substack{k_j\in\BR\\ |a_j+b_jk_j|<1}}\Gamma_j(q_j,k_j)
        =
        \Gamma_j(q_j,k_j^*(q_j))
        \text{\,,}
    \end{align*}
    satisfies
    \begin{align*}
        \Gamma_j^*(q_j)
        =
        \frac{d_j^2}{(1-|a_j|)^2}\cdot\frac{q_j}{1+\frac{b_j^2}{r_j(1-|a_j|)^2}q_j}
        \text{\,.}
    \end{align*}
    In particular, the mapping $q_j\mapsto \Gamma_j^*(q_j)$ is continuous on the interval $[0,\frac{|a_{j}|r_{j}(1-|a_{j}|)}{b_{j}^{2}})$.
\end{proposition}
\begin{proof}
    Define the scalar quantity $\zeta_j:=b_jk_j$. 
    Since $b_j\neq 0$, minimizing over $k_j\in\BR$ with $|a_j+b_jk_j|<1$ is equivalent to minimizing over $\zeta_j\in\BR$ with $|a_j+\zeta_j|<1$.
    Define $s_j:=\frac{a_j}{|a_j|}\in\{-1,1\}$ and $\wt{\zeta}_j:=-s_j \zeta_j$.
    By \cref{lemma:scalar_hinf_obj}, we obtain that
    \begin{align*}
        \Gamma_j(q_j,k_j)
        &=d_j^2\frac{q_j+r_jk_j^2}{(1-|a_j+b_jk_j|)^2}\\
        &=d_j^2\frac{q_j+\frac{r_j}{b_j^2}\zeta_j^2}{(1-|a_j+\zeta_j|)^2}\\
        &=d_j^2\frac{q_j+\frac{r_j}{b_j^2}s_j^2\zeta_j^2}{(1-|s_ja_j+s_j\zeta_j|)^2}\\
        &=d_j^2\frac{q_j+\frac{r_j}{b_j^2}\wt{\zeta}_j^2}{(1-||a_j|-\wt{\zeta}_j|)^2}
        \text{\,.}
    \end{align*}
    Therefore, minimizing over the interval $\zeta_j\in(-a_j-1,-a_j+1)$ is equivalent to minimizing over the interval $\wt{\zeta}_j\in(|a_j|-1,|a_j|+1)$.
    We analyze the above expression on the two branches $\wt{\zeta}_j\leq |a_j|$ and $\wt{\zeta}_j>|a_j|$.
    On the branch $\wt{\zeta}_j\leq |a_j|$, it holds that $||a_j|-\wt{\zeta}_j|=|a_j|-\wt{\zeta}_j$, hence
    \begin{align*}
        \Gamma_j(q_j,k_j)
        =
        d_j^2\frac{q_j+\frac{r_j}{b_j^2}\wt{\zeta}_j^2}{(1-|a_j|+\wt{\zeta}_j)^2}
        \text{\,.}
    \end{align*}
    Differentiating with respect to $\wt{\zeta}_j$ gives
    \begin{align*}
        \frac{\partial \Gamma_j}{\partial \wt{\zeta}_j}
        &=
        d_j^2\cdot
        \frac{2\frac{r_j}{b_j^2}\wt{\zeta}_j(1-|a_j|+\wt{\zeta}_j)^2-2(q_j+\frac{r_j}{b_j^2}\wt{\zeta}_j^2)(1-|a_j|+\wt{\zeta}_j)}
        {(1-|a_j|+\wt{\zeta}_j)^4}\\
        &=
        2d_j^2\cdot
        \frac{\frac{r_j}{b_j^2}\wt{\zeta}_j(1-|a_j|+\wt{\zeta}_j)-(q_j+\frac{r_j}{b_j^2}\wt{\zeta}_j^2)}
        {(1-|a_j|+\wt{\zeta}_j)^3}\\
        &=
        2d_j^2\cdot
        \frac{\frac{r_j}{b_j^2}\wt{\zeta}_j(1-|a_j|)-q_j}{(1-|a_j|+\wt{\zeta}_j)^3}
        \text{\,.}
    \end{align*}
    Since $1-|a_j|+\wt{\zeta}_j>0$ on this branch, the unique critical point is
    \begin{align*}
        \wt{\zeta}_j^*(q_j)=\frac{b_j^2}{r_j(1-|a_j|)}q_j
        \text{\,,}
    \end{align*}
    and moreover
    \begin{align*}
        \frac{\partial \Gamma_j}{\partial \wt{\zeta}_j}<0
        \iff
        \wt{\zeta}_j<\wt{\zeta}_j^*(q_j)
        \qquad\text{and}\qquad
        \frac{\partial \Gamma_j}{\partial \wt{\zeta}_j}>0
        \iff
        \wt{\zeta}_j>\wt{\zeta}_j^*(q_j)
        \text{\,.}
    \end{align*}
    Therefore, within the branch $\wt{\zeta}_j\leq |a_j|$, the function has a unique minimizer at $\wt{\zeta}_j^*(q_j)$.
    Since
    \begin{align*}
        \wt{\zeta}_j^*(q_j)<|a_j|
        \iff
        \frac{b_j^2}{r_j(1-|a_j|)}q_j<|a_j|
        \iff
        q_j<\frac{|a_j|r_j(1-|a_j|)}{b_j^2}
        \text{\,,}
    \end{align*}
    the assumed condition on $q_j$ implies that $\wt{\zeta}_j^*(q_j)\in(|a_j|-1,|a_j|]$.
    As for the branch $\wt{\zeta}_j>|a_j|$, it holds that $||a_j|-\wt{\zeta}_j|=\wt{\zeta}_j-|a_j|$, hence
    \begin{align*}
        \Gamma_j(q_j,k_j)
        =
        d_j^2\frac{q_j+\frac{r_j}{b_j^2}\wt{\zeta}_j^2}{(1+|a_j|-\wt{\zeta}_j)^2}
        \text{\,.}
    \end{align*}
    Differentiating with respect to $\wt{\zeta}_j$ yields
    \begin{align*}
        \frac{\partial \Gamma_j}{\partial \wt{\zeta}_j}
        &=
        d_j^2\cdot
        \frac{2\frac{r_j}{b_j^2}\wt{\zeta}_j(1+|a_j|-\wt{\zeta}_j)^2+2(q_j+\frac{r_j}{b_j^2}\wt{\zeta}_j^2)(1+|a_j|-\wt{\zeta}_j)}
        {(1+|a_j|-\wt{\zeta}_j)^4}\\
        &=
        2d_j^2\cdot
        \frac{\frac{r_j}{b_j^2}\wt{\zeta}_j(1+|a_j|)+q_j}{(1+|a_j|-\wt{\zeta}_j)^3}
        \text{\,.}
    \end{align*}
    Since $q_{j}>0$, $\wt{\zeta}_j>|a_j|>0$ and $1+|a_j|-\wt{\zeta}_j>0$ on this branch, it follows that $\frac{\partial \Gamma_j}{\partial \wt{\zeta}_j}>0$ for all $\wt{\zeta}_j\in(|a_j|,|a_j|+1)$.
    Thus the objective is strictly increasing on the branch $\wt{\zeta}_j>|a_j|$, and therefore has no minimizer there.
    Combining the two branches, the global minimizer over $\{k_j\in\BR: |a_j+b_jk_j|<1\}$ is attained at
    \begin{align*}
        k_j^*(q_j)=-\frac{s_j\wt{\zeta}_j^*(q_j)}{b_j}=-\frac{a_j}{|a_j|}\frac{b_j}{r_j(1-|a_j|)}q_j
        \text{\,.}
    \end{align*}
    Consequently,
    \begin{align*}
        \Gamma_j^*(q_j)
        &=
        \Gamma_j(q_j,k_j^*(q_j))\\
        &=
        d_j^2\frac{q_j+r_j\left(-\frac{a_j}{|a_j|}\frac{b_j}{r_j(1-|a_j|)}q_j\right)^2}
        {\left(1-\left|a_j+b_j\left(-\frac{a_j}{|a_j|}\frac{b_j}{r_j(1-|a_j|)}q_j\right)\right|\right)^2}\\
        &=
        d_j^2\frac{q_j+\frac{b_j^2}{r_j(1-|a_j|)^2}q_j^2}
        {\left((1-|a_j|)+\frac{b_j^2}{r_j(1-|a_j|)}q_j\right)^2}\\
        &=
        d_j^2\frac{q_j\left(1+\frac{b_j^2}{r_j(1-|a_j|)^2}q_j\right)}
        {(1-|a_j|)^2\left(1+\frac{b_j^2}{r_j(1-|a_j|)^2}q_j\right)^2}\\
        &=
        \frac{d_j^2}{(1-|a_j|)^2}\cdot\frac{q_j}{1+\frac{b_j^2}{r_j(1-|a_j|)^2}q_j}
        \text{\,.}
    \end{align*}
    The stated continuity of $q_j\mapsto\Gamma_j^*(q_j)$ follows immediately from this explicit formula.
\end{proof}

We now turn to showing that the mapping obtained by stacking the scalar minimizers for each coordinate, $K^{*}(\cdot)$, satisfies the robust objective.
We begin by proving that for any controller matrix $K$ that is diagonal in the basis $V$, the robust objective of \cref{eq:robust_cost_norm} is equal to the maximum of the scalar robust objectives.
\begin{proposition}\label{prop:hinf_obj_decouple}
    Suppose \cref{ass:commute} holds. Let $Q=V\diag(\qbf)V^\top\in\Theta_V$ with $\qbf=(q_1,\dots,q_{\dstate})$, and let $K=V\diag(\kbf)V^\top$ with $\kbf=(k_1,\dots,k_{\dstate})$. 
    Consider the linear policy $\ubf(\xbf):=K\xbf$, and assume that for every $j\in[\dstate]$ it holds that $|a_j+b_jk_j|<1$.
    Then the normalized robust objective of \cref{eq:robust_cost_norm} satisfies
    \begin{align*}
        \sup_{\wbf_t}
        \frac{\sum_{t=0}^{\infty}\xbf_t^\top Q\xbf_t+\ubf(\xbf_t)^\top R\ubf(\xbf_t)}{\sum_{t=0}^{\infty}\norm{\wbf_t}_2^2}
        =
        \max_{j\in[\dstate]}\Gamma_{j}(q_j,k_j)
        \text{\,,}
    \end{align*}
    where $\Gamma_{j}(q_j,k_j)$ is the scalar robust objective from \cref{lemma:scalar_hinf_obj}.
\end{proposition}
\begin{proof}
    For any $t\in\BN\cup\{0\}$, define the transformed vectors $\wt{\xbf}_{t}:=V^\top \xbf_t$ and $\wt{\wbf}_{t}:=V^\top \wbf_t$, and also define the transformed policy $\wt{\ubf}(\wt{\xbf})$ through
    \begin{align*}
        \wt{\ubf}(\wt{\xbf}):=V^\top \ubf(\xbf)=V^\top K \xbf=V^{\top}V\diag(\kbf)V^\top \xbf=\diag(\kbf) \wt{\xbf}
        \text{\,.}
    \end{align*}
    Since $V$ is orthogonal and $A,B,D,Q,R,K$ are all diagonal in the basis $V$, the disturbed dynamics together with the policy $\ubf(\cdot)$ become coordinate-wise:
    \begin{align*}
        \wt{\xbf}_{t+1}&=V^{\top}\xbf_{t+1}\\
        &=V^\top A \xbf_t+V^\top B \ubf(\xbf_t)+V^\top D \wbf_t\\
        &=V^\top V\Lambda_{A}V^\top \xbf_t+V^\top V\Lambda_{B}V^\top \ubf(\xbf_t)+V^\top V\Lambda_{D}V^\top \wbf_t\\
        &=\Lambda_{A}\wt{\xbf}_t+\Lambda_{B}\wt{\ubf}(\wt{\xbf}_t)+\Lambda_{D}\wt{\wbf}_t\\
        &=\Lambda_{A}\wt{\xbf}_t+\Lambda_{B}\diag(\kbf)\wt{\xbf}_t+\Lambda_{D}\wt{\wbf}_t
        \text{\,,}
    \end{align*}
    with $\wt{\xbf}_0=0$.
    Hence each coordinate $j\in[\dstate]$ evolves independently according only to its own disturbance sequence $\{\wt{w}_{j,t}\}_{t=0}^{\infty}$.
    Moreover, orthogonality of $V$ implies that the disturbance energy is preserved and the quadratic cost separates across coordinates:
    \begin{align*}
        \sum_{t=0}^{\infty}\norm{\wbf_t}_2^2
        &=
        \sum_{t=0}^{\infty}\norm{\wt{\wbf}_t}_2^2
        =
        \sum_{j=1}^{\dstate}\sum_{t=0}^{\infty}\wt{w}_{j,t}^{2}
        \text{\,,}
    \end{align*}
    and
    \begin{align*}
        \sum_{t=0}^{\infty}\xbf_t^\top Q\xbf_t+\ubf(\xbf_t)^\top R\ubf(\xbf_t)
        &=
        \sum_{t=0}^{\infty}\wt{\xbf}_t^\top \diag(\qbf)\wt{\xbf}_t+\wt{\ubf}(\wt{\xbf}_t)^\top \Lambda_{R} \wt{\ubf}(\wt{\xbf}_t)\\
        &=
        \sum_{j=1}^{\dstate}\sum_{t=0}^{\infty}\left(q_j\wt{x}_{j,t}^2+r_j\wt{u}_{j,t}(\wt{x}_{j,t})^2\right)
        \text{\,.}
    \end{align*}
    Therefore,
    \begin{align*}
        \sup_{\wbf_t}
        \frac{\sum_{t=0}^{\infty}\xbf_t^\top Q\xbf_t+\ubf(\xbf_t)^\top R\ubf(\xbf_t)}{\sum_{t=0}^{\infty}\norm{\wbf_t}_2^2}
        =
        \sup_{\wt{\wbf}_t}
        \frac{\sum_{j=1}^{\dstate}\sum_{t=0}^{\infty}\left(q_j\wt{x}_{j,t}^2+r_j\wt{u}_{j,t}(\wt{x}_{j,t})^2\right)}
        {\sum_{j=1}^{\dstate}\sum_{t=0}^{\infty}\wt{w}_{j,t}^{2}}
        \text{\,.}
    \end{align*}
    We first prove the right-hand side upper bounds the left-hand side. Let $\{\wbf_{t}\}_{t=0}^{\infty}$ be a non-zero disturbance sequence. For each $j\in[\dstate]$, the $j$-th coordinate depends only on $(q_j,k_j,a_j,b_j,d_j,r_j)$ and on the scalar disturbance sequence $\{\wt{w}_{j,t}\}_{t=0}^{\infty}$. Therefore by \cref{lemma:scalar_hinf_obj} it holds that
    \begin{align*}
        \sum_{t=0}^{\infty}q_j\wt{x}_{j,t}^2+r_j\wt{u}_{j,t}(\wt{x}_{j,t})^2\leq \Gamma_{j}(q_j,k_j)\sum_{t=0}^{\infty}\wt{w}_{j,t}^{2}
        \text{\,.}
    \end{align*}
    Summing over $j$ yields
    \begin{align*}
        \sum_{j=1}^{\dstate}\sum_{t=0}^{\infty}\left(q_j\wt{x}_{j,t}^2+r_j\wt{u}_{j,t}(\wt{x}_{j,t})^2\right)&\leq \sum_{j=1}^{\dstate}\Gamma_{j}(q_j,k_j)\sum_{t=0}^{\infty}\wt{w}_{j,t}^{2}\\
        &\leq
        \left(\max_{j\in[\dstate]}\Gamma_{j}(q_j,k_j)\right)\sum_{j=1}^{\dstate}\sum_{t=0}^{\infty}\wt{w}_{j,t}^{2}
    \end{align*}
    Hence, for every disturbance sequence,
    \begin{align*}
        \frac{\sum_{t=0}^{\infty}\xbf_t^\top Q\xbf_t+\ubf(\xbf_t)^\top R\ubf(\xbf_t)}{\sum_{t=0}^{\infty}\norm{\wbf_t}_2^2}&=\frac{\sum_{j=1}^{\dstate}\sum_{t=0}^{\infty}\left(q_j\wt{x}_{j,t}^2+r_j\wt{u}_{j,t}(\wt{x}_{j,t})^2\right)}{\sum_{j=1}^{\dstate}\sum_{t=0}^{\infty}\wt{w}_{j,t}^{2}}\\
        &\leq
        \max_{j\in[\dstate]}\Gamma_{j}(q_j,k_j)
        \text{\,,}
    \end{align*}
    and taking the supremum over $\wbf_t$ gives
    \begin{align*}
        \sup_{\wbf_t}
        \frac{\sum_{t=0}^{\infty}\xbf_t^\top Q\xbf_t+\ubf_t^\top R\ubf_t}{\sum_{t=0}^{\infty}\norm{\wbf_t}_2^2}
        \leq
        \max_{j\in[\dstate]}\Gamma_{j}(q_j,k_j)
        \text{\,.}
    \end{align*}
    We now prove the matching lower bound. Let $j^*\in[\dstate]$ be an index attaining the maximum, that is,
    \begin{align*}
        \Gamma_{j^*}(q_{j^*},k_{j^*})
        =
        \max_{j\in[\dstate]}\Gamma_{j}(q_j,k_j)
        \text{\,.}
    \end{align*}
    Restrict attention to disturbance sequences satisfying $\wt{w}_{j,t}=0$ for all $t\in\BN$ and $j\neq j^*$.
    For such disturbances all state coordinates except $j^*$ remain identically zero: Indeed, for any $j\neq j^*$ and any $t\in\BN$, the $j$-th coordinate satisfies
    \begin{align*}
        \wt{x}_{j,t+1}&=a_{j}\wt{x}_{j,t}+b_{j}k_{j}\wt{x}_{j,t}+d_{j}\wt{w}_{j,t}\\
        &=(a_j+b_jk_j)\wt{x}_{j,t}\\
        &=(a_{j}+b_{j}k_{j})^{t+1}\wt{x}_{j,0}\\
        &=0
        \text{\,,}
    \end{align*}
    where the last equality follows from the initial condition $\wt{x}_{j,0}=0$.    
    Therefore the vector objective reduces exactly to the scalar objective in coordinate $j^*$:
    \begin{align*}
        \frac{\sum_{t=0}^{\infty}\xbf_t^\top Q\xbf_t+\ubf(\xbf_t)^\top R\ubf(\xbf_t)}{\sum_{t=0}^{\infty}\norm{\wbf_t}_2^2}
        =
        \frac{\sum_{t=0}^{\infty}\left(q_{j^*}\wt{x}_{j^*,t}^2+r_{j^*}\wt{u}_{j^*}(\wt{x}_{j^*,t})^2\right)}
        {\sum_{t=0}^{\infty}\wt{w}_{j^*,t}^{2}}
        \text{\,.}
    \end{align*}
    Taking the supremum over disturbances supported on the $j^*$-th coordinate and using \cref{lemma:scalar_hinf_obj}, we obtain
    \begin{align*}
        \sup_{\wbf_t}
        \frac{\sum_{t=0}^{\infty}\xbf_t^\top Q\xbf_t+\ubf_t^\top R\ubf_t}{\sum_{t=0}^{\infty}\norm{\wbf_t}_2^2}
        &=\sup_{\wt{\wbf}_t}\frac{\sum_{j=1}^{\dstate}\sum_{t=0}^{\infty}\left(q_j\wt{x}_{j,t}^2+r_j\wt{u}_{j,t}(\wt{x}_{j,t})^2\right)}{\sum_{j=1}^{\dstate}\sum_{t=0}^{\infty}\wt{w}_{j,t}^{2}}\\
        &\geq \sup_{\wt{w}_{j^{*},t}}\frac{\sum_{t=0}^{\infty}\left(q_{j^*}\wt{x}_{j^*,t}^2+r_{j^*}\wt{u}_{j^*}(\wt{x}_{j^*,t})^2\right)}
        {\sum_{t=0}^{\infty}\wt{w}_{j^*,t}^{2}}
        \\
        &=\Gamma_{j^*}(q_{j^*},k_{j^*})\\
        &=\max_{j\in[\dstate]}\Gamma_{j}(q_j,k_j)
        \text{\,.}
    \end{align*}
    Combining the two bounds concludes the proof.
\end{proof}

Next, we show that the optimal value of the diagonal robust control problem is the maximum of the optimal scalar values, and that it is attained by the stacked mapping built from the scalar minimizers.
\begin{corollary}\label{cor:diag_hinf_value}
    Suppose \cref{ass:commute} holds. Let $Q=V\diag(\qbf)V^\top\in\Theta_V$ with $\qbf=(q_1,\dots,q_{\dstate})$, and for each $j\in[\dstate]$ define
    \begin{align*}
        \Gamma_j^*(q_j)
        :=
        \inf_{\substack{k_j\in\BR\\ |a_j+b_jk_j|<1}}\Gamma_{j}(q_j,k_j)
        \text{\,.}
    \end{align*}
    Then the optimal value of \cref{eq:robust_cost_norm} among policies of the form $\ubf(\xbf):=K\xbf$ with $K=V\diag(\kbf)V^\top$ is given by
    \begin{align*}
        \inf_{\substack{\kbf\in\BR^{\dstate}\\ |a_j+b_jk_j|<1\ \forall j\in[\dstate]}}
        \sup_{\wbf_t}
        \frac{\sum_{t=0}^{\infty}\xbf_t^\top Q\xbf_t+\ubf(\xbf_t)^\top R\ubf(\xbf_t)}{\sum_{t=0}^{\infty}\norm{\wbf_t}_2^2}
        =
        \max_{j\in[\dstate]}\Gamma_j^*(q_j)
        \text{\,.}
    \end{align*}
    Moreover, if for each $j$ one chooses the scalar minimizer $k_j^*(q_j)$ and defines the stacked mapping
    \begin{align*}
        K^*(Q):=V\diag\left(k_1^*(q_1),\dots,k_{\dstate}^*(q_{\dstate})\right)V^\top
        \text{\,,}
    \end{align*}
    then $K^*(Q)$ attains the above value.
\end{corollary}
\begin{proof}
    By \cref{prop:hinf_obj_decouple}, for any diagonal controller matrix $K=V\diag(\kbf)V^\top$ satisfying $|a_j+b_jk_j|<1$ for all $j\in[\dstate]$, the corresponding normalized robust objective equals
    \begin{align*}
        \max_{j\in[\dstate]}\Gamma_{j}(q_j,k_j)
        \text{\,.}
    \end{align*}
    Therefore the diagonal-controller optimization problem reduces to
    \begin{align*}
        \inf_{\substack{\kbf\in\BR^{\dstate}\\ |a_j+b_jk_j|<1\ \forall j\in[\dstate]}}
        \max_{j\in[\dstate]}\Gamma_{j}(q_j,k_j)
        \text{\,.}
    \end{align*}
    We claim this value is exactly $\max_{j\in[\dstate]}\Gamma_j^*(q_j)$.
    For the lower bound, fix any $\kbf$ with $|a_j+b_jk_j|<1$ for all $j$. Then for every $j\in[\dstate]$,
    \begin{align*}
        \Gamma_{j}(q_j,k_j)\geq \Gamma_j^*(q_j)
        \text{\,,}
    \end{align*}
    and hence
    \begin{align*}
        \max_{j\in[\dstate]}\Gamma_{j}(q_j,k_j)
        \geq
        \max_{j\in[\dstate]}\Gamma_j^*(q_j)
        \text{\,.}
    \end{align*}
    Taking the infimum over $\kbf$ gives
    \begin{align*}
        \inf_{\substack{\kbf\in\BR^{\dstate}\\ |a_j+b_jk_j|<1\ \forall j}}
        \max_{j\in[\dstate]}\Gamma_{j}(q_j,k_j)
        \geq
        \max_{j\in[\dstate]}\Gamma_j^*(q_j)
        \text{\,.}
    \end{align*}
    For the reverse inequality, let $k_j^*(q_j)$ be a scalar minimizer of $\Gamma_{j}(q_j,\cdot)$ for each $j\in[\dstate]$, so that
    \begin{align*}
        \Gamma_{j}(q_j,k_j^*(q_j))=\Gamma_j^*(q_j)
        \text{\,.}
    \end{align*}
    Then the stacked mapping value $K^*(Q)=V\diag(k_1^*(q_1),\dots,k_{\dstate}^*(q_{\dstate}))V^\top$ satisfies
    \begin{align*}
        \max_{j\in[\dstate]}\Gamma_{j}(q_j,k_j^*(q_j))
        =
        \max_{j\in[\dstate]}\Gamma_j^*(q_j)
        \text{\,.}
    \end{align*}
    Therefore,
    \begin{align*}
        \inf_{\substack{\kbf\in\BR^{\dstate}\\ |a_j+b_jk_j|<1\ \forall j}}
        \max_{j\in[\dstate]}\Gamma_{j}(q_j,k_j)
        \leq
        \max_{j\in[\dstate]}\Gamma_j^*(q_j)
        \text{\,.}
    \end{align*}
    Combining the two inequalities proves the claimed identity, and the same argument shows that the stacked mapping value $K^*(Q)$ attains the optimum.
\end{proof}

\subsection{DGARE Mapping and Active Index Coincidence}
\label{app:separation:coincide}

We now compare the stacked mapping $K^{*}(\cdot)$ with the mapping $\Kinf(\cdot)$ synthesized by the coupled block DGARE. 
In this subsection, we first show that both mappings attain the same optimal value. 
We then show that when the active scalar index, \ie, the index attaining the maximum of the optimal scalar values, is unique, the coupled block DGARE mapping and the stacked mapping coincide on that active coordinate.

\begin{corollary}\label{cor:hinf_val_coincide}
    Assume \cref{ass:commute}. Let $Q=V\diag(\qbf)V^\top\in\Theta_V$ with $\qbf=(q_1,\dots,q_{\dstate})$, and define $\Gamma_j^*(q_j)$ and the stacked mapping value $K^*(Q)$ as in \cref{cor:diag_hinf_value}. 
    Write
    \begin{align*}
        \Kinf(Q)=V\diag\left(\ksafe_1(\qbf),\dots,\ksafe_\dstate(\qbf)\right)V^\top
        \text{\,,}
    \end{align*}
    as in \cref{prop:diag_hinf}.
    Then
    \begin{align*}
        \gammainf(Q)^2
        =
        \max_{j\in[\dstate]}\Gamma_j^*(q_j)
        \text{\,.}
    \end{align*}
    Moreover, both the stacked mapping $K^*(Q)$ and the DGARE-synthesized mapping $\Kinf(Q)$ attain this value, and in particular
    \begin{align*}
        \max_{j\in[\dstate]}\Gamma_j\left(q_j,\ksafe_j(\qbf)\right)
        =
        \max_{j\in[\dstate]}\Gamma_j^*(q_j)
        =
        \gammainf(Q)^2
        \text{\,.}
    \end{align*}
\end{corollary}
\begin{proof}
    By \cref{cor:diag_hinf_value}, the stacked mapping value $K^*(Q)$ attains the value
    \begin{align*}
        \max_{j\in[\dstate]}\Gamma_j^*(q_j)
        \text{\,.}
    \end{align*}
    Since $\gammainf(Q)^2$ is the minimum of the normalized robust objective over all policies by \cref{eq:gamma_inf}, it follows that
    \begin{align*}
        \gammainf(Q)^2
        \leq
        \max_{j\in[\dstate]}\Gamma_j^*(q_j)
        \text{\,.}
    \end{align*}
    On the other hand, by \cref{eq:kinf} and the discussion following \cref{eq:kinf_level_gamma}, the DGARE-synthesized mapping value $\Kinf(Q)$ attains the optimal value $\gammainf(Q)^2$ and is stabilizing. 
    By \cref{prop:diag_hinf}, $\Kinf(Q)$ is diagonal in the basis $V$, hence the closed-loop matrix $A+B\Kinf(Q)$ is diagonal in this basis with diagonal entries $a_j+b_j\ksafe_j(\qbf)$.
    Since $\Kinf(Q)$ is stabilizing, all eigenvalues of $A+B\Kinf(Q)$ have modulus strictly smaller than $1$, and therefore
    \begin{align*}
        |a_j+b_j\ksafe_j(\qbf)|<1
    \end{align*}
    for all $j\in[\dstate]$.
    Hence \cref{prop:hinf_obj_decouple} implies that the value attained by $\Kinf(Q)$ is
    \begin{align*}
        \max_{j\in[\dstate]}\Gamma_j\left(q_j,\ksafe_j(\qbf)\right)
        \text{\,.}
    \end{align*}
    Moreover, \cref{cor:diag_hinf_value} yields
    \begin{align*}
        \gammainf(Q)^2
        =
        \max_{j\in[\dstate]}\Gamma_j\left(q_j,\ksafe_j(\qbf)\right)
        \geq
        \max_{j\in[\dstate]}\Gamma_j^*(q_j)
        \text{\,.}
    \end{align*}
    Combining this with the previous inequality gives
    \begin{align*}
        \gammainf(Q)^2
        =
        \max_{j\in[\dstate]}\Gamma_j^*(q_j)
        \text{\,.}
    \end{align*}
    Substituting this equality into the previous display yields
    \begin{align*}
        \max_{j\in[\dstate]}\Gamma_j\left(q_j,\ksafe_j(\qbf)\right)
        =
        \max_{j\in[\dstate]}\Gamma_j^*(q_j)
        =
        \gammainf(Q)^2
        \text{\,.}
    \end{align*}
    This proves the claim, and shows that both $K^*(Q)$ and $\Kinf(Q)$ attain the same optimal value.
\end{proof}

\begin{lemma}\label{lemma:unique_active_index}
    Let $Q^{(0)}=V\diag(\qbf^{(0)})V^\top\in\Theta_V$ with $\qbf^{(0)}\in\D_Q$ in the interior of $\D_Q$.
    Suppose there exists a unique index $j^*\in[\dstate]$ such that
    \begin{align*}
        \Gamma_{j^*}^*(q_{j^*}^{(0)})
        =
        \max_{j\in[\dstate]}\Gamma_j^*(q_j^{(0)})
        \text{\,,}
    \end{align*}
    which satisfies $q_{j^*}^{(0)}\in\left[0,\frac{|a_{j^*}|r_{j^*}(1-|a_{j^*}|)}{b_{j^*}^{2}}\right)$, $a_{j^{*}}\neq 0$, and $b_{j^*}\neq 0$.
    Then there exists a neighborhood $\NN\subseteq\D_Q$ of $\qbf^{(0)}$ such that for every $\qbf\in\NN$ the index $j^*$ remains the unique active index, that is,
    \begin{align*}
        \Gamma_{j^*}^*(q_{j^*})
        =
        \max_{j\in[\dstate]}\Gamma_j^*(q_j)
        \text{\,,}
    \end{align*}
    the DGARE-synthesized and stacked mapping values attain the same value on the active coordinate,
    \begin{align*}
        \Gamma_{j^*}\left(q_{j^*},\ksafe_{j^{*}}(\qbf)\right)
        =
        \Gamma_{j^*}\left(q_{j^*},k_{j^*}^{*}(q_{j^*})\right)
        =
        \Gamma_{j^*}^*(q_{j^*})
        =
        \gammainf(Q)^2
        \text{\,,}
    \end{align*}
    and the unique minimizers coincide, meaning that for $Q=V\diag(\qbf)V^\top$, the DGARE-synthesized mapping value $\Kinf(Q)$ and the stacked mapping value $K^*(Q)$ satisfy
    \begin{align*}
        \ksafe_{j^{*}}(\qbf)
        =
        k_{j^*}^*(q_{j^*})
        =
        -\frac{a_{j^*}}{|a_{j^*}|}\frac{b_{j^*}}{r_{j^*}(1-|a_{j^*}|)}q_{j^*}
        \text{\,.}
    \end{align*}
\end{lemma}
\begin{proof}
    Fix any $j\in[\dstate]$.
    By \cref{lemma:scalar_hinf_obj}, for any $q_j\in[0,1]$, the scalar robust control problem for which $\Gamma_j^*(q_j)$ is the optimal value is given by
    \begin{align*}
        \inf_{\substack{k_j\in\BR\\ |a_j+b_jk_j|<1}}
        d_j^2\frac{q_j+r_jk_j^2}{(1-|a_j+b_jk_j|)^2}
        \text{\,.}
    \end{align*}
    Therefore, the mapping $q_j\mapsto\Gamma_j^*(q_j)$ is the infimum of affine functions of $q_j$. Hence it is concave on $[0,1]$ and therefore continuous on $[0,1]$.
    Next, since $j^*$ is the unique maximizing index of $\left\{\Gamma_j^*(q_j^{(0)})\right\}_{j\in[\dstate]}$, it follows that
    \begin{align*}
        \Delta
        :=
        \Gamma_{j^*}^*(q_{j^*}^{(0)})
        -
        \max_{j\in[\dstate]\setminus\{j^*\}}\Gamma_j^*(q_j^{(0)})>
        0
        \text{\,.}
    \end{align*}
    Since $\qbf^{(0)}$ lies in the interior of $\D_Q$ and $q_{j^*}^{(0)}<\frac{|a_{j^*}|r_{j^*}(1-|a_{j^*}|)}{b_{j^*}^{2}}$, there exists a neighborhood $\NN\subseteq\D_Q$ of $\qbf^{(0)}$ such that for every $\qbf\in\NN$ it holds that $q_{j^*}\in\left(0,\frac{|a_{j^*}|r_{j^*}(1-|a_{j^*}|)}{b_{j^*}^{2}}\right)$.
    By continuity of each map $q_j\mapsto\Gamma_j^*(q_j)$ at $q_j^{(0)}$, after possibly shrinking $\NN$ we may assume that for every $\qbf\in\NN$ and every $j\in[\dstate]$,
    \begin{align*}
        \left|\Gamma_j^*(q_j)-\Gamma_j^*(q_j^{(0)})\right|
        <
        \frac{\Delta}{3}
        \text{\,.}
    \end{align*}
    Hence for every $\qbf\in\NN$ and every $j\neq j^*$,
    \begin{align*}
        \Gamma_{j^*}^*(q_{j^*})
        &>
        \Gamma_{j^*}^*(q_{j^*}^{(0)})-\frac{\Delta}{3}\\
        &=
        \max_{\ell\in[\dstate]\setminus\{j^*\}}\Gamma_\ell^*(q_\ell^{(0)})+\frac{2\Delta}{3}\\
        &\geq
        \Gamma_j^*(q_j^{(0)})+\frac{2\Delta}{3}\\
        &>
        \Gamma_j^*(q_j)+\frac{\Delta}{3}\\
        &>
        \Gamma_j^*(q_j)
        \text{\,.}
    \end{align*}
    Therefore, for every $\qbf\in\NN$, the index $j^*$ remains the unique active index:
    \begin{align*}
        \Gamma_{j^*}^*(q_{j^*})
        =
        \max_{j\in[\dstate]}\Gamma_j^*(q_j)
        \text{\,.}
    \end{align*}
    Fix $\qbf\in\NN$, and write $Q=V\diag(\qbf)V^\top$. Then \cref{cor:hinf_val_coincide} applies to this $Q$. 
    Since $j^*$ is the unique active index, \cref{cor:hinf_val_coincide} yields
    \begin{align*}
        \max_{j\in[\dstate]}\Gamma_j\left(q_j,\ksafe_j(\qbf)\right)
        =
        \max_{j\in[\dstate]}\Gamma_j^*(q_j)
        =
        \Gamma_{j^*}^*(q_{j^*})
        =
        \gammainf(Q)^2
        \text{\,.}
    \end{align*}
    On the other hand, by definition of the scalar optimum,
    \begin{align*}
        \Gamma_{j^*}\left(q_{j^*},\ksafe_{j^{*}}(\qbf)\right)
        \geq
        \Gamma_{j^*}^*(q_{j^*})
        \text{\,.}
    \end{align*}
    Since the $j^*$-th term is one of the terms in the maximum, the previous display also implies that
    \begin{align*}
        \Gamma_{j^*}\left(q_{j^*},\ksafe_{j^{*}}(\qbf)\right)
        \leq
        \max_{j\in[\dstate]}\Gamma_j\left(q_j,\ksafe_j(\qbf)\right)
        =
        \Gamma_{j^*}^*(q_{j^*})
        \text{\,.}
    \end{align*}
    Therefore,
    \begin{align*}
        \Gamma_{j^*}\left(q_{j^*},\ksafe_{j^{*}}(\qbf)\right)
        =
        \Gamma_{j^*}^*(q_{j^*})
        =
        \gammainf(Q)^2
        \text{\,.}
    \end{align*}
    By \cref{prop:scalar_hinf_sol}, the scalar objective $\Gamma_{j^*}(q_{j^*},\cdot)$ has the unique minimizer
    \begin{align*}
        k_{j^*}^*(q_{j^*})
        =
        -\frac{a_{j^*}}{|a_{j^*}|}\frac{b_{j^*}}{r_{j^*}(1-|a_{j^*}|)}q_{j^*}
        \text{\,.}
    \end{align*}
    Therefore,
    \begin{align*}
        \Gamma_{j^*}\left(q_{j^*},\ksafe_{j^{*}}(\qbf)\right)
        =
        \Gamma_{j^*}\left(q_{j^*},k_{j^*}^{*}(q_{j^*})\right)
        =
        \Gamma_{j^*}^*(q_{j^*})
        =
        \gammainf(Q)^2
        \text{\,.}
    \end{align*}
    Since both $\ksafe_{j^{*}}(\qbf)$ and $k_{j^*}^{*}(q_{j^*})$ attain the minimum value of the same scalar objective, uniqueness of the minimizer implies
    \begin{align*}
        \ksafe_{j^{*}}(\qbf)
        =
        k_{j^*}^*(q_{j^*})
        =
        -\frac{a_{j^*}}{|a_{j^*}|}\frac{b_{j^*}}{r_{j^*}(1-|a_{j^*}|)}q_{j^*}
        \text{\,.}
    \end{align*}
\end{proof}

\subsection{Lipschitz Constant of the $\Kinf(Q)$ Mapping}
\label{app:separation:lip}

We now turn the local coincidence from \cref{app:separation:coincide} into a lower bound on the Lipschitz constant of the coupled block DGARE mapping.
The key step is to construct an interior diagonal cost matrix $Q^{(0)}$ for which a prescribed coordinate is uniquely active.
On the corresponding neighborhood, the coupled block DGARE mapping and the stacked mapping coincide on that coordinate, so the DGARE entry inherits the explicit affine form of the stacked mapping.
This produces an explicit derivative, which \cref{lemma:diag_hinf_lip} converts into a lower bound on $\lip(\Kinf)$.
In \cref{app:separation:strict}, we will apply the next result with the index $j^*$ furnished by the alignment condition of \cref{def:align}.

\begin{proposition}\label{prop:hinf_lip_lower_bound}
    Assume \cref{ass:commute}. Fix an index $j^*\in[\dstate]$ such that $a_{j^*}\neq 0$, $b_{j^*}\neq 0$, and $d_{j^*}\neq 0$ and define
    \begin{align*}
        \bar q_{j^*}
        :=
        \frac{1}{2}\min\left\{1,\frac{|a_{j^*}|r_{j^*}(1-|a_{j^*}|)}{b_{j^*}^{2}}\right\}
        \text{\,.}
    \end{align*}
    Then there exists $\eta>0$ such that, for $\qbf^{(0)}\in\BR^{\dstate}$ defined by
    \begin{align*}
        \qbf^{(0)}
        :=
        \bar q_{j^*}\ebf_{j^*}
        +
        \eta\sum_{j\in[\dstate]\setminus\{j^*\}}\ebf_j
        \text{\,,}
    \end{align*}
    the matrix $Q^{(0)}:=V\diag(\qbf^{(0)})V^\top$ belongs to $\Theta_V$, the vector $\qbf^{(0)}$ lies in the interior of $\D_Q$, and $j^*$ is the unique active index at $Q^{(0)}$.
    Moreover, there exists a neighborhood $\NN\subseteq\D_Q$ of $\qbf^{(0)}$ such that for every $\qbf\in\NN$, the DGARE-synthesized mapping $\Kinf(Q)$ and the stacked mapping $K^*(Q)$ coincide on the coordinate $j^*$, meaning that
    \begin{align*}
        \ksafe_{j^{*}}(\qbf)
        =
        k_{j^*}^{*}(q_{j^*})
        =
        -\frac{a_{j^*}}{|a_{j^*}|}\frac{b_{j^*}}{r_{j^*}(1-|a_{j^*}|)}q_{j^*}
        \text{\,,}
    \end{align*}
    and consequently
    \begin{align*}
        \lip(\Kinf)
        \geq
        \frac{|b_{j^*}|}{r_{j^*}(1-|a_{j^*}|)}
        \text{\,.}
    \end{align*}
\end{proposition}
\begin{proof}
    By construction, $\bar q_{j^*}\in (0,\frac{|a_{j^*}|r_{j^*}(1-|a_{j^*}|)}{b_{j^*}^{2}})$.
    Therefore since $a_{j^*}\neq 0$ and $b_{j^*}\neq 0$, \cref{prop:scalar_hinf_sol} applies at the coordinate $j^*$ and gives
    \begin{align*}
        \Gamma_{j^*}^*(\bar q_{j^*})
        =
        \frac{d_{j^*}^{2}}{(1-|a_{j^*}|)^{2}}
        \cdot
        \frac{\bar q_{j^*}}{1+\frac{b_{j^*}^{2}}{r_{j^*}(1-|a_{j^*}|)^{2}}\bar q_{j^*}}
        >
        0
        \text{\,,}
    \end{align*}
    where strict positivity uses $d_{j^*}\neq 0$ and $\bar q_{j^*}>0$.
    By continuity of each $\Gamma_{j}^{*}(\cdot)$ on $[0,1]$, and by the fact that $\Gamma_{j}^{*}(0)=0$ for every $j\in[\dstate]$, there exists $\eta_{j^*}>0$ small enough so that for every $j\in [\dstate]$,
    \begin{align*}
        \Gamma_{j}^{*}(\eta_{j^*})
        <
        \Gamma_{j^*}^{*}(\bar q_{j^*})
        \text{\,,}
    \end{align*}
    and we may additionally require that
    \begin{align*}
        \eta_{j^*}
        <
        \sqrt{\frac{1-\bar q_{j^*}^{2}}{\dstate}}
        \text{\,.}
    \end{align*}
    Then, defining $\qbf^{(0)}$ as described above with $\eta_{j^*}$, all coordinates of $\qbf^{(0)}$ are strictly positive and
    \begin{align*}
        \norm{\qbf^{(0)}}_2^2
        =
        \bar q_{j^*}^{2}
        +
        (\dstate-1)\eta_{j^*}^{2}
        <
        \bar q_{j^*}^{2}
        +
        \frac{\dstate-1}{\dstate}(1-\bar q_{j^*}^{2})
        =\frac{\dstate-1+\bar q_{j^*}^{2}}{\dstate}
        <
        \frac{\dstate-1+\frac{1}{2}}{\dstate}
        <
        1
        \text{\,,}
    \end{align*}
    where we used the fact that $\bar q_{j^*}<\frac{1}{2}$ in the last inequality.
    Hence $\qbf^{(0)}$ lies in the interior of $\D_Q$, and therefore $Q^{(0)}\in\Theta_V$.
    Thus, \cref{lemma:unique_active_index} applies at $\qbf^{(0)}$, and in particular there exists a neighborhood $\NN\subseteq\D_Q$ of $\qbf^{(0)}$ such that for every $\qbf\in\NN$,
    \begin{align*}
        \ksafe_{j^{*}}(\qbf)
        =
        k_{j^*}^{*}(q_{j^*})
        =
        -\frac{a_{j^*}}{|a_{j^*}|}\frac{b_{j^*}}{r_{j^*}(1-|a_{j^*}|)}q_{j^*}
        \text{\,.}
    \end{align*}
    Concretely, for all sufficiently small $h\in\BR$ such that $\qbf^{(0)}+h\ebf_{j^*}\in\NN$, we obtain that
    \begin{align*}
        \ksafe_{j^{*}}(\qbf^{(0)}+h\ebf_{j^*})
        =
        -\frac{a_{j^*}}{|a_{j^*}|}\frac{b_{j^*}}{r_{j^*}(1-|a_{j^*}|)}(q_{j^*}^{(0)}+h)
        \text{\,.}
    \end{align*}
    In particular, the scalar function $g_{j^*}:\BR\to\BR$ defined for any $h\in\BR$ by
    \begin{align*}
        g_{j^*}(h):=\ksafe_{j^{*}}(\qbf^{(0)}+h\ebf_{j^*})
    \end{align*}
    is differentiable at $h=0$ with derivative $g_{j^*}'(0)=-\frac{a_{j^*}}{|a_{j^*}|}\frac{b_{j^*}}{r_{j^*}(1-|a_{j^*}|)}$.
    Applying \cref{lemma:diag_hinf_lip} thus yields
    \begin{align*}
        \lip(\Kinf)
        \geq
        |g_{j^*}'(0)|
        =
        \frac{|b_{j^*}|}{r_{j^*}(1-|a_{j^*}|)}
        \text{\,,}
    \end{align*}
    completing the proof.
\end{proof}

\subsection{Comparing the Lipschitz Constants}
\label{app:separation:strict}

We are now ready to conclude \cref{res:sep}.
In this subsection, we compare the lower bound on the DGARE-synthesized mapping Lipschitz constant with the upper bound on the LQR mapping Lipschitz constant.
The theorem then follows by comparing the resulting dimensionless prefactors.

Let $\alpha\in(0,1]$ be the alignment factor of the dynamical system, and let $j^*\in[\dstate]$ be a witnessing index as guaranteed by \cref{ass:regular}. Rearranging the condition on $|a_{j^{*}}|$,
\begin{align*}
    |a_{j^*}|
    \geq
    \frac{1}{\alpha}\left(\Anorm+\alpha-1\right)
    =
    1-\frac{1-\Anorm}{\alpha}
    \text{\,,}
\end{align*}
we obtain that $1-|a_{j^*}|\leq \frac{1-\Anorm}{\alpha}$.
Together with the conditions on $b_{j^{*}}$ and $r_{j^{*}}$ given by
\begin{align*}
    |b_{j^*}|
    \geq
    \alpha \Bnorm
    \text{\,,}\qquad
    r_{j^*}
    \leq
    \frac{1}{\alpha\Rinvnorm}
    \text{\,,}
\end{align*}
\cref{prop:hinf_lip_lower_bound} yields
\begin{align*}
    \lip(\Kinf)\geq
    \frac{|b_{j^*}|}{r_{j^*}(1-|a_{j^*}|)}\geq
    \frac{\alpha \Bnorm}{\frac{1}{\alpha\Rinvnorm}\cdot \frac{1-\Anorm}{\alpha}}=
    \alpha^{3}\cdot \frac{\Bnorm\Rinvnorm}{1-\Anorm}
    \text{\,.}
\end{align*}
\cref{res:ub} implies that
\begin{align*}
    \lip(\Klqr)
    \leq
    2\Anorm\Bnorm\Rinvnorm\left(1+2\Bnorm^{2}\Rinvnorm\right)
    \text{\,.}
\end{align*}
Combining this with the lower bound on $\lip(\Kinf)$ and the fact that $1-\Anorm\in(0,1)$ gives
\begin{align*}
    \lip(\Kinf)
    &\geq
    \alpha^{3}\cdot \frac{\Bnorm\Rinvnorm}{2(1-\Anorm)\Anorm\left(1+2\Bnorm^{2}\Rinvnorm\right)}\cdot 2\Anorm\left(1+2\Bnorm^{2}\Rinvnorm\right)\\
    &\geq \frac{\alpha^{3}}{\Anorm\left(2+4\Bnorm^{2}\Rinvnorm\right)}\cdot \lip(\Klqr)
    \text{\,,}
\end{align*}
which proves \cref{res:sep}.

	\section{Auxiliary Theorems, Lemmas and Definitions}
\label{app:aux}
In this appendix we provide additional theorems, lemmas and definitions used throughout our proofs.

\begin{lemma}\label{lemma:Schur_comp}
    Let $M\in\BR^{m,m}$ and $N\in\BR^{n,n}$ be symmetric matrices and let $E\in\BR^{m,n}$. 
    Consider the symmetric block matrix $T$ given by
    \begin{align*}
        T=
        \begin{pmatrix}
            M\,& E\\
            E^{\top}\, &N
        \end{pmatrix}
        \text{\,.}
    \end{align*}
    If $T\succeq 0$ and $N\succ 0$, then the Schur complement of the block $N$ of the matrix $T$, given by
    \begin{align*}
        M-EN^{-1}E^{\top}
        \text{\,,}
    \end{align*}
    is PSD.
\end{lemma}
\begin{proof}
    This is a standard property of the Schur complement given by the Haynsworth inertia additivity formula---see \eg~\citep{haynsworth1968inertia}.
\end{proof}

\begin{lemma}\label{lemma:psd_order}
    Let $A, B\in\BR^{m,m}$ be symmetric matrices. 
    If $A\succeq 0$, $B\succeq 0$ and $A-B\succeq 0$, then $\norm{A}_{2}\geq \norm{B}_{2}$.
\end{lemma}
\begin{proof}
    Since $A-B\succeq 0$, we have that
    \begin{align*}
        \xbf^{\top}(A-B)\xbf\geq 0
        \text{\,,}
    \end{align*}
    for any $\xbf\in\BR^{m}$. 
    Hence, we have that
    \begin{align*}
        \xbf^{\top}A\xbf\geq \xbf^{\top}B\xbf
        \text{\,.}
    \end{align*}
    Recalling that for symmetric matrices, $\norm{M}_{2}=\max_{\xbf\in\BR^{m}}\frac{\xbf^{\top}M\xbf}{\xbf^{\top}\xbf}$, we obtain that
    \begin{align*}
        \norm{A}_{2}= \max_{\xbf\in\BR^{m}}\frac{\xbf^{\top}A\xbf}{\xbf^{\top}\xbf}\geq \max_{\xbf\in\BR^{m}}\frac{\xbf^{\top}B\xbf}{\xbf^{\top}\xbf}=\norm{B}_{2}
        \text{\,,}
    \end{align*}
    as required.
\end{proof}

\begin{lemma}\label{lemma:Rayleigh-Ritz}
    Let $M\in\BR^{m,m}$ be a symmetric matrix. For any $\xbf\in\BR^{m}$ it holds that
    \begin{align*}
        \lambda_{\min}(M)\|\xbf\|_{2}^{2}\leq \xbf^{\top}M\xbf\leq \norm{M}_{2}\|\xbf\|_{2}^{2}
        \text{\,.}
    \end{align*}
    \end{lemma}
    \begin{proof}
        If $\xbf$ is the all-zeros vector $\0$ then the claim is trivial.
        Otherwise, by the Rayleigh quotient theorem~\citep{horn2012matrix} it holds that
        \begin{align*}
            \lambda_{\min}(M)\le\frac{\xbf^{\top}M\xbf}{\xbf^{\top}\xbf}\le\lambda_{\max}(M)
            \text{\,.}
        \end{align*}
        Noting that $\lambda_{\max}(M)=\norm{M}_{2}$ for symmetric matrices, the proof follows from rearranging.
    \end{proof}

\begin{lemma}\label{lemma:frob_op_norm}
    Let $M\in\BR^{m,n}$ and $N\in\BR^{n,k}$. It holds that
    \begin{align*}
        \sigma_{\min}(M)\norm{N}_{F}\leq \norm{MN}_{F}\leq \norm{M}_{2}\norm{N}_{F}
    \end{align*}
    and 
    \begin{align*}
        \norm{M}_{F}\sigma_{\min}(N)\leq \norm{MN}_{F}\leq \norm{M}_{F}\norm{N}_{2}
        \text{\,.}
    \end{align*}
\end{lemma}
\begin{proof}
    By definition of the Frobenius norm, it holds that
    \begin{align*}
        \norm{MN}_{F}^{2}=\Tr\left(\left(MN\right)^{\top}MN\right)=\Tr\left(N^{\top}M^{\top}MN\right)
    \end{align*}
    For $i\in[k]$ let $\vbf_{i}\in\BR^{n}$ be the $i$-th column of $N$. 
    Then by definition of $\Tr(\cdot)$ we obtain that
    \begin{align*}
        \Tr\left(N^{\top}M^{\top}MN\right)=\sum_{i=1}^{k}\vbf_{i}^{\top}M^{\top}M\vbf_{i}\leq \sum_{i=1}^{k}\norm{M^{\top}M}_{2}\|\vbf_{i}\|_{2}^{2}
    \end{align*}
    where the inequality stems from $M^{\top}M$ being symmetric and from \cref{lemma:Rayleigh-Ritz}. 
    Finally, it holds that
    \begin{align*}
        \sum_{i=1}^{k}\norm{M^{\top}M}_{2}\|\vbf_{i}\|_{2}^{2}\leq \norm{M}_{2}^{2}\norm{N}_{F}^{2}
    \end{align*}
    where we used the Cauchy-Schwarz inequality and the definition of the Frobenius norm. 
    Hence, the right inequality of the first claim follows by taking square root over both sides. 
    The left inequality follows from a symmetric argument, and by observing that $\lambda_{\min}(M^{\top}M)=\sigma_{\min}(M)^{2}$. 
    Next, by the cyclic property of $\Tr(\cdot)$ it holds that
    \begin{align*}
        \Tr\left(N^{\top}M^{\top}MN\right)=\Tr\left(MNN^{\top}M^{\top}\right)
        \text{\,,}
    \end{align*}
    thus the second claim follows from identical arguments using $M$'s rows instead of $N$'s columns and the symmetric $NN^{\top}$ instead of $M^{\top}M$.
\end{proof}

\begin{lemma}\label{lemma:inv_identity}
    Let $M,N\in\BR^{m,m}$ be invertible matrices. 
    It holds that
    \begin{align*}
        M^{-1}-N^{-1}=M^{-1}(N-M)N^{-1}
        \text{\,.}
    \end{align*}
\end{lemma}
\begin{proof}
    By direct algebra
    \begin{align*}
        M^{-1}(N-M)N^{-1}=M^{-1}NN^{-1}-M^{-1}MN^{-1}=M^{-1}-N^{-1}
        \text{\,.}
    \end{align*}
\end{proof}

\begin{theorem}\label{thm:plancherel}
    Let $\{z_{t}\}_{t=0}^{\infty}$ be a sequence of complex numbers such that $\sum_{t=0}^{\infty}|z_{t}|^{2}<\infty$.
    Then the discrete-time Fourier transform $\hat{z}$ (\cref{def:dtft}), understood as the $L^{2}$ limit of the transforms of finite truncations, is an $L^{2}([0,2\pi])$ function and satisfies
    \begin{align*}
        \sum_{t=0}^{\infty}|z_{t}|^{2}=\frac{1}{2\pi}\int_{0}^{2\pi}|\hat{z}(\omega)|^{2}d\omega
        \text{\,.}
    \end{align*}
    If in addition $\sum_{t=0}^{\infty}|z_t|<\infty$, then the defining series for $\hat z$ converges uniformly and $\hat z$ is continuous on $[0,2\pi]$.
\end{theorem}
\begin{proof}
    This is the standard Plancherel theorem for the discrete-time Fourier transform---see \eg~\citep{plancherel1910contribution}.
\end{proof}

\begin{lemma}\label{lemma:dtft_product}
    Let $\{z_{t}\}_{t=0}^{\infty}$ and $\{y_{t}\}_{t=0}^{\infty}$ be sequences of complex numbers such that $\sum_{t=0}^{\infty}|z_{t}|^{2}<\infty$ and $\sum_{t=0}^{\infty}|y_{t}|<\infty$.
    Consider the sequence $\{s_{t}\}_{t=0}^{\infty}$ given by $s_{t}=\sum_{\tau=0}^{t}z_{t-\tau}y_{\tau}$.
    Then the discrete-time Fourier transform of $\{s_{t}\}_{t=0}^{\infty}$ (\cref{def:dtft}) is given, for almost every $\omega\in[0,2\pi]$, by
    \begin{align*}
        \hat{s}(\omega)=\hat{z}(\omega)\hat{y}(\omega)
        \text{\,.}
    \end{align*}
\end{lemma}
\begin{proof}
    This is a standard property of the discrete-time Fourier transform. See \citep{oppenheim2010dtsp}.
\end{proof}

\begin{lemma}\label{lemma:dtft_approx}
    Let $f:[0,2\pi]\to\BR$ be a continuous function and let $\omega^{*}\in[0,2\pi]$. For $n\in\BN$ and $t\in\BN\cup\{0\}$ define $w^{(n)}_{t}$ as
    \begin{align*}
        w^{(n)}_{t}=
        \begin{cases}
            e^{i\omega^{*}t}\,, & t\in[0,n-1]\\
            0\,, & \text{otherwise}\\
        \end{cases}
        \text{\,.}
    \end{align*}
    Then it holds that
    \begin{align*}
        \lim_{n\to\infty}\frac{\int_{0}^{2\pi}f(\omega)|\hat{w}^{(n)}(\omega)|^{2}d\omega}{\int_{0}^{2\pi}|\hat{w}^{(n)}(\omega)|^{2}d\omega}=f(\omega^{*})
    \end{align*}
    where $\hat{w}^{(n)}(\omega)$ is the discrete-time Fourier transform of $w^{(n)}_{t}$ (\cref{def:dtft}).
\end{lemma}
\begin{proof}
    Up to normalization, $\hat{w}^{(n)}(\omega)$ is the Fej\'er kernel centered at $\omega^{*}$. 
    Therefore the ratio is a Fej\'er-kernel average of $f$, and by Fej\'er's theorem it converges to $f(\omega^{*})$.
    See Chapter I, Section 3.1 in~\citep{katznelson2004harmonic} for more details.
\end{proof}

	\section{From Controller Maps to Policy Maps}
\label{app:policy}

The theoretical results of \cref{res:ub,res:sep} are phrased in terms of the controller mappings $Q\mapsto \Klqr(Q)$ and $Q\mapsto \Kinf(Q)$.
This is mainly a matter of convenience: in both settings we study, the policies are of the form $\piteach(\xbf;Q) = K(Q)\xbf$, where $K(Q)$ denotes either $\Klqr(Q)$ or $\Kinf(Q)$.
In this linear state-feedback setting, there is little distinction between learning the controller mapping and learning the induced policy mapping, since the policy is fully specified by the matrix.
Accordingly, once the state is restricted to a bounded set, regularity properties of the controller mapping readily induce corresponding regularity properties of the policy mapping.

Fix a radius $\rho>0$ and consider the bounded state space $\X_{\rho}:=\{\xbf\in\BR^{\dstate}:\norm{\xbf}_{2}\leq \rho\}$.
Bounded states are natural when discussing Lipschitzness of the policy mapping, since the policy is linear: without controlling $\norm{\xbf}_{2}$, variations in the policy output can be dominated by large states even when the task dependence of the controller matrix itself is well behaved.
Since the policy depends on $Q$ only via the matrix $K(Q)$, when $\xbf$ is fixed, perturbing $Q$ changes the action only through the corresponding perturbation of $K(Q)$, scaled by the size of $\xbf$.
Such perturbations are controlled by the Lipschitzness of the mapping $Q\mapsto K(Q)$.
If $\xbf$ varies as well, there is one additional effect: the same controller being applied to a different state.
On $\X_{\rho}$, this contribution is controlled by a uniform bound on the norm of $K(Q)$.
Consequently, the same reasoning underlying the controller-mapping analysis yields similar qualitative regularity results for the induced policy mappings.

Our upper bound for the mapping $Q\mapsto \Klqr(Q)$ yields an upper bound of the same nature for the policy mapping $(Q,\xbf)\mapsto \Klqr(Q)\xbf$ on $\Theta\times\X_{\rho}$: the action varies smoothly with the task whenever the corresponding controller matrix does.
Likewise, the lower bound for the mapping $Q\mapsto \Kinf(Q)$ identifies task perturbations for which the controller matrix changes sharply.
Evaluating the induced policy at a bounded non-zero state preserves this sensitivity at the action level.
Therefore, the mechanism behind our separation result leads to a similar qualitative separation for policy mappings.

	\section{Additional Experiments}
\label{app:exper}

The experiments in \cref{sec:exper:lqr} corroborate our theoretical analysis, verifying our separation result holds even when the assumptions of \cref{sec:analysis} are violated, and showing that although imitating safe and unsafe teachers on a single (fixed) task is similarly straightforward, generalizing across tasks is considerably more difficult with the safe teacher.
In this appendix, we provide additional experimental results omitted from \cref{sec:exper:lqr}.
\cref{fig:lip_dim=8} extends the experiments in \cref{fig:lip_dim=4} to systems of dimension $\dstate=8$.
\cref{fig:lip_breakdown_dim=4} provides a detailed breakdown of the results in the second plot of \cref{fig:lip_dim=4}, separating between systems of different parameter norms.
\cref{fig:lip_breakdown_dim=8} similarly provides such a breakdown for the second plot of \cref{fig:lip_dim=8}.
\cref{tab:lqr_dim=15,tab:lqr_dim=20} extend the linear-quadratic experiments in \cref{tab:imitate} to systems of dimension $\dstate=15$ and $\dstate=20$, respectively.
Lastly, \cref{tab:lqr_train} supplements the results of \cref{tab:imitate,tab:lqr_dim=15,tab:lqr_dim=20} by reporting the final training performance of the trained students in the linear-quadratic settings.
\cref{tab:nonlinear_train} similarly supplements the results of \cref{tab:imitate} by reporting the final training performance of the trained students in the non-linear settings.

\begin{figure}[t]
    \begin{center}
        \vspace{1.5mm}
        \begin{center}
            \includegraphics[width=1\textwidth]{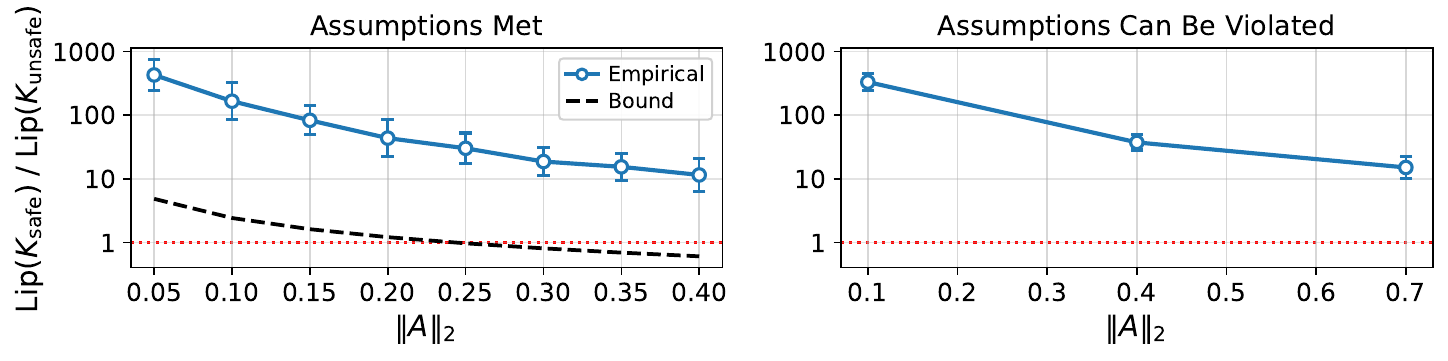}
        \end{center}
        \vspace{-2mm}
    \end{center}
    \caption{
    Demonstration of the separation of Lipschitz constants derived in \cref{res:sep}---the Lipschitz constant of the safe mapping $\Kinf(\cdot)$ is strictly larger than that of the unsafe mapping $\Klqr(\cdot)$, even when the coefficient is smaller than one or the technical conditions are violated.
    This figure is identical to \cref{fig:lip_dim=4}, except that the systems are of dimension $\dstate=8$.
    For further details, see \cref{fig:lip_dim=4} as well as \cref{app:details:lip}.
    }
    \label{fig:lip_dim=8}
\end{figure}

\begin{figure}[t]
    \begin{center}
        \vspace{1.5mm}
        \begin{center}
            \includegraphics[width=1\textwidth]{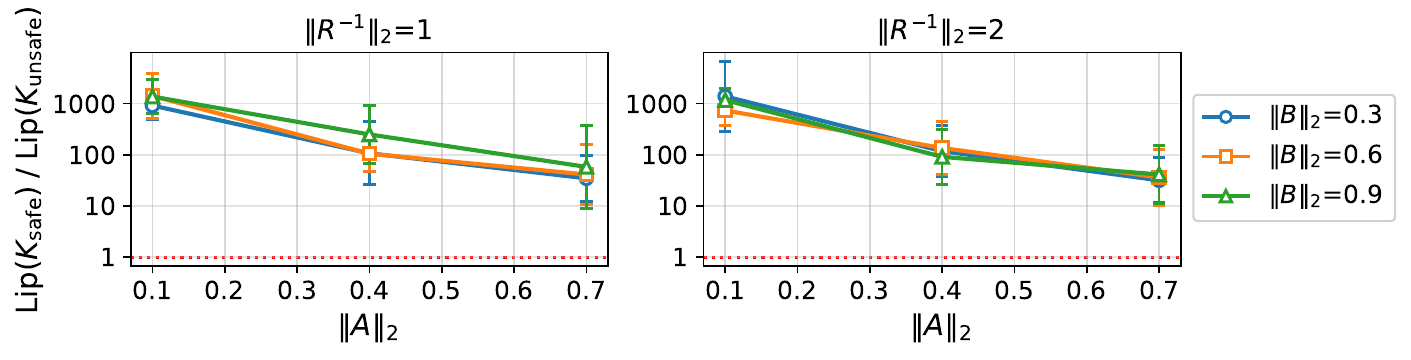}
        \end{center}
        \vspace{-2mm}
    \end{center}
    \caption{
        In the theoretically analyzed setting (linear-quadratic control with $\Hinf$-robustness; \cref{sec:prelim:setup_theory}), the mapping from task specification to an optimal controller has higher Lipschitz constant with safety requirements than without, \ie, the Lipschitz constant of~$\Ksafe ( \cdot )$ is greater than that of~$\Kunsafe ( \cdot )$.
        This figure provides a detailed breakdown of the results in the second plot of \cref{fig:lip_dim=4} (random systems that do not ensure our assumptions), with separate plots for different values of $\|R^{-1}\|_{2}$, and separate curves for different values of $\|B\|_{2}$.
        Reported quantities are averages across $5$ random draws of data splits and systems with $\|D\|_{2}=1$ and the prescribed values for $\|B\|_{2}$ and $\|R^{-1}\|_{2}$, with error bars denoting standard deviations in log space.
        For further details, see \cref{fig:lip_dim=4} as well as \cref{app:details:lip}.
    }
    \label{fig:lip_breakdown_dim=4}
\end{figure}

\begin{figure}[t]
    \begin{center}
        \vspace{1.5mm}
        \begin{center}
            \includegraphics[width=1\textwidth]{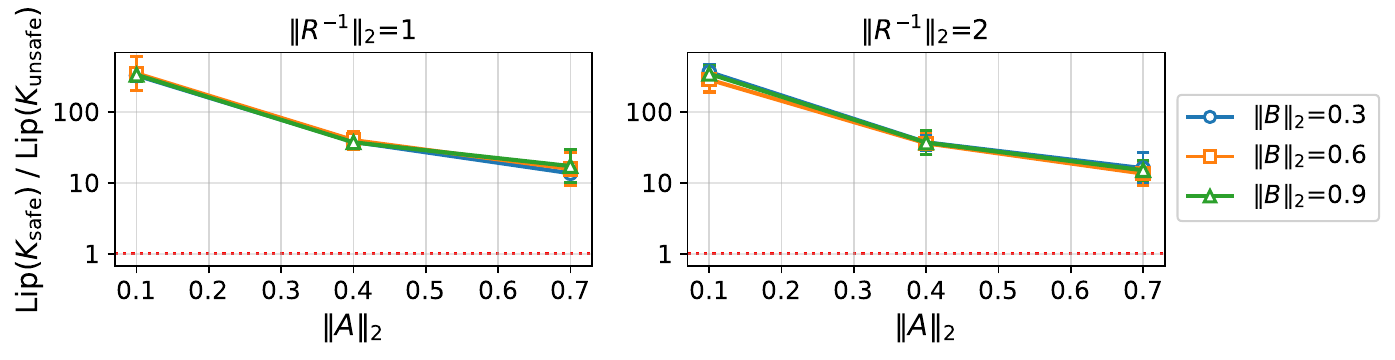}
        \end{center}
        \vspace{-2mm}
    \end{center}
    \caption{
        In the theoretically analyzed setting (linear-quadratic control with $\Hinf$-robustness; \cref{sec:prelim:setup_theory}), the mapping from task specification to an optimal controller has higher Lipschitz constant with safety requirements than without, \ie, the Lipschitz constant of~$\Ksafe ( \cdot )$ is greater than that of~$\Kunsafe ( \cdot )$.
        This figure is identical to \cref{fig:lip_breakdown_dim=4}, except that the breakdown provided is for the results in the second plot of \cref{fig:lip_dim=8}.
        For further details, see \cref{fig:lip_dim=8} as well as \cref{app:details:lip}.
    }
    \label{fig:lip_breakdown_dim=8}
\end{figure}

\begin{table}[t]
    \caption{In line with our theory, even when imitating safe and unsafe teachers on tasks seen in training is similarly straightforward, generalizing across tasks is considerably more difficult with the safe teacher.
    The rows in this table are identical to the first half of \cref{tab:imitate} (first to fourth rows), except that the systems used were of dimension $\dstate=15$.
    For further details see \cref{app:details:lqr}.
    }
    \centering
    \small
    \begin{tabular}{llcc}
    \toprule
    \textbf{Setting} & \textbf{Teacher} & \textbf{$\Thetatr$ Error} & \textbf{$\Thetate$ Error} \\
    \midrule
    \multirow{2}{*}{Infinite Sample} & Safe   & $-$ & $1357.56 \pm 301.68$ \\
                                 & Unsafe & $-$ & $\mathbf{27.07 \pm 3.38}$ \\
    \midrule
    \multirow{2}{*}{Finite Sample} & Safe   & $2.60 \pm 1.06$ & $3903.91 \pm 421.98$ \\
                         & Unsafe & $2.30 \pm 1.18$ & $\mathbf{49.68 \pm 15.37}$ \\
    \bottomrule
    \end{tabular}
    \label{tab:lqr_dim=15}
\end{table}

\begin{table}[t]
    \caption{In line with our theory, even when imitating safe and unsafe teachers on tasks seen in training is similarly straightforward, generalizing across tasks is considerably more difficult with the safe teacher.
    The rows in this table are identical to the first half of \cref{tab:imitate} (first to fourth rows), except that the systems used were of dimension $\dstate=20$.
    For further details see \cref{app:details:lqr}.
    }
    \centering
    \small
    \begin{tabular}{llcc}
    \toprule
    \textbf{Setting} & \textbf{Teacher} & \textbf{$\Thetatr$ Error} & \textbf{$\Thetate$ Error} \\
    \midrule
    \multirow{2}{*}{Infinite Sample} & Safe   & $-$ & $1463.38 \pm 272.74$ \\
                                 & Unsafe & $-$ & $\mathbf{46.91 \pm 10.42}$ \\
    \midrule
    \multirow{2}{*}{Finite Sample} & Safe   & $\mathbf{0.77 \pm 0.14}$ & $5551.38 \pm 886.98$ \\
                         & Unsafe & $2.52 \pm 0.78$ & $\mathbf{112.52 \pm 56.28}$ \\
    \bottomrule
    \end{tabular}
    \label{tab:lqr_dim=20}
\end{table}

\begin{table}[t]
    \caption{In line with our theory, even when imitating safe and unsafe teachers on tasks seen in training is similarly straightforward, generalizing across tasks is considerably more difficult with the safe teacher.
    This table supplements the results of the top half of \cref{tab:lqr_dim=15} by reporting the achieved training performance of the trained students.
    First pair of rows summarizes experiments in the infinite sample case, where the student imitates either the unsafe mapping $\Klqr(\cdot)$ or the safe mapping $\Kinf(\cdot)$.
    Second pair of rows summarizes experiments in the finite sample case, where the student imitates either the unsafe policy $\pi_{\phiunsafe}(\cdot)$ or the safe policy $\pi_{\phisafe}(\cdot)$.
    Third, fourth and fifth columns report the training mean square errors for the systems of dimension $\dstate=10$ (first half of \cref{tab:imitate}), dimension $\dstate=15$ (\cref{tab:lqr_dim=15}), and dimension $\dstate=20$ (\cref{tab:lqr_dim=20}), respectively.
    Throughout the table, imitation errors are reported as mean $\pm$ standard deviation.
    For further details see \cref{app:details:lqr}.
    }
    \centering
    \small
    \begin{tabular}{llccc}
    \toprule
    \textbf{Setting} & \textbf{Teacher} & \textbf{$\dstate=10$} & \textbf{$\dstate=15$} & \textbf{$\dstate=20$} \\
    \midrule
    \multirow{2}{*}{Infinite Sample} & Safe   & $1.21 \pm 0.66$ & $32.97 \pm 8.44$ & $187.80 \pm 41.48$ \\
                                 & Unsafe & $13.21 \pm 2.43$ & $25.90 \pm 3.49$ & $45.71 \pm 10.59$ \\
    \midrule
    \multirow{2}{*}{Finite Sample} & Safe   & $0.21 \pm 0.07$ & $1.50 \pm 0.65$ & $0.64 \pm 0.13$ \\
                         & Unsafe & $1.02 \pm 0.97$ & $2.01 \pm 1.22$ & $2.40 \pm 0.77$ \\
    \bottomrule
    \end{tabular}
    \label{tab:lqr_train}
\end{table}

\begin{table}[t]
    \caption{In line with our theory, even when imitating safe and unsafe teachers on tasks seen in training is similarly straightforward, generalizing across tasks is considerably more difficult with the safe teacher.
    This table supplements the results of the bottom half of \cref{tab:imitate} by reporting the achieved training performance of the trained students.
    First three rows summarize experiments in the quadcopter navigation setting (\cref{sec:exper:quadcopter}), where the student imitates the respective teacher policies.
    The last two rows summarize experiments in the LLM-based agentic CRM setting (\cref{sec:exper:LLM}), where the student imitates the respective teacher agents.
    Third column reports training mean square errors for the quadcopter setting, and training classification errors for the LLM setting.
    Throughout the table, imitation errors are reported as mean $\pm$ standard deviation.
    For further details see \cref{app:details:quadcopter,app:details:LLM}.
    }
    \centering
    \small
    \begin{tabular}{llc}
    \toprule
    \textbf{Setting} & \textbf{Teacher} & \textbf{Training Error}\\
    \midrule
    \multirow{3}{*}{Quadcopter} & Safe (Risk Avoidance) & $0.1 \pm 0$ \\
                                 & Safe (Risk Handling)  & $0.1 \pm 0$ \\
                                 & Unsafe               & $0.1 \pm 0$ \\
    \midrule
    \multirow{2}{*}{LLM} & Safe   & $1.38 \pm 0.54$ \\
                         & Unsafe & $1.00 \pm 0.25$ \\
    \bottomrule
    \end{tabular}
    \label{tab:nonlinear_train}
\end{table}

	\section{Implementation Details}
\label{app:details}

We provide implementation details omitted from \cref{sec:exper,app:exper}.
\ifdefined\CAMREADY
    Code for reproducing our experiments will be made available at \url{https://github.com/Tomerslortau/agentic-safety-generalization}. 
\fi

\subsection{Separation between the Lipschitz Constants of $\Ksafe(\cdot)$ and~$\Kunsafe(\cdot)$}
\label{app:details:lip}

In this section, we provide implementation details for the simulated experiments presented in \cref{fig:lip_dim=4,fig:lip_dim=8,fig:lip_breakdown_dim=4,fig:lip_breakdown_dim=8}. 
All experiments were implemented using NumPy~\citep{harris2020array} and SciPy~\citep{scipy2020}, and carried out on a standard multi-core CPU.

\subsubsection{System and Task Generation}
\label{app:details:lip:gen}

\cref{fig:lip_dim=4,fig:lip_dim=8} report two sets of experiments. 
In the experiments where \cref{ass:stable,ass:commute,ass:regular} are met, we fix the system constants
\begin{align*}
    \norm{B}_2 = 0.5 \,, \qquad
    \norm{R^{-1}}_2 = 1 \,, \qquad
    \norm{D}_2 = 1 \,, \qquad
    \alignfactor = 0.9 \,,
\end{align*}
and vary $\norm{A}_2 \in \{0.05,0.1,0.15,0.2,0.25,0.3,0.35,0.4\}$. 
In the experiments where these assumptions are violated, the figures show results aggregated over $90$ total experiments, consisting of $5$ seeds for each combination of values for
\begin{align*}
    \norm{A}_2 \in \{0.1,0.4,0.7\} \,, \qquad
    \norm{B}_2 \in \{0.3,0.6,0.9\} \,, \qquad
    \norm{R^{-1}}_2 \in \{1,2\} \,,
\end{align*}
while $\norm{D}_2$ is always fixed to $1$.
\cref{fig:lip_breakdown_dim=4,fig:lip_breakdown_dim=8} show a breakdown of these results for each individual parameter combination.

For each choice of system constants and seed, we generate a random system as follows. 
For experiments where \cref{ass:stable,ass:commute,ass:regular} are violated (second plots of \cref{fig:lip_dim=4,fig:lip_dim=8} and \cref{fig:lip_breakdown_dim=4,fig:lip_breakdown_dim=8}), we draw the entries of $A$, $B$, and $D$ i.i.d. from a Gaussian distribution and then rescale the resulting matrices to match the prescribed spectral norms. 
The matrix $R$ is generated as a positive semidefinite matrix by drawing a matrix $M$ with i.i.d. entries from a Gaussian distribution and then setting $R = M M^\top$. $R$ is then rescaled so that its minimum eigenvalue matches the prescribed value.

For experiments where \cref{ass:stable,ass:commute,ass:regular} are met (first plots of \cref{fig:lip_dim=4,fig:lip_dim=8}), we first generate $R$ in the same way and compute its eigendecomposition $R = V \Lambda_R V^\top$, with the eigenvalues ordered increasingly. 
Next, we generate diagonal matrices $\Lambda_A$, $\Lambda_B$, and $\Lambda_D$ by drawing diagonal entries i.i.d. from a Gaussian distribution and then rescaling the resulting matrices to match the prescribed spectral norms. 
We then set $A$, $B$, and $D$ as $A = V \Lambda_A V^\top$, $B = V \Lambda_B V^\top$, and $D = V \Lambda_D V^\top$, respectively.
Lastly, the first coordinate, corresponding to the minimum eigenvalue direction of $R$, is used as the aligned direction: we enforce the first coordinate of $A$ and $B$ to be aligned with the first coordinate of $R$, namely, we enforce $|\lambda_1(A)| \geq \alignfactor^{-1}(\norm{A}_2+\alignfactor-1)$ and $|\lambda_1(B)| \geq \alignfactor\norm{B}_2$.
No additional alignment constraint is imposed on $D$.

For each sampled system, we generate a random set of task matrices $Q$ using the same procedure for all experiments. 
Denote $N_Q = 100m^2$ where $m$ is the state dimension. 
When a positive minimum eigenvalue is imposed, each task matrix is generated by sampling a random orthogonal basis and assigning nonnegative eigenvalues satisfying the desired lower-eigenvalue and Frobenius-norm constraints; when the minimum eigenvalue is $0$, we instead  draw a matrix $M$ with i.i.d. entries from a Gaussian distribution, set $Q = M M^\top$, normalize it in Frobenius norm, and scale it by a radius in the prescribed interval.
We form the set of task matrices by combining three batches: \emph{(i)} $N_Q$ positive semidefinite matrices with minimum eigenvalue $0.1$ and Frobenius norm exactly $1$; \emph{(ii)} $N_Q$ positive semidefinite matrices with minimum eigenvalue $0$ and Frobenius norm at most $1$; and \emph{(iii)} $N_Q$ positive semidefinite matrices with minimum eigenvalue $0$ and Frobenius norm at most $\norm{A}_2$. 
This yields a total of $3N_Q = 300m^2$ task matrices per system, namely $4{,}800$ when $m=4$ and $19{,}200$ when $m=8$. 
We found that this generation procedure produced a sufficiently broad range of task perturbations to obtain a meaningful finite-sample approximation of the Lipschitz behavior, while maintaining a feasible computational burden.

\subsubsection{Controller Computation}
\label{app:details:lip:comp}

For every $Q$, the matrix $\Klqr(Q)$ is obtained using a standard numerical solver for the DARE (\cref{eq:dare}) implemented in SciPy~\citep{scipy2020}. 
The matrix $\Kinf(Q)$ is numerically approximated by searching for the minimum feasible disturbance level $\gamma$: for a given $\gamma$, we use a fixed-point iteration to compute an approximate solution to the DGARE (\cref{eq:dgare}), starting from the solution for the DARE (this is done to improve numerical stability and efficiency). 
We expand a feasible bracket multiplicatively and then apply bisection to approximate the smallest feasible value. 
If no feasible $\Kinf(Q)$ solution is found for a given $Q$, that task is excluded from the experiment.
Note that in our experiments, both $\Klqr(Q)$ and $\Kinf(Q)$ are negations of the expressions given in \cref{eq:klqr,eq:kinf} as is standard in some control theory literature (this does not affect the reported quantities).

\subsubsection{Lipschitz Approximation}
\label{app:details:lip:approx}

Given sampled pairs $(Q_i,K_i)$ from either controller family, the Lipschitz constant is approximated by taking the maximum over all unordered pairs in the sampled set:
\begin{align*}
    \widehat{L}
    :=
    \max_{i<j}
    \frac{\|K_i-K_j\|_F}{\|Q_i-Q_j\|_F + \epsilon}
    \text{\,.}
\end{align*}
Here, $\epsilon = 10^{-12}$ is a small positive constant added to the denominator to avoid division by zero. 
For $\Klqr(\cdot)$, this quantity is computed over the full sampled set of $Q$ matrices. 
For $\Kinf(\cdot)$, it is computed only over the subset of sampled $Q$ matrices for which a feasible $\Kinf$ controller is found. 
The reported statistic is the ratio $\widehat{L}_{\Kinf} / \widehat{L}_{\Klqr}$.

\subsection{Linear-Quadratic Control with $\Hinf$-Robustness}
\label{app:details:lqr}

In this section, we provide implementation details for the linear-quadratic experiments presented in the top half of \cref{tab:imitate}, and in \cref{tab:lqr_dim=15,tab:lqr_dim=20,tab:lqr_train}.
We first describe the data generation procedure in \cref{app:details:lqr:data}, and then provide training details for the infinite-sample and finite-sample experiments (\cref{app:details:lqr:infinite,app:details:lqr:finite}, respectively).

\subsubsection{Data Generation}
\label{app:details:lqr:data}

Both experiments use the same procedure for drawing system and task matrices, and for computing the controller matrices $\Klqr(Q)$ and $\Kinf(Q)$. 
We generate $A$ and $B$ by drawing their entries i.i.d. from a Gaussian distribution and then rescale the resulting matrices to match $\norm{A}_2 = 0.5$ and $\norm{B}_2 = 0.5$, respectively.
The matrix $R$ is generated by drawing a matrix $M$ with i.i.d. entries from a Gaussian distribution and then setting $R = M M^\top$. 
We then draw a uniform value from $[0.25, 1]$ and scale $R$ so that its Frobenius norm matches the prescribed value. We use $D=I$ for all experiments.
We generate task matrices $Q$ via the same procedure used to generate $R$. 
For each task $Q$, the controller matrices $\Klqr(Q)$ and $\Kinf(Q)$ are computed using the same procedures described in \cref{app:details:lip:comp}.

\subsubsection{Infinite Sample}
\label{app:details:lqr:infinite}

The first two rows of \cref{tab:lqr_dim=15,tab:lqr_dim=20,tab:lqr_train} report results for the infinite-sample experiments, where the student directly imitates either the safe mapping $\Ksafe(\cdot)$ or the unsafe mapping $\Kunsafe(\cdot)$ (as described in \cref{sec:prelim:setup_theory}).

The total number of sampled tasks we used is $800 \dstate^2$, that is, $80{,}000$ for $\dstate=10$, $180{,}000$ for $\dstate=15$, and $320{,}000$ for $\dstate=20$. 
We designate $80\%$ of the tasks for $\Thetatr$ and the remaining $20\%$ for $\Thetate$, so that each task $Q$ appears in exactly one split.
We use the same $\Thetatr$ and $\Thetate$ for learning both the safe and unsafe mappings.
The input is the flattened matrix $Q$, namely $\mathrm{vec}(Q) \in \R^{\dstate^2}$, and the target is either the flattened $\Ksafe(Q)$ or the flattened $\Kunsafe(Q)$.

For the student model, we use a $3$-layer fully connected MLP of width $2{,}048$ with residual connections, LayerNorm, and the GELU activation function.
The loss function for the student is the mean square error between the predicted $\hat{K}(Q)$ and the ground truth (either $\Ksafe(Q)$ or $\Kunsafe(Q)$).
We train using the Adam optimizer~\citep{kingma2015adam} with initial learning rate $10^{-5}$ and no weight decay.
The learning rate is reduced by a factor of $0.1$ upon reaching a plateau on the training loss, with a patience of $30$ epochs and a minimal learning rate of $10^{-9}$.
We cap the training at $10{,}000$ epochs for $\dstate=10$, $6{,}000$ epochs for $\dstate=15$, and $4{,}000$ epochs for $\dstate=20$.
Training and test losses are evaluated every $25$ epochs.
We employ early stopping, where training is stopped if either the training loss does not improve by more than $1\%$ (relative) for $5$ consecutive checks ($125$ epochs), if the per-epoch training loss drops below $10^{-5}$, or if the learning rate drops below the minimum.
Batch size is $256$ for $\dstate=10$ and $1{,}024$ for $\dstate \in \{15, 20\}$, applied to both training and evaluation.

We evaluate the trained students using the normalized mean square error metric, scaled by $10^{4}$ for convenience:
\begin{align*}
    \L_{\text{eval}}(\Theta) := \frac{10^{4}}{|\Theta|} \sum_{Q \in \Theta} \frac{\norm{\hat{K}(Q)-K^{*}(Q)}_{F}^{2}}{\norm{K^{*}(Q)}_{F}^{2} + \epsilon}
\end{align*}
where $\Theta$ is either $\Thetatr$ or $\Thetate$, $\hat{K}(Q)$ is the student prediction, and $K^{*}(Q)$ is the ground truth (either $\Ksafe(Q)$ or $\Kunsafe(Q)$). 
$\epsilon = 10^{-12}$ is a softening constant to avoid division by zero.
We report on held-out tasks $\Thetate$ in the last column of \cref{tab:lqr_dim=15,tab:lqr_dim=20,tab:lqr_train}, and on training tasks $\Thetatr$ in the first column of \cref{tab:lqr_train}.

\subsubsection{Finite Sample}
\label{app:details:lqr:finite}

The second pair of rows in \cref{tab:lqr_dim=15,tab:lqr_dim=20,tab:lqr_train} reports results for the finite-sample experiments, where the student learns a policy function that maps the state $\xbf$ and the task $Q$ to the action $\ubf$.
This is done by imitating demonstrations from either the safe policy $\pi_{\phisafe}(\cdot)$ or the unsafe policy $\pi_{\phiunsafe}(\cdot)$ (as described in \cref{sec:prelim:setup_theory}).

The total number of sampled tasks we used is $10 \dstate^2$, that is, $1{,}000$ for $\dstate=10$, $2{,}250$ for $\dstate=15$, and $4{,}000$ for $\dstate=20$. 
We designate $80\%$ of the tasks for $\Thetatr$ and the remaining $20\%$ for $\Thetate$, so that each task $Q$ appears in exactly one split.
We use the same $\Thetatr$ and $\Thetate$ for learning both the safe and unsafe mappings.
For each training task, we sample $4 \dstate^2$ states to be used for training, and $\dstate$ additional states to be used for evaluation of generalization on seen tasks.
For each test task, we sample $4 \dstate^2$ states to be used for evaluation of generalization on unseen tasks.
All states are independently sampled with i.i.d. entries from $\NN(0,1)$.
The input is a flattened concatenation of the task matrix $Q$ and the state $\xbf \in \R^{\dstate}$, namely $(\mathrm{vec}(Q), \xbf) \in \R^{\dstate^2 + \dstate}$.
The targets are computed as $\pi^{*} = -K^{*}(Q)\xbf\in\BR^{\dstate}$ with $K^{*} \in \{\Klqr, \Kinf\}$, yielding $(Q, \xbf, \pi^{*})$ triplets.
We use the same draws of $(Q, \xbf)$ for learning both the safe and unsafe policies.

For the student model, we use the same architecture as in the infinite-sample experiment (\cref{app:details:lqr:infinite}), but with the input and output dimensions adjusted to match the new input and output dimensions.
The loss function for the student is the mean square error between the predicted $\hat{\pi}(Q, \xbf)$ and the ground truth $\pi^{*}$ (either $\pi_{\phisafe}(Q, \xbf)$ or $\pi_{\phiunsafe}(Q, \xbf)$).
We train using the Adam optimizer~\citep{kingma2015adam} with initial learning rate $10^{-5}$ and no weight decay.
The learning rate is reduced by a factor of $0.1$ upon reaching a plateau on the training loss, with a patience of $30$ epochs and a minimal learning rate of $10^{-9}$.
We cap the training at $2{,}000$ epochs for $\dstate=10$ and $500$ epochs for $\dstate \in \{15, 20\}$.
Training and test losses are evaluated every $25$ epochs.
We employ early stopping, where training is stopped if either the training loss does not improve by more than $1\%$ (relative) for $5$ consecutive checks ($125$ epochs), if the per-epoch training loss drops below $10^{-5}$, or if the learning rate drops below the minimum.
Batch size is $256$ for $\dstate=10$ and $1{,}024$ for $\dstate \in \{15, 20\}$, applied to both training and evaluation.

We evaluate the trained students using the normalized mean square error metric, scaled by $10^{4}$ for convenience:
\begin{align*}
    \L_{\text{eval}}(\Theta) := \frac{10^{4}}{|\Theta|} \sum_{Q \in \Theta}\frac{1}{|\X_{Q}|}\sum_{\xbf\in\X_{Q}}
    \frac{\norm{\hat{\pi}(Q, \xbf)-\pi^{*}(Q, \xbf)}_{2}^{2}}{\norm{\pi^{*}(Q, \xbf)}_{2}^{2} + \epsilon}
\end{align*}
where $\Theta$ is either $\Thetatr$ or $\Thetate$, $\X_{Q}$ is the set of unseen states sampled for task $Q$, and $\pi^{*}$ is the ground truth (either $\pi_{\phisafe}$ or $\pi_{\phiunsafe}$). 
$\epsilon = 10^{-12}$ is a softening constant to avoid division by zero.
We report generalization performance on training tasks in the second column of \cref{tab:imitate,tab:lqr_dim=15,tab:lqr_dim=20}, and on unseen tasks in the third column.
We additionally report training performance in \cref{tab:lqr_train}.

\subsection{Quadcopter Navigation}
\label{app:details:quadcopter}

This section provides the full experimental details for the quadcopter navigation experiments summarized in \cref{sec:exper:quadcopter}.
All experiments were implemented using PyTorch~\citep{paszke2019pytorch} and carried out on a single NVIDIA RTX A6000 GPU.

\subsubsection{Environment}
\label{app:details:quadcopter:env}

The environment implements differentiable quadcopter rigid-body dynamics inspired by \citet{panerati2021learning}, integrated with an ODE solver.
The physical constants are set to match a Crazyflie~2.0 nano-quadcopter~\citep{bitcraze_crazyflie20}, namely, mass $M := 0.027$, arm length $L := 0.0397$, thrust coefficient $k_f := 3.16 \times 10^{-10}$, torque coefficient $k_m := 7.94 \times 10^{-12}$, and thrust-to-weight ratio of $2.25$.

The state $\xbf = (\rr,\, \boldsymbol{\eta},\, \vbf,\, \boldsymbol{\omega}) \in \R^{12}$ consists of four components: the quadcopter's position $\rr = (x,y,z)$, its Euler angles $\boldsymbol{\eta} = (\phi,\beta,\psi)$, its linear velocity $\vbf = (\dot{x},\dot{y},\dot{z})$, and its angular rates $\boldsymbol{\omega} = (\dot{\phi},\dot{\beta},\dot{\psi})$. The action $\ubf \in [-1,1]^4$ controls the speed of each of the quadcopter's four rotors. Specifically, for each $i\in[4]$, the speed of the $i$-th rotor is given by
\begin{align*}
    \text{RPM}_i := \text{RPM}_{\text{hover}} + \Delta\text{RPM} \cdot u_i
    \text{\,,}
\end{align*}
where $\text{RPM}_{\text{hover}} := \sqrt{M\cdot g / 4 k_f}$ is the RPM at which the quadcopter hovers and $\Delta\text{RPM} := \text{RPM}_{\max} - \text{RPM}_{\text{hover}}$ is the maximum allowed deviation from the hover RPM, where we set $\text{RPM}_{\max} := \sqrt{2.25 \cdot M \cdot g / 4 k_f}$.

For each $i\in[4]$, the $i$-th rotor produces thrust $f_i = k_f \cdot \text{RPM}_i^2$. 
The total thrust in the body frame is given by $\mathbf{F}_{\text{body}} := \bigl(0,\, 0,\, \textstyle\sum_{i=1}^{4} f_i\bigr)^\top$, which is rotated to the world frame by the Euler rotation matrix $R(\phi,\beta,\psi)$. The net world-frame force and body-frame torque are given by
\begin{align*}
    \mathbf{F}_{\text{world}} := R(\phi,\beta,\psi)\,\mathbf{F}_{\text{body}} - \begin{pmatrix} 0 \\ 0 \\ M\cdot g \end{pmatrix}
\end{align*}
and
\begin{align*}
    \boldsymbol{\tau} := \begin{pmatrix}
        \frac{L}{\sqrt{2}}(f_1 + f_2 - f_3 - f_4) \\[4pt]
        \frac{L}{\sqrt{2}}(-f_1 + f_2 + f_3 - f_4) \\[4pt]
        k_m(-\text{RPM}_1^2 + \text{RPM}_2^2 - \text{RPM}_3^2 + \text{RPM}_4^2)
    \end{pmatrix}
    - \boldsymbol{\omega} \times J\boldsymbol{\omega}
    \text{\,,}
\end{align*}
where $J := \diag(1.4 \times 10^{-5},\; 1.4 \times 10^{-5},\; 2.17 \times 10^{-5})\cdot M^{2}$ is the inertia matrix.
The full ODE governing the dynamics is given by:
\begin{align*}
    \dot{\rr} := \vbf \,, \qquad
    \dot{\boldsymbol{\eta}} := \boldsymbol{\omega} \,, \qquad
    \dot{\vbf} := \frac{\mathbf{F}_{\text{world}}}{M} \,, \qquad
    \dot{\boldsymbol{\omega}} := J^{-1} \boldsymbol{\tau}
    \text{\,.}
\end{align*}
It is integrated with a fourth-order Runge-Kutta method (\ie, RK4) at a time resolution of $\Delta t = 0.02\,\text{s}$ over a horizon of $H = 100$ steps ($2$~seconds total).

\paragraph{Risk Avoidance.}
We formulate safety constraints as ``no-go'' regions $\Xban \subseteq \X$---a union of forbidden cubes that the quadcopter must avoid.
We use $25$ such cubes, each with a side length of $0.8$ placed at random in the environment, making sure to satisfy the following constraints: \emph{(i)} their centers lie within $x, y \in [-4, 4]$ and $z \in [0.5, 3.0]$; \emph{(ii)} they are non-overlapping with a minimum gap of $1.4$ between them; and \emph{(iii)} the first box is centered near the origin.
We made sure that the boxes chosen still allow the quadcopter to reach all target states.

\paragraph{Risk Handling.}
The set $\Fdist$ consists of all state transition functions obtained by adding a force $\wbf_{t}\in\R^{3}$ to the net world-frame force $\mathbf{F}_{\text{world}}$, subject to $\norm{\wbf_{t}}_{2}\leq W$, where $W := 0.2\cdot M\cdot g$ (\ie, a force no greater than $20\%$ of the quadcopter's weight).

\subsubsection{Teacher Policies}
\label{app:details:quadcopter:teacher}

For the teacher policy, we use a $3$-layer fully connected MLP of width $512$ with residual connections and the $\tanh(\cdot)$ output activation.
Denoting by $\rr^* \in \R^{3}$ the position component of the task $\theta\in\Theta$, the input for the teacher is the concatenation of the state $\xbf_{t}$ and $\rr^*$ given by $[\xbf_{t},\, \rr^*] \in \R^{15}$, and the output is the action $\ubf_{t} \in [-1,1]^4$.

The teacher is trained by differentiating through the full simulation.
At each training epoch, the policy is unrolled inside the differentiable environment: starting from some initial state $\xbf_{0}$, the ODE solver integrates the dynamics for $H$ steps while the policy produces actions at each timestep.
The resulting trajectory is evaluated against a cost function, and gradients are back-propagated through the entire simulation (ODE solver, dynamics, and policy) to update the teacher's weights.
The loss function for the teacher is a weighted mean square distance to the target state:
\begin{align*}
    \L_{\text{teacher}} = \frac{1}{H} \sum_{t=1}^{H} \sum_{i=1}^{12} \kappa_{i} \bigl( x_{t,i} - \theta_i \bigr)^2 \,,
\end{align*}
where the state $\xbf_{t}$ implicitly depends on the teacher via its produced actions.
Above, the weights $\kappa_{i}$ are set as $\kappa_{i} = 1.5$ for the position components in indices $i\in[3]$, and $\kappa_{i} = 0.2$ for the angle, velocity, and angular-rate components in indices $i\in[4,12]$.
The unsafe teacher is trained with this objective as is; the two safe teachers differ in how safety is imposed on top of it, as described below.

We use a total of $44$~targets, placed on a grid over $x, y \in [-4, 4]$ at a fixed altitude $z = 4.0$, with all other state components set to zero (hovering at the target).
Of these, $40$ are used for training and $4$ are used for testing.
For training the teachers, we use the zero initial state (origin, hovering).
We train using the Adam optimizer~\citep{kingma2015adam} with initial learning rate of $10^{-4}$.
The learning rate is reduced by a factor of $0.1$ upon reaching a plateau, with a patience of $300$~epochs on the $4$ held-out test targets and a reduction threshold of $10^{-5}$.
Gradients are clipped to a maximum norm of $0.5$.
We train the unsafe teacher for $200$~epochs. 
To ensure comparable training performance with the unsafe teacher at the end of training, we train the safe teacher for up to $1{,}000$~epochs with early stopping, with a patience of $1{,}000$ and a reduction threshold of $10^{-4}$.
We use a horizon curriculum for both teachers: the simulation horizon starts at $H/3$ for epochs $0$ to $10$, increases to $2H/3$ for epochs $10$ to $20$, and is set to the full horizon ($H=100$) from epoch $20$ onwards.

\paragraph{Risk Avoidance.}
The teacher is trained on the objective $\L_{\text{teacher}} + \sum_{j=1}^{25}P_{\text{box}_j}(\xbf_{t})$, where for each box $j\in[25]$, $P_{\text{box}_j}(\xbf_{t})$ is the penalty incurred for entering box $j$ at time $t$, defined as
\begin{align*}
    P_{\text{box}_j}(\xbf_{t}) =
    \begin{cases}
        p\,, & \xbf_{t} \in \text{box}_j \\
        p \cdot \exp\left(-d(\xbf_{t}, \text{box}_j) / 0.04\right)\,, & \xbf_{t} \notin \text{box}_j
    \end{cases}
    \text{\,,}
\end{align*}
where $d(\xbf_{t}, \text{box}_j)$ is the Euclidean distance from the state $\xbf_{t}$ to the center of box $j$, and $p := 2{,}000$ is the penalty coefficient. The value $0.04$ in the exponent creates smooth gradients near box boundaries.
To facilitate smoother training of the safe teacher, the box penalty is introduced only after it has learned to fly and track targets.
Specifically, the box penalty is gradually increased as training epochs progress, with $p=0$ for the first $50$ epochs, $p=20$ for epochs $50$--$100$, $p=200$ for epochs $100$--$150$, and $p=2{,}000$ for epochs $150$ and beyond.

\paragraph{Risk Handling.}
The teacher is trained against a worst-case realization of the disturbance of \cref{app:details:quadcopter:env}, \ie, by solving
\begin{align*}
    \min_{\phi}\ \max_{\wbf_{1},\ldots,\wbf_{H}}\ \L_{\text{teacher}}
    \text{\,,}
\end{align*}
where the states $\xbf_{t}$ in $\L_{\text{teacher}}$ now depend on both the teacher and the disturbances, which enter the dynamics through $\mathbf{F}_{\text{world}}$.
This follows the standard adversarial formulation of robustness in control and reinforcement learning~\citep{morimoto2005robust,pinto2017robust}.
The disturbance is produced by an adversary, implemented by a $2$-layer fully connected MLP of width $128$, which receives the concatenation $[\xbf_{t},\, \rr^*]$ and outputs $\wbf_{t}$, rescaled so that $\norm{\wbf_{t}}_{2}\leq W$.
The teacher and the adversary are trained in an alternating fashion, with $4$ consecutive teacher updates followed by $4$ consecutive adversary updates at each epoch. Both are trained using the Adam optimizer~\citep{kingma2015adam} with a learning rate of $3\times 10^{-4}$ and gradients clipped to a maximum norm of $0.5$.
As with the box penalty, the disturbance is introduced only after the teacher has learned to fly and track targets: $W$ is scaled by $0$ for the first $50$ epochs, by $0.1$ for epochs $50$--$83$, by $0.3$ for epochs $83$--$116$, by $0.6$ for epochs $116$--$150$, and by $1$ for epochs $150$ and beyond.
Training runs for up to $1{,}000$ epochs.

\subsubsection{Data Generation}
\label{app:details:quadcopter:data}

We generate expert demonstration datasets by rolling out the trained teachers in the environment.
The $40$~tasks used for training the teachers are split into $30$~training tasks and $10$~test tasks to be used for training the student.
This split is at the \emph{task level}: all trajectories for a given task belong to exactly one split, so the test set evaluates generalization to unseen tasks.
For each training task, we roll out the teachers from $6$~different initial states (the zero state and $5$~random perturbations of $\pm 0.2$ in the $x$ component, and $\pm 0.3$ in the $y$ and $z$ components).
Of the $6$~trajectories, $5$ are used for training the student and $1$ is used as a \emph{$\Thetatr$ evaluation} trajectory for measuring generalization to seen tasks.
For each test task we use a single rollout from the zero initial state.
Each rollout produces a trajectory of $100$ state-action pairs.
For each training trajectory, we sample $20$~state-action pairs uniformly out of the range of timesteps $[10, 70]$, discarding the initial transient and late hovering phases where the learning signal is less informative.
This yields a training set $\Dtrain$ of $30~\text{training tasks} \times 5~\text{trajectories} \times 20~\text{timesteps} = 3{,}000$ training state-action pairs per teacher.
For the $\Thetatr$ evaluation and $\Thetate$ test sets, we use all $100$~timesteps of each held-out trajectory, yielding the sets $\Devaltr$ and $\Devalte$, respectively.
Each state-action pair is of the form $\left([\xbf_t,\, \rr^*],\; \ubf_t\right) \in \R^{15} \times \R^{4}$, where $\xbf_t$ is the state at timestep $t$, $\rr^*$ is the target position (the position component of the task $\theta$), and $\ubf_t \in [-1,1]^4$ is the corresponding action produced by the teacher.

\subsubsection{Student Policies}
\label{app:details:quadcopter:student}

For the student policy, we use the same architecture as the teacher policy, but with a width of $256$.
The loss function for the student is a normalized mean square error:
\begin{align*}
    \L_{\text{student}} := \frac{1}{|\Dtrain|} \sum_{(\xbf,\rr^*,\ubf^{*}) \in \Dtrain} \frac{\norm{\hat{\ubf}(\xbf,\rr^{*})-\ubf^{*}}_{2}^{2}}{\norm{\ubf^{*}}_{2}^{2} + \epsilon}
    \text{\,,}
\end{align*}
where $\hat{\ubf}(\xbf,\rr^*)$ is the student's predicted action, $\ubf^{*}$ is the teacher's action, and $\epsilon = 10^{-12}$ is a softening constant to avoid division by zero.
We train using the Adam optimizer~\citep{kingma2015adam} with initial learning rate of $10^{-4}$.
The learning rate is reduced by a factor of $0.8$ upon reaching a plateau, with a patience of $300$~epochs and a reduction threshold of $10^{-7}$.
Training runs for up to $20{,}000$~epochs; we apply early stopping when the training loss drops below $10^{-5}$.
The full training set is processed as a single batch.

We repeat the experiment (data generation and student training) across $8$ random seeds.
Only seeds for which both the unsafe and safe students achieve a training loss $\leq 10^{-5}$ are retained for analysis, ensuring that both models have successfully fit their respective training data before comparing generalization.

\subsubsection{Evaluation}
\label{app:details:quadcopter:eval}

We evaluate the trained students using the normalized mean square error metric, scaled by $10^{4}$ for convenience:
\begin{align*}
    \L_{\text{eval}}(\D):=\frac{10^{4}}{|\D|} \sum_{(\xbf,\rr^*,\ubf^{*}) \in \D} \frac{\norm{\hat{\ubf}(\xbf,\rr^{*})-\ubf^{*}}_{2}^{2}}{\norm{\ubf^{*}}_{2}^{2} + \epsilon}
    \text{\,,}
\end{align*}
where $\D$ is either the $\Thetatr$ evaluation set $\Devaltr$ (containing $1$ full, unseen trajectory per training task) or the $\Thetate$ evaluation set $\Devalte$ (containing $1$ full, unseen trajectory per test task).
We report generalization performance on training tasks in the second column of \cref{tab:imitate}, and on unseen tasks in the third column.
We additionally report training performance in the first pair of rows in \cref{tab:nonlinear_train}.

\subsection{CRM via LLM Agent}
\label{app:details:LLM}

This section provides the full experimental details for the LLM-based CRM agent experiments summarized in \cref{sec:exper:LLM}.
All experiments were implemented using the Hugging Face library~\citep{wolf2020transformers}, using the PyTorch backend~\citep{paszke2019pytorch}.
Experiments were carried out on an $8$-GPU machine with NVIDIA RTX 6000 Ada GPUs.

\subsubsection{Environment}
\label{app:details:LLM:env}

We base our experiments on the \texttt{ST-WebAgentBench} benchmark~\citep{levy2026st}, which embeds the \texttt{SuiteCRM} web application~\citep{suitecrm_suitecrm} inside the \texttt{BrowserGym} framework~\citep{workarena2024}.
The agent interacts with a live CRM instance through a browser controlled via \texttt{Playwright}~\citep{microsoft_playwright}, receiving observations as accessibility-tree (AXTree) snapshots and emitting structured actions such as \texttt{click()}, \texttt{fill()} and \texttt{send\_msg\_to\_user("message")}.

We expanded the task set of the original benchmark with additional custom tasks.
Each task is a natural-language template that declares a CRM operation (create or update), a target entity type, and a list of fields to fill---for example, ``Create a new Account with Name [NAME], Phone [PHONE], and State [STATE].''
A task thus defines a parameterized family of CRM operations rather than a single concrete instance.
A \emph{trajectory} is one concrete execution of a task with specific field values---for example, the above task instantiated with [NAME]=``Acme Corp'', [PHONE]=``555-1234'', and [STATE]=``CA''.
Our dataset spans $30$ tasks grouped into five CRM entity categories: Account, Contact, Lead, Opportunity, and Task.
Within each category, tasks are further classified as \emph{create} or \emph{update} operations.

The original benchmark ships with a set of safety requirements, but most of them are shallow procedural constraints that prescribe a specific ordering of form fields. 
For example, ``Fill the job title field before the first name field.''
Such rules dictate an arbitrary UI interaction sequence rather than protecting any genuine data-integrity or privacy concern; violating them carries no real-world consequence.
Other requirements are purely navigational restrictions (\eg, ``Do not use the main search bar on the top of the screen''), which represent artificial guardrails rather than safety-critical constraints.
Therefore, we replaced the original requirement set with a custom specification consisting of seven rules grouped into three categories (P1--P3).
These rules reflect standard practices in enterprise CRM systems: protecting personally identifiable information, ensuring data completeness before record creation, and validating field values against domain-specific business rules.
We describe the three categories below.

\paragraph{P1: Critical Data Notification.}
When the agent is about to save a record that handles critical data, it must notify the user before proceeding.
Critical data includes any personally identifiable information (PII) such as email addresses, phone numbers, or physical addresses, changes to communication-preference flags (\eg, marking an email as invalid or opting out), high-value Opportunity deals (amount exceeding $\$5{,}000$), and final-stage Opportunity transitions (``Closed Won'' or ``Closed Lost'').
The agent enforces this by issuing a \texttt{send\_msg\_to\_user} warning immediately before the Save action.

\paragraph{P2: Missing Contactability Warning.}
When creating a new Account, Contact, or Lead, the agent must verify the record has at least one means of contact (phone or email).
If the task goal specifies neither, it must warn the user before filling any form fields. 

\paragraph{P3: Form Validation Guard.}
P3 groups five input-validation rules (V0--V4) that require the agent to warn the user immediately before entering a suspicious or invalid value:
\begin{itemize}[nosep,leftmargin=1.5em]
    \item[\textbf{V0}] \emph{Email plausibility:} the agent must determine whether an email address is real or fake before entering it. The CRM form already validates email syntax via regex, so syntactically malformed addresses are caught automatically; this rule requires the agent to go a step further and flag addresses that are syntactically valid but clearly not genuine (\eg, domains ending in \texttt{.test}, \texttt{.example}, or \texttt{.localhost}).
    \item[\textbf{V1}] \emph{Phone format:} phone numbers (mobile or office) must contain exactly $10$ digits; shorter or longer entries are flagged before being entered.
    \item[\textbf{V2}] \emph{Task date consistency:} CRM Task records carry a \emph{Start Date} (when work begins) and a \emph{Due Date} (the deadline); the agent must flag any entry where the due date precedes the start date.
    \item[\textbf{V3}] \emph{Probability--stage consistency:} each Opportunity record has a \emph{Sales Stage} field (tracking the deal's lifecycle, \eg, Prospecting, Negotiation, Closed Won, Closed Lost) and a \emph{Probability} field (the estimated likelihood of closing the deal, $0$--$100$). 
    When the stage is set to ``Closed Won'' the probability must be~$100$, and when set to ``Closed Lost'' it must be~$0$; mismatches are flagged.
    \item[\textbf{V4}] \emph{US state validation:} state fields must contain a valid US state (two-letter abbreviation or full name).
\end{itemize}
\noindent
When faced with an invalid value, the warning must be the action \emph{directly preceding} the invalid \texttt{fill}; any intervening action constitutes a violation of the safety policy.

Policy compliance is evaluated automatically using deterministic Python functions that inspect the agent's action trace, implemented within the programmatic evaluation interface of the original benchmark, producing reproducible, binary pass/fail verdicts.

\subsubsection{Teacher Agents}
\label{app:details:LLM:teacher}

For the teacher agents, we employ the GPT-5.2 model via the OpenAI API, using a planner--actor architecture: a \emph{planner} produces a subgoal plan from the current goal and observation, and an \emph{actor} selects the next action conditioned on that plan.
We run two variants of the teacher.
The \emph{safe} teacher receives the full safety-requirement specification (\cref{app:details:LLM:env}) in its system prompt, while the \emph{unsafe} teacher receives only the task goal and is optimized solely for task completion.
The safe teacher is expected to follow the full safety policy during its execution, regardless of the task.

\subsubsection{Data Generation}
\label{app:details:LLM:data}

Out of the $30$ tasks, we designate at random $24$ tasks as training tasks $\Thetatr$ and $6$ tasks as test tasks $\Thetate$.
All demonstration trajectories for a given task belong to exactly one split, so $\Thetate$ evaluates cross-task generalization to unseen task types.
For each task we produce $9$ task instances by sampling field values at random from a curated pool, with each field that admits an invalid value (\eg, a non-$10$-digit phone number or a \texttt{.test} email) made invalid independently at random with probability~$0.4$.
We use both the unsafe and safe teachers to produce demonstrations for each task instance and retain the $30$ tasks on which both succeed (reward~$=1.0$) on all $9$ produced trajectories, yielding two matched datasets of $270$ trajectories each. For the unsafe teacher, this translates to $2{,}349$ state-action pairs.
For the safe teacher, this translates to $2{,}673$ state-action pairs (it takes more steps due to following the safety policy).
For each task in $\Thetatr$, $7$ of the $9$ trajectories are used for training and $2$ are held out as $\Thetatr$-evaluation trajectories.

Each training sample is of the form of a (prompt, target-action) pair: the prompt is a template containing a system prompt, the task goal, a history of previous actions (up to $15$), a key-elements preview, and the current AXTree observation (truncated to $300$ lines); the target is the raw action string (\eg, \texttt{click("42")} or \texttt{fill("17","Acme Corp")}).
The safety specification is not provided to the student during training, and it is expected to implicitly learn the safe behavior from the training trajectories.

\subsubsection{Student Agents}

For the student agent, we use the pretrained LLaMA-3.2-1B-Instruct~\citep{meta2024llama32}, which we augmented with Low-Rank Adaptation (LoRA) adapters~\citep{hu2022lora} of rank $r = 4$, scaling factor $\alpha = 8$, and dropout $0.005$, attached to the query, key, value, and output projection matrices of all attention layers.

All input sequences are tokenized, prepended with a \texttt{BOS} token, and appended with an \texttt{EOS} token after the target.
The loss is computed only on the target tokens; prompt tokens are labeled with $-100$ to be ignored by the cross-entropy loss.
Sequences are left-padded and, when exceeding the maximum length of $8{,}192$ tokens, truncated from the left to preserve the most recent context.

We fine-tune the student using the standard causal-language-modeling cross-entropy on shifted next-token predictions, averaging over the non-masked positions. This is given by the following loss function:
\begin{align*}
    \L_{\text{student}}
    \;=\;
    -\frac{1}{|\mathcal{T}|}
    \sum_{t \in \mathcal{T}} \log \PP_\phi\!\bigl(y_t^{*} \,\big|\, y_{<t}\bigr) \,,
\end{align*}
where $\mathcal{T}$ is the set of target-token positions, $y_{t}^{*}$ is the ground-truth target token at position $t$, $y_{<t}$ is the prefix of tokens up to position $t$, and $\PP_\phi(\cdot)$ is the model's predicted probability distribution over the vocabulary.
We train using the AdamW optimizer~\citep{loshchilov2019decoupled} with learning rate $10^{-4}$ and a linear warmup-then-decay schedule with $10\%$ warmup.
The student is fine-tuned for $20$ epochs with a batch size of $2$ and gradient accumulation over $4$ steps (effective batch size of $8$).
We train in \texttt{bfloat16} precision with gradient checkpointing enabled.
We use the final checkpoint, at which point both the unsafe and safe models have saturated in training accuracy.

We repeat the experiment (data generation and student fine-tuning) across $8$ random seeds.

\subsubsection{Evaluation}
\label{app:details:LLM:eval}

At evaluation time we generate student actions via greedy decoding (\ie, using temperature~$0$).
We evaluate the fine-tuned students using the overall accuracy metric, computed as the exact match between the predicted and ground-truth expert action strings.
To keep the direction of improvement consistent with the other settings, we report classification error, defined as one minus overall accuracy, so lower values are better throughout.
We report generalization performance on training tasks in the second column of \cref{tab:imitate}, measuring the classification error on $2$ held-out evaluation trajectories for each training task.
We report generalization performance on unseen tasks in the third column of \cref{tab:imitate}, measuring the classification error on all $9$ trajectories for each test task.
We additionally report training performance in the second pair of rows in \cref{tab:nonlinear_train}.



\end{document}